%% file: malicious-nasty-writeup.tex
\documentclass[11pt]{article}

\usepackage{amsmath, amssymb, amsthm, xparse, xcolor}
\usepackage[margin=1in]{geometry}
\usepackage{dsfont}
\usepackage[parfill]{parskip}
\usepackage{url}
\usepackage{mathrsfs}
\usepackage{tcolorbox}
\usepackage{caption}
\usepackage{enumitem}
\usepackage{thm-restate}
\usepackage{microtype}
\usepackage{subcaption}
\usepackage{setspace}

\usepackage{tikz}
\usetikzlibrary{positioning}
\usetikzlibrary{shapes.geometric, arrows}
\usetikzlibrary{decorations.pathreplacing}

\usepackage{float}
\usepackage{sn-preamble}

\title{
Is nasty noise actually harder than malicious noise? \vspace{10pt}
}
\author{
    Guy Blanc \thanks{Stanford University.} \and
    Yizhi Huang \thanks{Columbia University.} \and
    Tal Malkin \thanks{Columbia University.} \and
    Rocco A.~Servedio \thanks{Columbia University.}
}

\date{\today}

\NewDocumentCommand{\ot}{g}{
    \IfNoValueTF{#1}{\tilde{O}}{\tilde{O}(#1)}
}

\newcommand{\ham}{\mathrm{ham}}
\newcommand{\Round}{\mathbf{Round}}

\newcommand{\keybitguess}{\mathrm{KeyBitGuess}}

\newcommand{\key}{\mathrm{key}}
\newcommand{\mal}{\mathrm{malicious}}
\newcommand{\nasty}{\mathrm{nasty}}
\newcommand{\val}{\mathrm{value}}

\newcommand{\error}{\mathrm{error}}
\newcommand{\epsa}{\eps_{\mathrm{additional}}}

\newcommand{\Amplify}{\textsc{Amplify}}
\newcommand{\BadAmplify}{\textsc{BadAmplify}}
\newcommand{\replaced}{\mathrm{replaced}}
\newcommand{\original}{\mathrm{original}}
\newcommand{\survivenoise}{\mathrm{survive}\text{-}\mathrm{noise}}
\newcommand{\surviveice}{\mathrm{survive}\text{-}\mathrm{ICE}} 
\newcommand{\correct}{\mathrm{correct}}
\newcommand{\incorrect}{\mathrm{incorrect}}
\newcommand{\contradictory}{\mathrm{contradictory}}
\newcommand{\noncontradictory}{\mathrm{noncontradictory}}

\newcommand{\new}{\mathrm{new}}

\newcommand{\test}{\mathrm{test}}
\newcommand{\fixed}{\mathrm{fixed}}

\newcommand{\additional}{\mathrm{additional}}
\newcommand{\add}{\mathrm{add}}
\newcommand{\sm}{\mathrm{small}}
\newcommand{\lar}{\mathrm{large}}

\newcommand{\dtv}{d_{\mathrm{TV}}}
\newcommand{\Bin}{\mathrm{Bin}}
\newcommand{\Ber}{\mathrm{Ber}}
\newcommand{\Enc}{\mathrm{Enc}}
\newcommand{\Dec}{\mathrm{Dec}}

\newcommand{\sub}{\Phi}
\newcommand{\subn}[1]{\sub_{{#1} \to n}}
\newcommand{\submn}{\subn{m}}

\newcommand{\Ext}{\mathrm{Ext}}

\newcommand{\zetaN}{\eta_N}

\newcommand{\TV}{d_{\mathrm{TV}}}

\begin{document}

\pagenumbering{gobble}

\maketitle

\begin{abstract}

We consider the relative abilities and limitations of computationally efficient algorithms for learning in the presence of noise, under two well-studied and challenging adversarial noise models for learning Boolean functions:

\medskip

\begin{itemize}

\item \emph{malicious} noise, in which an adversary can arbitrarily corrupt a \emph{random} subset of examples given to the learner; and

\smallskip

\item \emph{nasty} noise, in which an adversary can arbitrarily corrupt an \emph{adversarially chosen} subset of examples given to the learner.

\end{itemize}
\medskip

We consider both the distribution-independent and fixed-distribution settings.  Our main results highlight a dramatic difference between these two settings:
\medskip

\begin{enumerate}

\item For distribution-independent learning, we prove a \emph{strong equivalence} between the two noise models:  If a class ${\cal C}$ of functions is efficiently learnable 
in the presence of $\eta$-rate malicious noise, then it is also efficiently learnable 
in the presence of $\eta$-rate nasty noise.

\smallskip

\item In sharp contrast, for the fixed-distribution setting we show an \emph{arbitrarily large separation}:  Under a standard cryptographic assumption, for any arbitrarily large value $r$ there exists a concept class for which there is a ratio of $r$ between the rate $\eta_{\mathrm{malicious}}$ of \emph{malicious} noise that polynomial-time learning algorithms can tolerate, versus the rate $\eta_{\mathrm{nasty}}$ of \emph{nasty} noise that such learning algorithms can tolerate.

\end{enumerate}
\medskip

To offset the negative result given in (2) for the fixed-distribution setting, we define a broad and natural class of algorithms, namely those that \textsl{ignore contradictory examples (ICE)}. We show that for these algorithms, malicious noise and nasty noise are equivalent up to a factor of two in the noise rate: Any efficient ICE learner that succeeds with $\eta$-rate malicious noise can be converted to an efficient learner that succeeds with $\eta/2$-rate nasty noise. We further show that the above factor of two is necessary, again under a standard cryptographic assumption.

\smallskip
As a key ingredient in our proofs, we show that it is possible to efficiently amplify the success probability of nasty noise learners in a black-box fashion.  Perhaps surprisingly, this was not previously known; it turns out to be an unexpectedly non-obvious result which we believe may be of independent interest.

\end{abstract}

\newpage

\newpage
\setcounter{tocdepth}{2}
 \begin{spacing}{0.93}
 {\small
      \tableofcontents}
 \end{spacing}
\newpage

\setcounter{page}{1}
\pagenumbering{arabic}

\input{sections/intro}

\input{sections/techniques}

\input{sections/preliminaries}

\input{sections/distribution-independent}

\input{sections/Boosting}

\input{sections/strong-separation-malicious-nasty}

\input{sections/fixed-distribution-ICE}

\section*{Acknowledgements}
G.B.~was supported by a Jane Street Graduate Research Fellowship, NSF awards CCF-1942123, CCF-2211237, CCF-2224246, a Sloan Research Fellowship, and a Google Research Scholar Award. T.M.~and Y.H.~were supported by an Amazon Research Award, Google CyberNYC award, and NSF award CCF-2312242. R.A.S.~and Y.H.~were supported by NSF awards CCF-2211238 and CCF-2106429. Y.H.~was also supported by NSF award CCF-2238221.

\newpage

\bibliography{allrefs}
\bibliographystyle{alphaurl}

\appendix
\input{sections/AmplificationCounterexample}

\end{document}

%% file: sections/intro.tex
%!TEX root = ../malicious-nasty-writeup.tex

\section{Introduction} \label{sec:intro}

The Probably Approximately Correct (PAC) model of learning Boolean functions, introduced forty years ago in a seminal work of Valiant \cite{Valiant:84}, has emerged as a tremendously influential framework for the theoretical analysis of machine learning algorithms.  
A rich theory of PAC learning has been developed, with an emphasis on understanding the abilities and limitations of \emph{computationally efficient} (i.e., polynomial-time) learning algorithms. 
However, one limitation of the original PAC learning model is that it makes the strong (and arguably unrealistic) assumption that the data set provided to the learning algorithm consists of a collection of labeled examples $(\bx,f(\bx))$ where each $\bx$ is an i.i.d.~draw from some distribution ${\cal D}$ --- in other words, it does not allow for the possibility of \emph{noise} in the input data set.
To remedy this, a range of different noise models extending the basic PAC learning framework have been proposed and considered, starting already with the \emph{malicious} noise model proposed in early work of Valiant \cite{Valiant:85}.

A broad distinction can be made between ``benign'' and ``adversarial'' noise models.  Intuitively, the former captures ``white noise''-type models in which the noise process affects all examples in the same way; a canonical example of such a noise model is the \emph{random classification noise} model of Angluin and Laird \cite{AngluinLaird:88}, where each example independently has its correct label flipped with some probability $\eta<1/2$.  Well-known results \cite{Kearns:98} show that many computationally efficient learning algorithms can be modified to learn to arbitrarily high accuracy in the presence of random classification noise at any noise rate bounded away from $1/2$.
In contrast, in \emph{adversarial} noise models an adversary may corrupt different examples in different ways in an attempt to thwart the learning algorithm.  
Not surprisingly, learning is much more challenging in adversarial models, and  it is well known that in several such models, including the ones we consider, the noise rate $\eta$ essentially provides a lower bound on the best possible error rate that can be achieved by any learning algorithm \cite{KearnsLi:93}.

\noindent {\bf This paper:  Malicious versus nasty noise.} 
We investigate the relationship between two of the most well-studied and challenging adversarial noise models in the PAC learning literature, namely \emph{malicious noise} and \emph{nasty noise}.  These models were introduced and analyzed in influential early papers in computational learning theory~\cite{Valiant:85,KearnsLi:93,BEK:02} and have continued to be the subject of much ongoing research down to the present day, see e.g.~\cite{MansourParnas:96,Auer:97,bshouty:98,CBD+:99,Servedio:03jmlr,Auer2008,KKMS:08,KLS:09jmlr,LongServedio:11malicious,BS12,ABL17,DKS18a,ShenZhang21,Shen23,ZengShen23,HKLM24}.\footnote{As discussed in e.g.~\cite{MDM18,MDM19,DMM20,BBHS22,HKMMM22}, the study of learning with malicious and nasty noise is also closely related to the study of \emph{data poisoning} attacks, a topic of considerable recent interest in machine learning.}
We give detailed definitions in \Cref{sec:prelims}, but for now, roughly speaking, 
\begin{itemize}

\item a \emph{malicious} adversary at noise rate $\eta$ works as follows:  for each (initially noiseless) example that is generated
for the learner, the malicious adversary has an \emph{independent $\eta$ probability} of being allowed to replace the example with an arbitrary (and hence adversarially chosen) noisy example. In contrast,

\item a \emph{nasty} adversary at noise rate $\eta$ is first allowed to see the entire (initially noiseless) data set of $n$ labeled examples that the learner would receive in the noise-free PAC setting, and then can \emph{choose} $\Bin(n,\eta)$ of the $n$ examples to corrupt.
\end{itemize}

While many papers, including the ones listed above, have considered specific algorithms for particular learning problems in the presence of malicious and nasty noise, there is a surprising dearth of general results on the relation between efficient malicious-noise-tolerant versus efficient nasty-noise-tolerant learning, or indeed of general results of any sort about nasty-noise-tolerant learning.  As an example of our lack of knowledge about this model, we remark that prior to the current paper it was not known whether \emph{success probability amplification} --- one of the most basic folklore primitives for 
noise-free learning
--- can be efficiently carried out in the setting of nasty noise.  (Our work answers this question 
positively 
en route to proving our main results; see \Cref{thm:boost-no-holdout-informal} and \Cref{sec:amplify}.)

One thing that is clear from the definitions is that learning in the presence of $\eta$-rate nasty noise is at least as difficult as learning in the presence of $\eta$-rate malicious noise. The impetus for this paper comes from asking about a \emph{converse}:  

\begin{quote}
{\bf Question:}  Suppose that a class of functions ${\cal C}$ is efficiently learnable in the presence of $\eta$-rate malicious noise. What can be said about its efficient learnability in the presence of nasty noise at various rates?
\end{quote} 

Surprisingly, this basic question seems to have gone completely unexplored prior to our work. 
There is some indirect evidence in the literature that malicious noise may be easier to handle than nasty noise; in particular, there are (at least) two
general-purpose techniques for constructing malicious noise learners that have not been shown to extend to nasty noise learners:
\begin{enumerate}
    \item The influential boosting framework of Freund and Schapire \cite{FreundSchapire:97} gives an algorithm for transforming distribution-independent \emph{weak}-learners, that produce a hypothesis with accuracy only slightly larger than $1/2$, into \emph{strong}-learners that achieve accuracy close to $1$. With a modified \emph{smooth}-boosting algorithm, Servedio \cite{Servedio:03jmlr} shows that weak-learners that are tolerant to malicious noise can be similarly upgraded to strong-learners that are tolerant to malicious noise.
    
    \item Suppose there exists an efficient learner in the noise-free setting that uses a sample of size $m$ for a class of functions $\mcC$. 
    Kearns and Li \cite{KearnsLi:93} showed that this algorithm can be black-box upgraded into an efficient learner that succeeds with $O((\log m)/m)$-rate malicious noise. This noise rate is asymptotically larger than the noise rate that can be handled by a naive observation that also applies to the nasty noise model, which is that if the noise rate is $\ll 1/m$, then a size-$m$ sample is likely to be unaffected by noise.  
\end{enumerate}

As our main contribution, we give a detailed study of the relative power of computationally efficient algorithms for learning in the presence of malicious noise versus nasty noise, in both distribution-independent and fixed-distribution settings.

\begin{remark} [On computational efficiency] \label{remark:time}
Before describing our results, we first explain why our focus is on \emph{computationally efficient} learning algorithms. The reason is simple: if we consider learning algorithms with unbounded running time, then it is well known and straightforward to establish (see e.g. Theorem~6 of \cite{KearnsLi:93} and Theorem~7 of \cite{BEK:02}) that a brute-force search over functions in the class ${\cal C}$, to maximize agreements with the input data set of labeled examples, is an effective (though typically computationally inefficient) algorithm to learn any PAC learnable concept class to optimal accuracy $O(\eta)$ in both the $\eta$-nasty noise and $\eta$-malicious noise models.
Thus, from an information-theoretic perspective, a class of functions ${\cal C}$ is learnable to accuracy $\Theta(\eta)$ in the $\eta$-nasty noise model, if and only if it is learnable to accuracy $\Theta(\eta)$ in the $\eta$-malicious noise model, if and only if it is learnable to arbitrary accuracy $\Theta(\eps)$ in the standard noise-free PAC learning model (which in turn holds if and only if ${\cal C}$ has finite VC dimension \cite{BEH+:89,VapnikChervonenkis:71}).  

In contrast, as we shall see, a much richer and more interesting range of phenomena manifest themselves when we consider 
\emph{computationally efficient} learning algorithms. Thus, throughout the paper we consider learning problems over the domain $X=\bits^d$, and we are interested in algorithms whose running time is polynomial in $d$ and in the other relevant parameters; we refer to such learning algorithms as ``computationally efficient.'' 
Note that even over infinite domains, efficient learners can only access a finite number of bits about each example. Hence, learners on infinite domains can be recast as equivalent learners over an appropriate discretized domain. In this setting, our results apply with $d$ representing the dimensionality of this discretized domain.

\end{remark}

\subsection{Overview of results}

\noindent {\bf Distribution-independent learning.}
In the original ``distribution-independent'' PAC learning framework proposed by Valiant \cite{Valiant:84}, a learning algorithm must succeed in constructing a high-accuracy hypothesis when it is run on examples drawn from \emph{any} (unknown) distribution ${\cal D}$ over the domain $X$.
Our main result for the distribution-independent setting, given in \Cref{sec:dist-independent}, is that efficient malicious-noise-tolerant learning is essentially equivalent to efficient nasty-noise-tolerant learning:  

\begin{restatable}
[Nasty noise is no harder than malicious noise for efficient distribution-independent learning, informal version of \Cref{thm:dist-free-combined}]{theorem}{distfreecombined}
\label{thm:main-distribution-free-informal}
Suppose that a class ${\cal C}$ of functions over $\bits^d$ is learnable to accuracy $\eps$ and confidence $\delta$ in $\poly(d,1/\eps,\log(1/\delta))$ time in the presence of $\eta$-rate malicious noise.
Then ${\cal C}$ is learnable to accuracy $1.01(\eps(1-\eta) + \eta)$ and confidence $1.01\delta$ in $\poly(d,1/\eps,\log(1/\delta))$ time in the presence of $\eta$-rate nasty noise.
\end{restatable}

Returning to the smooth boosting technique of \cite{Servedio:03jmlr} and the generic $O((\log m)/m)$ malicious noise tolerance transformation of \cite{KearnsLi:93} mentioned earlier, we remark that since each of these 
applies to distribution-independent learning, in each case \Cref{thm:main-distribution-free-informal} implies analogous results for nasty noise.
See \Cref{sec:consequences} for formal statements of these consequences of \Cref{thm:main-distribution-free-informal}.

The argument establishing \Cref{thm:main-distribution-free-informal} turns out to be fairly involved; it goes through a number of intermediate noise models, including the Huber contamination model, a total-variation-noise model, and a ``fixed-rate'' variant of the nasty noise model, and utilizes a recent result of \cite{BV24}. A schematic overview of the proof is given in \Cref{fig:structure} and a proof sketch in \Cref{subsec:dist-free-overview}.

We remark that a crucial step 
in obtaining only a logarithmic dependence on $1/\delta$ in the sample size of \Cref{thm:main-distribution-free-informal} the strong $1.01\eps$, $1.01\delta$ parameters that are achieved in \Cref{thm:main-distribution-free-informal} is a proof that it is possible to efficiently amplify the success probability of arbitrary learners in the setting of nasty noise:

\begin{restatable}[Amplifying the success probability with nasty noise, informal version of \Cref{thm:boost-no-holdout}]
{theorem}
{boostnoholdoutinformal}
\label{thm:boost-no-holdout-informal}
Given parameters $\epsa,\delta > 0$, and any learning algorithm $A$ running in time $T$ and taking in $n$ samples, there is a learner $A'$ taking $n \cdot \log(1/\delta) \cdot \poly(1/\epsa)$ samples and running in time $T \cdot \log(1/\delta) \cdot \poly(1/\epsa)$ with the following property: For any $\eps,\eta > 0$, concept class $\mcC$, and distribution $\mcD$, if $A$ learns $\mcC$ with expected error\footnote{See \Cref{sec:expected-error} for a discussion of learning with expected error.}
$\eps$ in the presence of $\eta$-nasty noise, then $A'$ learns $\mcC$ with $\eta$-nasty noise to confidence $1 - \delta$ and error $\eps + \epsa.$
\end{restatable}

While success probability amplification is a fundamental and widely-used primitive in noise-free PAC learning \cite{HKL+:91},
it seems not to have been previously known for the nasty noise model; we believe that this result may be of independent interest. (We remark that to achieve  amplification in the nasty noise setting, one cannot simply use the usual straightforward approach of generating multiple hypotheses and testing them on a holdout set since, as discussed in \Cref{sec:amplify}, the hypotheses need not be independent of the holdout set or each other.  Instead, as shown in \Cref{sec:amplify} and \Cref{sec:counterexample-amplify}, a significantly more delicate argument is required.)

\medskip
\noindent {\bf On the optimality of \Cref{thm:main-distribution-free-informal}.}
It is well known that for any for any concept class satisfying mild conditions the best error achievable by deterministic decision rules is $\eta/(1-\eta)$ when learning with $\eta$-malicious noise \cite{KearnsLi:93} and $2\eta$ when learning with $\eta$-nasty noise \cite{BEK:02}. The error overhead of Theorem 1, up to the 1.01 factor, exactly recovers the optimal $2\eta$ nasty noise error for deterministic decision rules given optimal malicious noise error, but Theorem 1 does this using a randomized decision rule, and it is known that randomized decision rules can improve these optimal error rates \cite{CBD+:99}. Theorem 8, in the body, does maintain this optimal error overhead and uses a deterministic decision rule, but has a polynomial dependence on $\delta$. It is an interesting open question whether the parameters of Theorem 1 are possible with deterministic decision rules, or whether the parameters can be improved using randomized decision rules.

\medskip
\noindent {\bf Fixed-distribution learning.} In the ``fixed-distribution'' or ``distribution-specific'' variant of PAC learning, there is a fixed distribution ${\cal D}$ over the domain $X$ that generates the examples, and the learning algorithm need only succeed in constructing a high-accuracy hypothesis with respect to ${\cal D}$. 
Fixed-distribution PAC learning has been widely studied for many years \cite{BenedekItai:91}, with a particular focus on the case where the fixed distribution ${\cal D}$ is the uniform distribution over $\bits^d$ (see e.g.~\cite{kha93,Jackson:97,LMN:93} and many other works).

In sharp contrast with the equivalence established in \Cref{thm:main-distribution-free-informal} for distribution-free learning, our main result for fixed-distribution learning is a strong separation between the malicious noise rate and the nasty noise rate that can be handled by computationally efficient algorithms:

\begin{restatable}[Separation between malicious and nasty noise, informal version of \Cref{thm:fixed-separation}]
{theorem}
{fixedseparationinformal}
    \label{thm:fixed-separation-informal}
    If one-way functions exist,
    then for any (arbitrarily large) constant $r$, there is a concept class ${\cal C}$ over $\bits^d$ and a constant $\eta>0$ such that under the uniform distribution,
    \begin{itemize}
    
    \item [(a)] ${\cal C}$ is $\poly(d)$-time learnable to high accuracy in the presence of \emph{malicious noise} at rate $\eta$; but
    
    \item [(b)] no $\poly(d)$-time algorithm can learn ${\cal C}$ to accuracy even 51\% in the presence of \emph{nasty noise} at rate $\eta/r$.
    
    \end{itemize}
    \end{restatable}
    
    Our approach to prove \Cref{thm:fixed-separation-informal} combines pseudorandom functions, seeded extractors, and low-Hamming-weight subcodes of efficiently erasure-list-decodable binary codes; we give an overview of the proof in \Cref{subsec:overview-fixed-separation}.
    
\medskip
\noindent {\bf ICE-algorithms.}
While \Cref{thm:fixed-separation-informal} shows that in general there can be an arbitrarily large separation between how much malicious versus nasty noise can be handled by efficient learning algorithms in the fixed-distribution setting, it is natural to ask whether there are circumstances in which this gap can be less acute.
Towards this end, we define a broad and natural class of algorithms, namely those that \emph{ignore contradictory examples}, and show that for such algorithms malicious noise and nasty noise are essentially equivalent.
\begin{definition}[Ignore contradictory examples (ICE) algorithms, informal version of \Cref{def:ICE-algorithm}]
    A learning algorithm is said to \emph{ignore contradictory examples} if, for any $x \in \bits^d$, its output is unaffected by the addition of the two labeled points, $(x, +1)$ and $(x,-1)$, to its sample (note that its sample is a multiset of labeled points).
\end{definition}
ICE-algorithms capture a natural and intuitive way to deal with noisy data.
As an example, we observe that any \emph{correlational statistical query} algorithm (see \cite{BDIK90,Feldman08,Reyzin20}) corresponds naturally to an ICE-algorithm.  (Recall that a correlational statistical query algorithm uses only estimates of the expected value of $g(\bx) \cdot \by$ for query functions $g: X \to [-1,1]$ of the learner's choosing; a contradictory pair of examples $(x,1)$ and $(x,-1)$ contribute 0 to the estimate of $\E[g(\bx) \by]$ for any $g$.)

\begin{restatable}[ICE-malicious learners imply nasty noise learners, informal version of \Cref{thm:ICE-malicious-nasty}]
{theorem}
{ICEmaliciousnastyinformal}
    \label{thm:ICE-malicious-nasty-informal}
    Fix any constant ratio $\kappa < 0.5$. If there is an efficient $(\eps,\delta)$-ICE-learner for a concept class $\mcC$ with $\eta$-rate malicious noise on a fixed distribution $\mcD$, then there is also an efficient $(1.01\eps, 1.01\delta)$-learner for $\mcC$ with $\kappa \eta$-rate nasty noise on the same distribution $\mcD$.
\end{restatable}

In the typical setting where a single algorithm learns $\mcC$ with $\eta$-malicious noise to accuracy $O(\eta)$ for all noise rates $\eta$, then \Cref{thm:ICE-malicious-nasty-informal} implies the same statement in the nasty noise setting (with a slightly worse constant within the $O(\cdot)$). At a high level, the proof of \Cref{thm:ICE-malicious-nasty-informal} is by a careful coupling of adversary strategies, showing that $\kappa \eta$-rate nasty noise is essentially a special case of $\eta$-rate malicious noise as far as ICE-algorithms are concerned; we give an overview in \Cref{subsec:overview-ICE-upper}.

Finally, we complement \Cref{thm:ICE-malicious-nasty-informal} with a lower bound which, like \Cref{thm:fixed-separation-informal}, relies on the standard cryptographic assumption that one-way functions exist. This final result shows that the factor-of-2 difference between the malicious and nasty noise rates in \Cref{thm:ICE-malicious-nasty-informal} is inherent:

\begin{restatable}[Tightness of \Cref{thm:ICE-malicious-nasty-informal}, informal version of \Cref{thm:ICE-bad-news}]
{theorem}
{ICEbadnewsinformal}    \label{thm:ICE-bad-news-informal}
    If one-way functions exist, then \Cref{thm:ICE-malicious-nasty-informal} does not hold for any constant $\kappa > 0.5$, even if $\mcD$ is the uniform distribution.
    \Cref{thm:ICE-malicious-nasty-informal} does not hold for any constant $\kappa > 0.5$, even if $\mcD$ is the uniform distribution.
    If one-way functions exist, then for any constants $\eta > 0$, $\kappa > 0.5$, there is a concept class efficiently learnable by an ICE algorithm to $O(\eta)$-error with $\eta$-malicious noise but not efficiently learnable to 51\% accuracy with $\kappa \eta$-nasty noise.

\end{restatable}

The proof of \Cref{thm:ICE-bad-news-informal} bears some similarity to the proof of \Cref{thm:fixed-separation-informal}, but there are also significant differences:  it does not require seeded extractors, and it uses binary codes which are efficiently \emph{bit-flip list-decodable} (rather than erasure-list-decodable as in \Cref{thm:fixed-separation-informal}). See \Cref{subsec:overview-ICE-lower} for an overview of the argument.

We remark 
that in light of \Cref{thm:ICE-malicious-nasty-informal}, the concept class of our main lower bound \Cref{thm:fixed-separation-informal} is one for which a non-ICE algorithm is a more effective malicious noise learner than any ICE algorithm can be. Intuitively, our non-ICE algorithm of part (b) of \Cref{thm:ICE-malicious-nasty-informal} uses the existence of a contradictory pair $(x, +1), (x,-1)$ to deduce that the malicious adversary was responsible for at least one of these two points, which leaks information about the adversary's behavior. In contrast, an ICE algorithm would not know that the adversary had chosen to make such a corruption.

\subsection{Related work}
\label{subsec:related-work}
    Our results fit into the recent and growing body of work that aims to understand the relationship between various types of adversarial noise. 
    This includes a recent question, raised by \cite{BLMT22} and solved by \cite{BV24}, showing that ``adaptive" noise, in which the adversary has full knowledge of the (initially noiseless) data set before choosing its corruption, is no more difficult to handle than ``oblivious" noise, in which the adversary must commit to its corruptions of the initial noiseless data distribution in advance before the data set is drawn from it.
    Indeed, we use \cite{BV24}'s result along the way to proving our \Cref{thm:main-distribution-free-informal,thm:ICE-malicious-nasty-informal}. We remark that malicious and nasty noise differ in more than just the level of adaptivity
    afforded to the adversary (see the discussion in \Cref{subsec:dist-free-overview}); indeed this is why, as our results show, whether or not nasty and malicious noise are equivalent is markedly different for the distribution-free versus fixed-distribution settings.

    Interestingly, while adversarial noise in PAC (supervised) learning only affects computational complexity rather than sample complexity (see \Cref{remark:time}), the effect of adversarial noise in distribution learning is different.
    Recent work of \cite{BBKL23} shows that, without computational constraints, learning with ``additive" noise is no harder than learning without any noise, but learning with ``subtractive" noise can be strictly harder, even allowing for inefficient algorithms.

    More broadly, there has been a long effort (see the textbook  \cite{DK23book} and citations therein) to understand the most efficient algorithms for various tasks of interest in the presence of different noise models. For example, \cite{CHLLN23} show a polynomial separation in the number of samples needed for Gaussian mean testing under different noise models.

\subsection{Future Work}
We have shown that, while in the distribution-independent setting nasty and malicious noise are equivalent, in the fixed-distribution setting there is an arbitrarily large gap between the noise rates that can be efficiently tolerated. 
This motivates the quest to understand under what circumstances this gap can be bounded: When is nasty noise actually harder than malicious noise in the fixed-distribution setting? 

The approach we took in \Cref{thm:ICE-malicious-nasty-informal} and \Cref{thm:ICE-bad-news-informal} is to identify a natural condition on the \emph{learning algorithm} -- ICE -- for which we can tightly bound the gap. 
It is interesting to better understand the significance of ICE-learning: Is ignoring contradictory examples the best strategy in all 
``natural" cases?  Or is there any ``natural" concept class for which ICE-learners are weaker than non-ICE-learners?  
(as discussed, our proof of the gap constructs an ``artificial" concept class where this is the case).  

More generally, an intriguing direction for future work is to 
identify a natural condition on the \emph{concept classes}
for which we can bound the gap and show how \emph{any} efficient learning algorithm in the malicious noise model can be transformed to an efficient learning algorithm in the nasty noise model.

%% file: sections/techniques.tex
%!TEX root = ../malicious-nasty-writeup.tex

\section{Technical Overview} \label{sec:techniques}

\subsection{Proof sketch for \texorpdfstring{\Cref{thm:main-distribution-free-informal}}{Theorem~\ref{thm:main-distribution-free-informal}}}
\label{subsec:dist-free-overview}

We recall \Cref{thm:main-distribution-free-informal}:

\distfreecombined*

Broadly speaking, there are two main steps needed to prove \Cref{thm:main-distribution-free-informal} (though the full proof is broken up further, as depicted in \Cref{fig:structure}). These two steps correspond to the two key ways in which the nasty adversary is, \textsl{a priori}, more powerful than the malicious adversary.
\begin{enumerate}
    \item The types of corruptions the adversaries can make: Recall that the nasty adversary gets to choose $\eta$-fraction of the dataset to change. We can think of this as first removing $\eta n$ examples and then adding $\eta n$ corrupted examples. In contrast, we can think of the malicious adversary as, roughly speaking, being unable to remove examples: It is given a dataset of $(1-\eta)n$ examples and only gets to choose $\eta n$ to add.
    \item The level of adaptivity of the adversaries: The nasty adversary is fully adaptive, meaning that it knows the entire sample when deciding its corruptions. In contrast, the malicious adversary is partially adaptive in that whenever it chooses a corruption, it only knows the sample points that have been received so far, and hence has partial knowledge of the sample.
\end{enumerate}
The two broad steps of our reduction correspond to the above two distinctions. In the first step, we move from malicious noise to a non-adaptive variant of nasty noise (defined below in \Cref{def:TV-noise-intro}). In the second step, we move from this non-adaptive variant of nasty noise to standard (fully adaptive) nasty noise. This intermediate noise model has previously been referred to as ``general, non-adaptive, contamination"\cite{DK23book}, though we'll give it a different name that emphasizes its relationship to total variation distance.
 \begin{definition}[TV noise]
 \label{def:TV-noise-intro}
    In PAC learning with $\eta$-TV noise, given a base distribution $\mcD$ and concept $c$, let $\mcD_c$ be the resulting distribution over clean labeled examples, $(\bx, c(\bx))$, where $\bx \sim \mcD$. The adversary chooses any $\mcD'$ where
    \begin{equation}
        \label{eq:TV-shift}
        \TV(\mcD_c, \mcD') \leq \eta.
    \end{equation}
    The sample is then generated by taking i.i.d.~draws from $\mcD'$.
\end{definition}

\medskip
\noindent
{\bf From malicious noise to TV noise.}
Let $A$ be any distribution-independent learner with malicious noise. In the first step, we show that $A$ is also a distribution-independent learner with TV noise with no loss in parameters. To do so, it is helpful to think of the distribution $\mcD'$ in \Cref{eq:TV-shift} as being created in two stages: First, the adversary chooses $\eta$-fraction of the mass of $\mcD_c$ to remove and then it chooses $\eta$-fraction of new mass to add. Adding mass is not an issue, since $A$ works with respect to malicious noise which, as previously mentioned, can be thought of as adding corruptions to a dataset. The key argument of this step shows that removing mass from $\mcD_c$ is equivalent to changing the input distribution. Hence, since $A$ is a distribution-independent learner, it must also work on this new input distribution.

\medskip
\noindent
{\bf From TV noise to nasty noise.} We use recent work of Blanc and Valiant that showed, in a general sense, that adaptive and non-adaptive statistical adversaries are equivalent \cite{BV24}. In our case, we wish to show that a learner for TV-noise can black-box be upgraded to one for nasty-noise.

Executing this reduction introduces two subtleties. First, the result of \cite{BV24} doesn't give a learner for the most standard variant of nasty noise. Instead, it gives a learner for ``fixed-rate" nasty noise (see \Cref{def:fixed-rate-nasty}). We show that such learners can be transformed into learners for (standard) nasty noise in \Cref{prop:standard-nasty-harder}. Second, and more technically challenging, is that the result of \cite{BV24}  increases the dependence of the sample size on the inverse failure probability from polylogarithmic to polynomial. To get around this, we prove the surprisingly non-obvious success amplification result given in \Cref{thm:boost-no-holdout-informal}.

\subsection{Proof sketch for \texorpdfstring{\Cref{thm:boost-no-holdout-informal}}{Theorem~\ref{thm:boost-no-holdout-informal}}} \label{subsec:technque-boost}

It is a well-known fact that in the standard noise-free PAC learning setting, it is possible to generically amplify a constant success probability to a $1-\exp(-k)$ success probability at the cost of only an $O(k)$ overhead in sample complexity. The simple strategy for this is as follows: If the original learner uses a size-$n$ sample, instead request a size-$n(k+1)$ sample. Split it into $k+1$ pieces, and use $k$ of them to run $k$ independent copies of the original learner. With probability $1 - \exp(-k)$, one of the $k$ generated hypotheses will have small error. Then, use the final piece as a holdout set to identify one of the $k$ hypotheses which has small error.

If instead, we want to amplify the success rate of a learner that operates in the presence of nasty noise, the situation is more delicate. Since the adversary can choose its corruptions as a function of the entire size-$n(k+1)$ data set, after splitting the data set into $k+1$ pieces, the pieces are not guaranteed to be \emph{independent}, so it is \emph{a priori} unclear that one of the $k$ hypotheses will have small error with extremely high probability.  This issue also manifests in regard to testing the quality of our hypotheses using a holdout set: the training set, and therefore the hypotheses, are not necessarily independent of the holdout set, and so it is not clear that error rate on the holdout set gives an accurate estimate of true hypothesis accuracy. Indeed, we show in \Cref{sec:counterexample-amplify} that the standard noiseless approach to success amplification provably fails in the nasty noise setting.

Nonetheless, using a different approach and analysis, we show that efficient success probability amplification is indeed possible in the presence of nasty noise. We recall \Cref{thm:boost-no-holdout-informal}:
\boostnoholdoutinformal*

We remark that in \Cref{thm:boost-no-holdout-informal} the new learner $A'$ does not depend on the underlying distribution $\mcD$. Therefore, \Cref{thm:boost-no-holdout-informal} can be used to amplify success probability both for both fixed-distribution and distribution-independent learners.

The main technical piece of \Cref{thm:boost-no-holdout-informal} is showing that, given black-box access to $A$ and a size-$nk$ sample that was corrupted by a $\eta$-nasty adversary, it is possible to generate $k$ candidate hypotheses $\bh_1, \ldots, \bh_k$ satisfying
\begin{equation}
    \label{eq:errors-concentrate}
    \Pr\bracket*{\sum_{i \in [k]} \error(\bh_i) \leq \eps k + O\paren*{\sqrt{k \ln(1/\delta)}}} \geq 1 - \delta.
\end{equation}
The concentration in \Cref{eq:errors-concentrate} is exactly what one would obtain via a Chernoff bound if the $\bh_i$ were independent each with average error $\eps$. However, as discussed earlier, since the adversary can simultaneously corrupt the entire dataset, we cannot generate $k$ independent hypotheses.

Instead, we show that it is possible to generate $k$ hypotheses that are ``independent enough" for \Cref{eq:errors-concentrate} to hold. This argument proceeds by contrapositive: Using a carefully constructed coupling, we show that if \Cref{eq:errors-concentrate} is violated, then $A$ must have expected error larger than $\eps$. This argument is further explained in \Cref{subsec:amplify-description}.

\subsection{Proof sketch for \texorpdfstring{\Cref{thm:fixed-separation-informal}}{Theorem~\ref{thm:fixed-separation-informal}}}
\label{subsec:overview-fixed-separation}

We recall \Cref{thm:fixed-separation-informal}:

\fixedseparationinformal*

    This separation is based on the construction of a concept class in which any efficient learner has no choice but to decode a corrupted secret key and where the malicious and nasty adversaries correspond to two distinct ways of corrupting this key:
    \begin{enumerate}
        \item The malicious adversary can \emph{erase} $\eta$-fraction
        of the key bits. That is, it can overwrite any example indicating that a key bit is $b$ with a ``?" symbol. It can do this by introducing a corrupted version of the same example indicating that that bit is $\bar b$ rather than $b$. As a result, the learner will know that this bit must have been corrupted, but will not know its original value.
        \item The nasty adversary can \emph{flip} $\eta/r$-fraction of the key bits. That is, it can flip $b$ to $\overline{b}$. In this case, the learner will not know that this particular bit was corrupted.
    \end{enumerate}
    We will construct an error correcting code that is decodable with $\eta$-fraction of erasures (and hence learnable with $\eta$-malicious noise), but not even ``weakly decodable" with $\eta/r$-fraction of bit flips (and hence not learnable with $(\eta/r$)-nasty noise). For this second desideratum, we want it to be possible to flip $\eta/r$ fraction of the bits of the codeword in such a way that the resulting bit string provides no information about the uncorrupted codeword.

    \textbf{Construction of the code:} The code we use contains only low-Hamming-weight codewords and is efficiently list-decodable from erasures. Because the codewords have low-Hamming-weight, the nasty adversary can flip every $1$ in the codeword to $0$, leaving the learner with an all-$0$ corrupted codeword. This provides no information as the adversary is able to perform this corruption regardless of what the uncorrupted codeword was. On the other hand, when the dataset is corrupted by malicious noise, we use that the code is list-decodable from erasures to construct a learner which generates a small list of possible codewords. We then determine which of these is best using standard hypothesis testing arguments. We note that in order to achieve the arbitrarily large multiplicative gap between the malicious and nasty noise rates, we must use list-decoding rather than unique-decoding.

We use the probabilistic method to construct our code: Specifically, we start with a uniformly random linear code of an appropriately chosen rate. Then, to ensure that all codewords have low-Hamming-weight, we simply remove all high-Hamming-weight codewords: so our code is only defined over plaintext messages that happen to map to low-Hamming-weight codewords (this would be undesirable for many coding applications, but with some work it  suffices for our purposes). Our analysis of this code combines a result of Ding, Jin, and Xing  \cite{DingJX14} which shows that with high probability, a random binary linear code of small constant rate can be efficiently erasure-list-decoded from a high constant fraction of adversarial erasures, along with an elementary bound on the number of small-Hamming-weight code words in a random linear code.

\medskip
\noindent
{\bf Construction of the concept class given the code.}
    One part of the concepts we construct will correspond to the code, but it cannot be the only part. Intuitively, the reason is that the learner is only able to list-decode, so in order to distinguish the correct uncorrupted codeword from all of the other candidates in the list, the learner needs some other source of information.

    Thus, the domain for concepts in the concept class $\calC$ that we construct is partitioned into two pieces, a ``key-side'' and a ``value-side.'' The outputs on the key-side give a very redundant representation of the codeword $w$ (each bit of $w_i$ is repeated as the output bit for many key-side inputs), and the outputs on the value-side are the outputs of a pseudorandom function $f_{\mathrm{key}}$, where $\mathrm{key}$ depends on $w$. 
    The learner under malicious noise can list-decode the codeword using the examples on the key-side, and then using examples from the value side, it will be able to choose the correct uncorrupted codeword from the list with high probability. On the other hand, a learner under nasty noise will not gain any information on the codeword from the key-side, as the nasty adversary can corrupt the key-side outputs as described earlier. Consequently, the learner will not get any information from the value-side,
    as the pseudorandom function is (computationally) indistinguishable from a random function as long as the distribution of $\mathrm{key}$ is close to uniform.
    
    The only thing left now is to choose a good $\mathrm{key}$, which has to (i) depend on the codeword $w$, and (ii) have close-to-uniform distribution. Note that we can \emph{not} simply decode $w$ in order to get a close-to-uniform key, since our code is only defined over a small subset of plaintext messages and hence cannot be uniformly random (recall that this was done to ensure that all codewords have low Hamming weight). Instead, we set $\mathrm{key}=\Ext(w,s)$, the output of a seeded extractor $\Ext$, where $s$ is a parameter of the concept. We stress here that a crucial property is that the extractor $\Ext$ only has a very short seed length $|s|=O(\log d)$, as this allows the learner under malicious noise to also enumerate over all $2^{O(\log d)} = \poly(d)$ many possible seeds $s$. We thus obtain the full construction of the concept class $\calC$ in which each concept corresponds to a pair $(w, s)$, based on (a) pseudorandom functions, (b) seeded extractors, and (c) large, low-Hamming-weight subcodes of efficiently erasure-list-decodable binary codes. See \Cref{fig:domain-sep,fig:key-sep,fig:value-sep} for an illustration of how functions in ${\cal C}$ label various points in the domain.

\subsection{Proof sketch for \texorpdfstring{\Cref{thm:ICE-malicious-nasty-informal}}{Theorem~\ref{thm:ICE-malicious-nasty-informal}}}
\label{subsec:overview-ICE-upper}

We recall \Cref{thm:ICE-malicious-nasty-informal}:

\ICEmaliciousnastyinformal*

We show \Cref{thm:ICE-malicious-nasty-informal} in two steps. We first apply \Cref{lem:compare-mal} to transform an ICE-learner tolerating $\eta$-rate (standard) malicious noise to an ICE-learner tolerating $\eta$-rate \emph{strong} malicious noise. Strong malicious noise (formally defined in \Cref{remark:strong-malicious}), is, intuitively, a fully adaptive variant of malicious noise (recall item~2 of \Cref{subsec:dist-free-overview}).  Roughly speaking, under the strong malicious noise setting, the adversary first see all the examples and then decides what the corrupted examples should be, as opposed to setting the corrupted examples in an online manner under the standard malicious noise setting. Strong malicious noise (compared to standard malicious noise) better aligns with nasty noise, where the adversary also decide the corrupted examples after seeing all examples.

The next step is to show that an ICE-learner tolerating $\eta$-rate strong malicious noise can be transformed to a learner tolerating $\kappa\eta$-rate nasty noise.  The idea is that a noisy data set corrupted by a nasty-noise adversary can also be achieved by a strong-malicious-noise adversary after ICE-ing. Suppose we have a clean data set $\bS$, and the nasty-noise adversary corrupts it by replacing a set $S_{\replaced}$ of examples with a set $S_{\new}$ of examples, so the corrupted data set is $S_{\nasty} \coloneq (\bS \setminus S_{\replaced}) \sqcup S_{\new}$. For a strong-malicious-noise adversary, when given a data set $\bS'$ such that $\bS \subset \bS'$ and $\bS' \setminus \bS$ are exactly those examples that the malicious-noise adversary can corrupt, the malicious-noise adversary will use the ``corruption budget'' to first construct contradictory examples for all examples in $S_{\replaced}$ (so that these contradictory examples and $S_{\replaced}$ are removed after ICE-ing), and then construct examples in $S_{\new}$.
The resulting corrupted data set after ICE-ing is exactly $S_{\nasty} = (\bS \setminus S_{\replaced}) \sqcup S_{\new}$.

Some care is required in managing the parameters in this argument, as the above two steps do not directly yield the parameters we want;  in particular we use the success probability amplification of \Cref{thm:boost-no-holdout-informal}. Full details are given in \Cref{sec:ICE}.

\subsection{Proof sketch for \texorpdfstring{\Cref{thm:ICE-bad-news-informal}}{Theorem~\ref{thm:ICE-bad-news-informal}}}
\label{subsec:overview-ICE-lower}

We recall \Cref{thm:ICE-bad-news-informal}:

\ICEbadnewsinformal*

The proof of \Cref{thm:ICE-bad-news-informal} is similar to \Cref{thm:fixed-separation-informal} with trickier details. The main distinction is that we must construct a learner that succeeds with malicious noise and also ignores contradictory examples. In the setting of \Cref{thm:fixed-separation-informal}, we argued that, intuitively, the corruptions of the malicious adversary correspond only to erasures and not bit-flips of the error correcting code. This is no longer true in this setting.

In particular, for a key bit example with label bit $b$, the malicious adversary can introduce two copies of this example, both with label bit $\bar b$. Since the learner is ICE, one of these corrupted examples cancels out the uncorrected example, and as a result, the learner will only receive a single example labeled $\bar b$ and will have no way of knowing that this particular bit was flipped. (Crucially, for the malicious adversary to perform one bit flip in this way, it must introduce two corruptions, unlike the nasty adversary which can perform a bit flips at unit cost. This corresponds to the bound of $\kappa \approx 1/2$ in \Cref{thm:ICE-bad-news-informal}.)

Further complicating the analysis is that the malicious adversary can choose a clever mixture of erasures and bit flips, where erasing is cheaper than flipping (since it only requires one corrupted example). It can even perform corruptions that, roughly speaking, fall in between erasures and flips both in terms of cost to the adversary's corruption budget and in terms of effectiveness at fooling the learner. To handle this, we use a randomized rounding procedure that allows the ICE learner to generate a guess for the codeword that, with high probability, will be somewhat close to the uncorrupted codeword. This distance scales linearly with the number of corruptions the malicious adversary makes, and is, roughly speaking, half as large as it would be for the nasty adversary.

Each concept in the class $\calC$ that we construct corresponds to a key of a fixed pseudorandom function. As in the proof of \Cref{thm:fixed-separation-informal}, the domain for concepts in $\calC$ is partitioned into two pieces, a ``key-side'' and a ``value-side''. As in the earlier construction, the outputs on the key-side give a very redundant representation of the encoding of $\key$ in the bit-flip code, and the outputs on the value-side are the outputs of a pseudorandom function $f_{\key}$. The size of the key-side will be a fraction less than $2\kappa\eta$ but greater than $\eta$ of the size of the whole domain (as $\kappa > 0.5$). See \Cref{fig:domain,fig:key,fig:value} for an illustration of how functions in ${\cal C}$ label various points in the domain.

The high-level idea for proving the hardness of learning $\calC$ in the presence of $\kappa\eta$-rate nasty noise is simple: with very high probability, a nasty-noise adversary with noise rate $\kappa\eta$ can choose to corrupt half of the examples on the key-side and turn them into contradictory examples for the other half. In such a corrupted data set, the key-side provides no information to the learner, so the learner can only learn from the value-side. However, learning from the value-side is cryptographically hard since the value-side is a pseudorandom function $f_{\key}$.

We also need to show that an ICE-learner can learn $\calC$ in the presence of $\eta$-rate malicious noise. Intuitively, a malicious-noise adversary does not get to choose which examples to corrupt, so it can only corrupt the key-side at roughly half the ``efficiency'' of the nasty-noise adversary. Since the key-side is a greater than $\eta$ fraction of the whole domain, the amount of key-side examples introduced by the adversary is a  constant fraction of the amount of clean key-side examples (we ignore the clean examples that are replaced for now). The idea then is to use the key-side examples to assign a presumed value for each bit of the codeword.  Since there are a constant fraction more clean examples than corrupted examples, ideally the number of correct codeword bits would be a constant greater than $1/2$; given this, we can then list-decode, obtaining a short list of candidates for the key, and use hypothesis testing to choose a final high-accuracy hypothesis. A natural method to come up with the presumed value for each codeword bit would be to take the majority vote of the examples for each bit, but this method does not work as the adversary would be able to cause an incorrect value on a bit ``efficiently'' by introducing just one more corrupted example than clean examples. Instead, we use a random procedure $\Round$ to choose a presumed value for each bit of the codeword, in which we are more likely to choose a value if it is consistent with more examples for that bit; see \Cref{remark:keybitguess} for an elaboration of this point. The full details are in \Cref{sec:ICE-bad-news}.

%% file: sections/preliminaries.tex
%!TEX root = ../malicious-nasty-writeup.tex

\section{Definitions and preliminaries} \label{sec:prelims}

Throughout the paper we use {\bf bold font} to denote random variables, which may be real-valued, vector-valued, set-valued, etc. For any (multi)set $S$, we use $\bx \sim S$ as shorthand for $\bx$ being drawn uniformly from $S$.

\subsection{General notation and basic tools from probability}

For $x,y \in \bits^d$ we write $\ham(x,y)$ to denote the number of coordinates $i$ in which $x_i \neq y_i$, and we call $\ham(x, (+1)^d)$ the Hamming weight of $x$. More generally, given $x,y \in \R^d$ we write $\|x-y\|_1$ to denote the $\ell_1$-distance,
\[
\|x-y\|_1 = \sum_{i=1}^d |x_i - y_i|.
\]
Note that if $x,y \in \bits^d$ then $\ham(x,y) = {\frac 1 2} \|x-y\|_1.$

\medskip

Let ${\cal D}_1,{\cal D}_2$ be two distributions over the reals.  Recall that ${\cal D}_1$ is \emph{stochastically dominated} by ${\cal D}_2$ if for all $t \in \R$ we have $\Pr[\bx \leq t] \geq \Pr[\by \leq t]$. Equivalently, ${\cal D}_1$ is stochastically dominated by ${\cal D}_2$ if there is a coupling of $\bx \sim {\cal D}_1$, $\by \sim {\cal D}_2$ (i.e. a joint distribution of $(\bx,\by)$ pairs such that the marginal distribution of $\bx$ is equal to ${\cal D}_1$ and the marginal distribution of $\by$ is equal to ${\cal D}_2$) such that $\bx \leq \by$ with probability 1. 
More generally, we say that ${\cal D}_1$ is \emph{$(1-\theta)$-stochastically dominated} by ${\cal D}_2$ if there is a coupling of $\bx \sim {\cal D}_1$, $\by \sim {\cal D}_2$ such that $\bx \leq \by$ with probability $1-\theta.$

\begin{definition} \label{def:dtv}
Let ${\cal D},{\cal D}'$ be two distributions over the same domain $\Omega$. 
The \emph{total variation distance} $\dtv({\cal D},{\cal D}')$ between ${\cal D}$ and ${\cal D}'$ is defined to be
\[
\dtv({\cal D},{\cal D}') :=
\sup_{T: \Omega \to [0,1]}
\cbra{
\Ex_{\bx \sim {\cal D}}[T(\bx)] - \Ex_{\bx' \sim {\cal D}'}[T(\bx')]}.
\]
Alternatively, $\dtv({\cal D},{\cal D}')$ is equal to the infimum, over all couplings of $\bx \sim {\cal D}$ and $\bx' \sim {\cal D}'$, of $\Pr[\bx \neq \bx']$.
\end{definition}

\begin{fact}[Hoeffding's inequality, \cite{Hoeffding:63}]
    \label{fact:hoeffding-inequality}
    Let $\bx_1, \ldots, \bx_k$ be independent random variables each supported on $[0,1]$ and $\bX$ their sum. Then, for $\mu \coloneqq \Ex[\bX]$ and any $t > 0$
    \begin{equation*}
        \Pr[\bX \geq \mu + t] \leq \exp\paren*{-\frac{2t^2}{k}}.
    \end{equation*}
\end{fact}

We will also use the following standard multiplicative Chernoff bound:

\begin{fact}[Multiplicative Chernoff bound]
    \label{fact:multiplicative-Chernoff-bound}
    Let $\bx_1, \ldots, \bx_k$ be independent random variables each supported on $[0,1]$ and $\bX$ their sum. Then, for $\mu \coloneqq \Ex[\bX]$ and any $0<\delta<1,$
    \begin{align*}
\Pr[\bX \geq (1+\delta)\mu] &\leq \exp\pbra{{\frac {-\delta^2 \mu}{2+\delta}}}\\
\Pr[\bX \leq (1-\delta)\mu] &\leq \exp\pbra{{\frac {-\delta^2 \mu}{2}}}.
    \end{align*}
\end{fact}

\subsection{Learning}

Throughout this paper, a \emph{learning algorithm} is an algorithm $A$ that
\begin{itemize}

\item takes as input a sequence $S = (x_1,y_1),\dots,(x_n,y_n)$ of \emph{labeled examples}, where each example $x_i$ belongs to the \emph{domain} $X$ and each $y_i \in \bits$ is the binary \emph{label} of the $i$-th example; and

\item returns a (representation of a) \emph{hypothesis} $h: X \to \bits.$

\end{itemize}

For simplicity, we will always consider the setting in which $X=\bits^d$, where $d$ is an asymptotic parameter. We say that a learning algorithm is ``efficient'' if its sample complexity $n=n(d)$ of $A$ and its running time $T=T(d)$ are both polynomial in $d$ (and possibly other parameters).

\subsubsection{Noise-free PAC learning.} \label{sec:noise-free-PAC} We begin by recalling the definition of standard (noise-free) PAC learning, since the noise models that we are interested in build on this:

\begin{definition} [Noise-free PAC learning] \label{def:PAC}
An \emph{$(\eps,\delta)$-distribution-free PAC learning algorithm} $A$ for a \emph{concept class} ${\cal C}$ (a class of functions from $X$ to $\bits$) \emph{that uses $n$ examples} is a learning algorithm with the following property:
For any (unknown) \emph{target concept} $c \in {\cal C}$, for any (unknown) distribution ${\cal D}$ over $X$, if the input to algorithm $A$ is a sequence $\bS$ of $n=n(\eps,\delta)$ labeled examples $(\bx,\by)$ where each $\bx$ is an i.i.d.~draw from ${\cal D}$ and each $\by=c(\bx)$, then with probability at least $1-\delta$ the hypothesis $h$ that $A$ returns has \emph{error rate} at most $\eps$, i.e.~
\[
\error_{\mcD}(h, c) \leq \eps, \quad \text{where} \quad
\error_{\mcD}(h, c):= \Prx_{\bx \sim \mcD}[h(\bx) \neq c(\bx)].
\]
If a learning algorithm $A$ only satisfies the above guarantee for some particular distribution ${\cal D}$ (such as the uniform distribution over $X$), then we say that $A$ is a \emph{fixed-distribution $(\eps,\delta)$-PAC learning algorithm for distribution ${\cal D}$.}
\end{definition}

\subsubsection{PAC learning with malicious noise.} \label{sec:malicious-PAC}
In the \emph{malicious noise} variant of the PAC learning model described above, the sequence of $n$ labeled examples that is given as input to the learning algorithm $A$ is obtained as follows:  for $i=1,\dots,n$,

\begin{itemize}

\item [(i)] First, an independent choice of $\bx_i$ is drawn from ${\cal D}$.

\item [(ii)] Then, an independent coin with heads probability $\eta$ is flipped. If the coin comes up tails then the $i$-th labeled example is $(\bx_i,c(\bx_i))$, but if the coin comes up heads then that $i$-th labeled example is replaced by an \emph{arbitrary pair} $(x_i,y_i) \in X \times \bits$. (This pair $(x_i,y_i)$ can be thought as being chosen by a malicious adversary who knows the outcome of the previous $i-1$ examples, but does not know the outcome of future independent random events such as the draw of $\bx_{i+1} \sim {\cal D}.$) The value $\eta>0$ is known as the \emph{malicious noise rate}.

\end{itemize}

The requirement for successful $(\eps,\delta)$-PAC learning is unchanged from \Cref{def:PAC}; only the assumption on the data is different. In more detail, we have the following:

\begin{definition} [PAC learning with malicious noise] 
\label{def:pac-learn-malicious}
An algorithm that \emph{$(\eps,\delta)$-PAC learns ${\cal C}$ using $n$ examples in the presence of malicious noise at rate $\eta$} has the following property:  for any unknown target concept $c \in {\cal C}$, for any unknown distribution ${\cal D}$ over $X$, if the input to the learning algorithm is a sequence of  examples generated as described in (i) and (ii) above, then with probability at least $1-\delta$ the hypothesis $h$ that $A$ returns has error rate at most $\eps$ under ${\cal D}$.

Similar to \Cref{def:PAC}, if a learning algorithm $A$ only satisfies the above guarantee for some particular distribution ${\cal D}$, then we say that $A$ is a \emph{fixed-distribution $(\eps,\delta)$-PAC learning algorithm for distribution ${\cal D}$ in the presence of malicious noise at rate $\eta$.}

\end{definition}

Since each noisy example pair $(x,y)$ is arbitrary, the noisy examples could be generated by a malicious adversary; hence the name of the model.  

\begin{remark} [Strong malicious noise] \label{remark:strong-malicious}
Note that in the above model the malicious adversary must come up with the adversarial examples in an online fashion without knowledge of the future. We will sometimes have occasion to consider a variant of this model, which we refer to as \emph{strong malicious noise.}  In the strong malicious noise model, 

\begin{enumerate}

\item There is a first phase in which the entire sequence $(\bx_1,c(\bx_1)),\dots,(\bx_n,c(\bx_n)) $ of noiseless examples is generated, where each $\bx_i$ is an i.i.d.~draw from ${\cal D}.$

\item Then, in a second phase, an independent $\eta$-heads-probability coin is tossed for each $i \in [n]$; let $\bZ_\mal \subseteq [n]$ be the set of those indices $i$ such that the $i$-th coin toss comes up heads.  For each $i \in \bZ_\mal$, the example $(\bx_i,c(\bx_i))$ is replaced by an arbitrary and adversarially chosen example $(x_i,y_i)$. 
\end{enumerate}

The difference from the usual malicious noise model is that in the strong model, the adversarial choice of noisy examples is made with full knowledge of the entire data set that the learner will receive and full knowledge of which examples in that data set can be corrupted.
\end{remark}

The definition of PAC learning with strong malicious noise, both for the distribution-free and fixed-distribution versions, is entirely analogous to \Cref{def:pac-learn-malicious}: the success criterion for successful learning does not change, but now the assumption on the data is that it is generated as described in \Cref{remark:strong-malicious}.

We remark that each of the malicious noise models described above can be viewed as an ``additive'' noise model, in which the adversary is only able to add some noisy data points into a ``clean'' sample of noiseless labeled examples (namely, the examples for which the coin comes up tails). 

\subsubsection{PAC learning with nasty noise.} \label{sec:nasty-PAC}
Nasty noise, which was defined by Bshouty, Eiron and Kushilevitz in \cite{BEK:02}, is a more demanding model than malicious noise.  In the nasty noise model, an adversary can \emph{choose} which examples to corrupt (recall that in the malicious model, the examples to be corrupted are selected via coin toss), as long as the number of points that can be corrupted is drawn from a binomial distribution (as is the case in the malicious noise model). We recall the precise definition:

\begin{definition}[Nasty noise]
    \label{def:variable-nasty-noise}
    In PAC learning with $\eta$-\emph{rate nasty noise}, the input to the $n$-sample learner is obtained via the following process.
    \begin{enumerate}
        \item First, a ``clean" labeled data set of $n$ examples, $(\bx_1, c(\bx_1)), \ldots, (\bx_n, c(\bx_n))$ is drawn, where each $\bx_i$ is an i.i.d. draw from $\mcD$.
        \item Then, the adversary chooses a subset $\bZ \subseteq [n]$ of indices to corrupt. The only requirement of this subset is that the marginal distribution of $|\bZ|$ is equal to $\Bin(n, \eta)$.\footnote{In particular, the distribution of $|\bZ|$ can be a function of the clean sample, meaning the adversary can choose the number of corruptions as a function of the sample.}
        \item For each $i \in \bZ$, the sample $(\bx_i, c(\bx_i))$ is replaced by an arbitrary and adversarially chosen example $(x_i, y_i)$.
    \end{enumerate}
\end{definition}

Similar to before, the definition of PAC learning with nasty noise is entirely analogous to \Cref{def:pac-learn-malicious}: the success criterion for successful learning does not change, but now the assumption on the data is that it is generated according to the nasty process described above.

\begin{remark} [Robustness to having ``too many examples'']
\label{rem:too-many-examples}   We make the following observation: Suppose that $A$ is an algorithm that PAC learns in the presence of (either standard malicious noise; or strong malicious noise; or nasty noise) using $n$ examples.  Then if algorithm $A$ is given a data set of $n' > n$ examples generated according to the relevant noise process, it can simply run on a randomly chosen subset of $n$ examples and disregard the other $n'-n$ examples. It is easy to verify that for each of these three noise models, a random subset of $n$ of the $n'$ examples will correspond to a $n$-example data set generated under the relevant noise model, so the performance guarantee of $A$ will hold.  (Intuitively, this means that for any of these three noise models, being given a larger-than-necessary data set is ``not a problem.'')
\end{remark}

\subsubsection{Learning with expected error} \label{sec:expected-error}

Some of our arguments involving nasty noise will need a notion of learning with expected error, which we now define.
Given a base distribution $\mcD$ and a target concept $c$, let $\mcD_c$ be the resulting distribution over clean labeled examples, $(\bx, c(\bx))$, where $\bx \sim \mcD$.
Given a learning algorithm $A$ which takes in a sample $S$ of $n$ labeled examples, we write ``$A(S) = \bh$'' to indicate that $\bh$ is the hypothesis $A$ generates (note that $\bh$ is a random variable because $A$ may employ internal randomness to generate $h$).
We say that an algorithm $A$ taking in $n$ samples learns a concept class $\mcC$ with \emph{expected error} $\eps$ over distribution $\mcD$ with $\eta$-nasty noise if for every $c \in {\cal C}$ we have
\[
        \Ex_{\bS \sim (\mcD_c
        )^n}\bracket*{\sup_{S'\text{ is a valid corruption of }\bS}
        \cbra{\Prx_{\bx \sim \mcD}[\bh(\bx) \neq c(\bx), \text{ where }\bh=A(S')] }} \leq \eps,
\]
where the inner probability is both over the random draw of $\bx \sim {\cal D}$ and any internal randomness of $A$, and the supremum is over all possible corruptions that an $\eta$ nasty noise adversary can carry out.

We observe the following simple relationships between expected-error PAC learning and $(\eps,\delta)$-PAC learning: (1) since the error rate of any hypothesis is at most 1, any $(\eps,\delta)$-PAC learner is an $(\eps+\delta)$-expected-error PAC learner; and (2) by Markov's inequality, any $\eps\delta$-expected-error PAC learner is an $(\eps,\delta)$-PAC learner.

\subsection{Error-correcting codes}

Throughout this paper we will consider only \emph{binary codes} (i.e.~codes over the alphabet $\bits$).  We require two different types of list-decodable binary codes:  codes that are (efficiently) list-decodable from \emph{adversarial erasures}, and codes that are (efficiently) list-decodable from \emph{adversarial bit-flips}.  Below we recall the basic definitions and known results about the existence of suitable codes of each of these types that we require.

\medskip
\noindent {\bf Erasure list decoding.}  First some notation: for $c \in \bits^n$ and $T \subseteq [n]$, we write $[c]_T$ to denote the projection of $c$ onto the coordinates in $T$.

\begin{definition} [Efficiently erasure-list-decodable binary codes and subcodes]
\label{def:eldbc}
An \emph{$(p,L)$-efficiently erasure-list-decodable binary code $C$} is a subset of $\bits^n$ with the following properties:

\begin{itemize}

\item There is a bijection $\Enc$ from $\bits^k$ onto $C$ (so $|C|=2^k$).  $\Enc$ is called the \emph{encoding function} and $k/n$ is called the \emph{rate} of $C$; a $k$-bit string $x$ that is an input to $\Enc$ is called a \emph{message}, and an element of $C$ is called a \emph{codeword}.

\item For every $r \in \bits^{(1 - p)n}$ and every set $T \subset [n]$ of size $(1 - p)n$, we have that 
\[
|\{c \in C: [c]_T = r\}| \leq L,
\]
i.e.~given any received word in $\{-1,1,?\}^n$  that is obtained from a codeword in $\bits^n$ by (adversarially) erasing at most $pn$ coordinates and replacing them with $?$'s, there are at most $L$ codewords that are consistent with the received word.

\item There is a function $\Dec: \{-1,1,?\}^n \to (\bits^k)^{\leq L}$ with the following property: given as input a received word $c' \in \{-1,1,?\}^n$ that was obtained by  replacing at most $pn$ coordinates of some codeword $c \in C$ with $?$'s, $\Dec(c')$ runs in (deterministic) time $\poly(n,L)$ and outputs the set of all $k$-bit strings $m$ such that $\Enc(m)$ is a codeword in $C$ that is consistent with the received word.

\end{itemize}

Finally, a \emph{subcode} of an \emph{$(p,L)$-efficiently erasure-list-decodable binary code $C$} is simply a subset $C' \subseteq C$.
\end{definition}

\medskip

\noindent {\bf Bit-flip list decoding.} We recall the basic definition of list-decodable binary codes for the bit-flip channel:

\begin{definition} [Efficiently bit-flip-list-decodable binary codes and subcodes]
\label{def:bfldbc}
A \emph{$(p,L)$-efficiently bit-flip-list-decodable binary code $C$} is a subset of $\bits^n$ with the following properties:

\begin{itemize}

\item There is a bijection $\Enc$ from $\bits^k$ onto $C$ (so $|C|=2^k$).  $\Enc$ is called the \emph{encoding function} and $k/n$ is called the \emph{rate} of $C$; a $k$-bit string $x$ that is an input to $\Enc$ is called a \emph{message}, and an element of $C$ is called a \emph{codeword}. 

\item Every Hamming ball of radius $pn$ in $\bits^n$ contains at most $L$ elements of $C.$ (So given any received word in $\bits^n$ that was obtained from a codeword by (adversarially) flipping at most $pn$ coordinates, there are at most $L$ codewords that the received word could have come from).

\item There is a function $\Dec: \bits^n \to (\bits^k)^{\leq L}$ with the following property: given as input a received word $c' \in \bits^n$,
$\Dec(c')$ runs in (deterministic) time $\poly(n,L)$ and outputs the set of all (at most $L$ many) $k$-bit strings $m$ such that $\Enc(m)$ has Hamming distance at most $pn$ from $c'$.

\end{itemize}

\end{definition}

\subsection{Pseudorandom functions}

We require the standard notion of a Boolean
pseudorandom function family (where we think of bits as being values in $\{-1,+1\}$ rather than  $\{0,1\}$). This notion says that no polynomial time algorithm with oracle access to a function, can distinguish between the case where the function was randomly selected from the family,  and a truly random function.  More formally: 

\begin{definition} [Pseudorandom function family (PRF)] \label{def:prf}
Let $F:\{-1,+1\}^\ell \times \{-1,+1\}^d \to \{-1,+1\}$ be an efficiently computable function, where $\ell=\poly(d)$.  We consider the function family 
$\calF = \{f_k: \{-1, +1\}^d \to \{-1, +1\}\}_{k \in \{-1, +1\}^\ell}$
where $f_k$ is the  function $F(k,\cdot)$. 
Let $\calR_d$ be the set of all Boolean functions on $\{-1,+1\}^d$. 
We say that $\calF$ is a pseudorandom function family 
if for every $\poly(d)$-time oracle-algorithm $D$, and every constant $c$, for sufficently large $d$ we have that 
\begin{equation*}
    \left|\Prx_{\bk \sim \{-1,+1\}^\ell}[A^{f_{\bk}}=1]-
    \Prx_{\boldf \sim \calR_d}[A^{\boldf}=1]\right| 
    < \frac{1}{d^c}.
\end{equation*}

\end{definition}

We recall the well-known result \cite{GGM:86,HIL+:99} that pseudorandom function families exist if and only if one-way functions exist (arguably the weakest cryptographic complexity assumption).

\subsection{Extractors}
Below we give the definition of seeded extractors and state their existence. The definition of seeded extractors relies on the notion of $k$-sources. We refer the reader to \cite[Chapter 6]{Vadhan12} for an exposition on extractors.

\begin{definition}[$k$-sources]
    A random variable $\bx$ is a $k$-source if for all $y$ in the support of $\bx$, $\Pr[\bx = y] \le 2^{-k}$.
\end{definition}
\begin{definition}[Seeded extractors]
    A function $\Ext: \{-1, +1\}^n \times \{-1, +1\}^d \to \{-1, +1\}^m$ is a $(k, \varepsilon)$-extractor if for every $k$-source $\bx$ on $\{-1, +1\}^n$, for $\bs$ uniformly sampled from $\{-1, +1\}^d$, the total variation distance between distribution $\Ext(\bx, \bs)$ and the uniform distribution on $\{-1, +1\}^m$ is at most $\varepsilon$.
\end{definition}

\begin{theorem}[{\cite[Theorem 6.36]{Vadhan12}}]\label{thm:seeded-extractor}
    For all positive integers $n \ge k$ and all $\varepsilon > 0$, there is an efficiently-computable $(k, \varepsilon)$-extractor $\Ext: \{-1, +1\}^n \times \{-1, +1\}^d \to \{-1, +1\}^m$ where $m \ge k/2$ and $d = O(\log(n/\varepsilon))$.
\end{theorem}

%% file: sections/distribution-independent.tex
%!TEX root = ../malicious-nasty-writeup.tex

\section{Distribution-independent malicious versus nasty noise: Proof of \texorpdfstring{\Cref{thm:main-distribution-free-informal}}{Theorem~\ref{thm:main-distribution-free-informal}}}
\label{sec:dist-independent}

The main result of this section, \Cref{thm:dist-free-combined}, shows that a \emph{distribution-independent} learner $A$ that succeeds with $\eta$-malicious noise can be upgraded to a distribution-independent learner $A'$ that succeeds with $\eta$-nasty noise. The overhead in sample complexity and runtime of this conversion is only polynomial in the original sample complexity and runtime, polynomial in $1/\eps$, and logarithmic in $1/\delta$, and moreover the accuracy and confidence of the nasty noise learner are very nearly as good as the accuracy and confidence of the original malicious noise learner.

\begin{theorem} [Formal version of \Cref{thm:main-distribution-free-informal}:  Nasty noise is no harder than malicious noise for distribution independent learning]
    \label{thm:dist-free-combined}
    For any $\eps, \delta, \eta > 0$ and concept class $\mcC$ of functions $X \to \bits$, if there is a time-$T$ algorithm $A$ that $(\eps, \delta)$-distribution-free learns $\mcC$ using $n$ samples with $\eta$-malicious noise, for
    \begin{equation*}
    k \coloneqq O\paren*{\frac{\ln (1/\delta)}{\eps^2}} \quad\quad\text{and}\quad\quad
 m \coloneqq O\paren*{\frac{n^4 \log(2|X|)^2}{\eps^4}},
    \end{equation*}
    there is a time-$O(k(m + T))$  algorithm $A'$, using $mk$ samples, that $(1.01(\eps(1-\eta) + \eta), 1.01\delta)$-distribution-free learns $\mcC$ with $\eta$-nasty noise.\footnote{As will be clear from the proof, each occurrence of $1.01$ could be replaced with $1+\tau$ where $\tau$ is any fixed constant greater than 0; we give the proof with $1.01$ for ease of presentation.}
\end{theorem}

The proof of \Cref{thm:dist-free-combined} requires reasoning about a number of intermediate noise models. We summarize each of these steps in \Cref{fig:structure}.

\input{sections/StepsDiagram}

We will execute the proof structure shown in \Cref{fig:structure} in a top-down manner. \Cref{thm:dist-free-combined} follows easily from two different and incomparable statements about upgrading malicious noise learning to nasty noise learning.  The first of these statements is as follows:

\begin{theorem}[Nasty noise is no harder than malicious noise for distribution independent learning, I]
    \label{thm:dist-free-polynomial}
    For any $\eps, \delta, \delta_{\additional}, \eta > 0$ and concept class $\mcC$ of functions $X \to \bits$, if there is a time-$T$ algorithm $A$ that $(\eps, \delta)$-distribution-independently learns $\mcC$ using $n$ samples with $\eta$-malicious noise, for
    \begin{equation*}
        m \coloneqq O\paren*{\frac{n^4 \log(2|X|)^2}{(\delta_{\additional})^4}},
    \end{equation*}
    there is a time-$O(m + T)$ algorithm $A'$, using $m$ samples, that $(\eps(1-\eta) + \eta, \delta + \delta_{\additional})$-distribution-independently learns $\mcC$ with $\eta$-nasty noise.
\end{theorem}
One downside of \Cref{thm:dist-free-polynomial} is that many learners have failure probabilities, $\delta$, that are exponentially small in the sample size. However, \Cref{thm:dist-free-polynomial} can only efficiently convert such learners into learners that have polynomially small failure probabilities. To deal with the regime where the desired failure probability is extremely small, we introduce a second version of \Cref{thm:dist-free-polynomial}:

\begin{theorem}[Nasty noise is no harder than malicious noise for distribution independent learning, II]
    \label{thm:dist-free-boost}
     For any $\eps, \delta, \epsa, \eta,\delta' > 0$ and concept class $\mcC$ of functions $X \to \bits$, if there is a time-$T$ algorithm $A$ that $(\eps, \delta)$-distribution-independently learns $\mcC$ using $n$ samples with $\eta$-malicious noise, for
    \begin{equation*}
        k \coloneqq O\paren*{\frac{\ln (1/\delta')}{(\epsa)^2}} \quad\quad\text{and}\quad\quad m \coloneqq O\paren*{\frac{n^4 \log(2|X|)^2}{(\epsa)^4}}
    \end{equation*}
     there is a time-$O(k(m + T))$ algorithm $A'$, using $mk$ samples, that $(\eps(1-\eta) + \eta + \delta + \epsa, \delta')$-distribution-independently learns $\mcC$ with $\eta$-nasty noise.
\end{theorem}

In \Cref{thm:dist-free-boost}, the time and sample complexity of $A'$ depend only logarithmically on $\delta'$. In exchange, the error of $A'$ deteriorates by $\delta + \epsa$ is worse than that of $A$ by an additive factor of $\delta$ and $\epsa$. We remark that in general, the best error of a learning algorithm is only polynomial in the sample size, and we can choose to make $\epsa$ also polynomial in the sample size. Furthermore, \Cref{thm:dist-free-boost} is useful in the regime where $\delta$ is exponentially small in the sample size, so the additive $\delta$ term is dominated by the other terms.

We will see that the proof of \Cref{thm:dist-free-boost} is an easy consequence of \Cref{thm:dist-free-polynomial} and our general result that shows how to amplify success probabilities of learners even in the presence of nasty noise, \Cref{thm:boost-no-holdout}.
First, let us show how \Cref{thm:dist-free-combined} follows easily from \Cref{thm:dist-free-polynomial} and \Cref{thm:dist-free-boost}:

\begin{proof}[Proof of \Cref{thm:dist-free-combined}]
Suppose first that $\delta \geq 0.005\eps$.  In this case, we apply \Cref{thm:dist-free-polynomial} with $\delta_{\additional}=0.01\delta$ to get an $(\eps(1-\eta) + \eta,1.01\delta)$-distribution-independent learner in the presence of $\eta$-nasty noise.  Observing that
\[
        m \coloneqq O\paren*{\frac{n^4 \log(2|X|)^2}{(\delta_{\additional})^4}}
        = O\paren*{\frac{n^4 \log(2|X|)^2}{\eps^4}},
\]
the time and sample complexity are as claimed.

The other case is that $\delta < 0.005\eps$. 
In this case, we apply \Cref{thm:dist-free-boost} with $\eps_{\additional}=0.005\eps,\delta'=\delta$ to get a $(1.01(\eps(1-\eta) + \eta),\delta)$-distribution-independent learner in the presence of $\eta$-nasty noise, with time and sample complexity as claimed.
\end{proof}

\subsection{The ingredients needed in the proof of \texorpdfstring{\Cref{thm:dist-free-polynomial,thm:dist-free-boost}}{Theorems~\ref{thm:dist-free-polynomial}~and~\ref{thm:dist-free-boost}}}
In this section, we'll state, without proof, the main ingredients needed to prove \Cref{thm:dist-free-polynomial,thm:dist-free-boost} and show how those theorems follow from the ingredients. We provide proofs of the ingredients in following subsections.

We introduce two additional noise models that play a key role in this proof.
\begin{definition}[PAC learning in Huber's contamination model \cite{Hub64}]
    \label{def:Huber}
    In PAC learning with $\eta$-Huber contamination, the adversary chooses a distribution of outliers $\mcO$ over labeled examples\footnote{Note this distribution is entirely adversarial. In particular, it need not be consistent with any member of the concept class.}. Each sample the learner receives is then i.i.d.~generated as follows: With probability $1-\eta$, the sample is set to a clean point $(\bx, c(\bx))$ for $\bx \sim \mcD$. Otherwise, it is drawn from the outlier distribution, $(\bx, \by) \sim \mcO$.
\end{definition}

We'll show, quite easily, that a learner that succeeds with malicious noise also succeeds with Huber contamination.
\begin{restatable}[Malicious noise is at least as hard as Huber contamination]{proposition}{malHard}
    \label{prop:mal-harder-oblivious}
    For any concept class $\mcC$ of functions $X \to \bits$ and distribution $\mcD$ over $X$, if an algorithm $A$ $(\eps, \delta)$-learns $\mcC$ over distribution $\mcD$ with $\eta$-malicious noise, then it also $(\eps,\delta)$-learns $\mcC$ over distribution $\mcD$ with $\eta$-Huber contamination.
\end{restatable}
 Second, we introduce a noise model that has sometimes been referred to as ``general, non-adaptive, contamination"\cite{DK23book}, though we'll give it a different name that emphasizes the definition.
 \begin{definition}[TV noise, restatement of \Cref{def:TV-noise-intro}]
 \label{def:TV-noise}
    In PAC learning with $\eta$-TV noise, given a base distribution $\mcD$ and concept $c$, let $\mcD_c$ be the resulting distribution over clean labeled examples, $(\bx, c(\bx))$, where $\bx \sim \mcD$. The adversary chooses any $\mcD'$ where
    \begin{equation*}
        \TV(\mcD_c, \mcD') \leq \eta.
    \end{equation*}
    The sample is then generated by taking i.i.d.~draws from $\mcD'$.
\end{definition}
The second ingredient connects TV noise and nasty noise.
\begin{lemma}[TV noise is at least as hard as nasty noise]
    \label{lem:TV-harder-nasty}
    For any $\eps, \delta, \delta_{\additional} > 0$, concept class $\mcC$ of functions $X \to \bits$ and distribution $\mcD$ over $X$, if an algorithm $A$ using $n$ samples $(\eps, \delta)$-learns $\mcC$ over distribution $\mcD$ with $\eta$-TV noise, then, for
    \begin{equation*}
        m \coloneqq O\paren*{\frac{n^4 \log (2|X|)^2}{(\delta_{\additional})^4}}
    \end{equation*}
    and $A'$ being the $m$-sample algorithm which runs $A$ on a uniformly random size-$n$ subsample of its input, $A'$ $(\eps,\delta + \delta_{\additional})$-learns $\mcC$ over distribution $\mcD$ with $\eta$-nasty noise.
\end{lemma}
We show how to prove \Cref{thm:dist-free-polynomial} with these two ingredients.
\begin{proof}[Proof of \Cref{thm:dist-free-polynomial} assuming \Cref{prop:mal-harder-oblivious} and \Cref{lem:TV-harder-nasty}]
    We observe that if $A$ runs in time $T$, the $A'$ in \Cref{lem:TV-harder-nasty} runs in time $O(m + T)$. Therefore, by \Cref{lem:TV-harder-nasty}, it suffices to show that for any distribution $\mcD$ over $X$, $A$ $(\eps(1-\eta) + \eta,\delta)$-learns $\mcC$ over $\mcD$ with $\eta$-TV noise. We recall what this means (\Cref{def:TV-noise}): For any $c \in \mcC$, using $\mcD_c$ to refer to the distribution of $(\bx, c(\bx))$ where $\bx \sim \mcD$, the adversary is allowed to specify any $\mcD'$ for which $\TV(\mcD_c, \mcD') \leq \eta$, and $A$ receives i.i.d. samples from $\mcD'$.

    The main observation of this proof is to show that $\mcD'$ can be viewed as a Huber corruption to $(\mcD^{\add})_c$ where $\mcD^{\add}$ is a distribution over $X$ that is not necessarily equal to $\mcD$. By the definition of TV distance,
there is a coupling of $(\bx, c(\bx)) \sim \mcD_c$ and $(\bx', \by') \sim \mcD'$ for which,
    \begin{equation*}
        \Pr[(\bx, c(\bx)) \neq (\bx', \by')] \leq \eta.
    \end{equation*}
    Let $\be$ be an event that occurs with probability exactly $\eta$ and for which $(\bx, c(\bx)) \neq (\bx', \by')$ can only occur if $\be$ occurs (such an event always exists because $\Pr[(\bx, c(\bx)) \neq (\bx', \by')] \leq \eta$). Let $(\mcD^{\add})_c$ be the distribution over $(\bx, c(\bx))$ conditioned on $\overline{\be}$ and $\mcO$ be the distribution over $(\bx', \by')$ conditioned on $\be$. Then, $\mcD'$ can be written as the mixture,
    \begin{equation*}
        \mcD' = (1 - \eta) \cdot (\mcD^{\add})_c + \eta \mcO.
    \end{equation*}
    Hence, $\mcD'$ can be written as a $\eta$-Huber corruption of $(\mcD^{\add})_c$. Since $A$ is a distribution-independent learner, it should $(\eps,\delta)$ learn with $\eta$-malicious noise even over $\mcD^{\add}$. 
By \Cref{prop:mal-harder-oblivious}, $A$ $(\eps,\delta)$-learns over $\mcD_\add$ with $\eta$-Huber contamination. 
It therefore outputs, with probability at least $1-\delta$, a hypothesis $h$ that has error at most $\eps$ on the distribution $\mcD_{\add}$. We observe that
\begin{equation*}
    \mcD = (1-\eta)\cdot \mcD_{\add} + \eta \cdot \mcD_{\mathrm{remaining}},
\end{equation*}
where $\mcD_{\mathrm{remaining}}$ is the distribution of $\bx$ conditioned on $\be$. This implies that if the error of $h$ on $\mcD_{\add}$ is $\eps$, the error of $h$ on $\mcD$ is at most $\eps(1-\eta) + \eta$.
Therefore, we do have that $A$ $(\eps(1-\eta) + \eta ,\delta)$-learns $c$ given samples from $\mcD'$. The desired result then follows from \Cref{lem:TV-harder-nasty}.
\end{proof}
We now show how \Cref{thm:dist-free-polynomial} implies \Cref{thm:dist-free-boost}. This requires our generic success amplification technique, \Cref{thm:boost-no-holdout}, given in \Cref{sec:amplify}.
\begin{proof}[Proof of \Cref{thm:dist-free-boost}]
    First, we argue that we can, without loss of generality, assume that $\delta < 1/2$. If $\delta \geq 1/2$, it suffices to give a learner achieving $1/2$ error. Such an error is achieved by the trivial (randomized) hypothesis that, on every $x$, outputs $+1$ and $-1$ with equal probability.

    Hence, we assume $\delta < 1/2$. We will apply \Cref{thm:dist-free-polynomial} with $\delta_{\additional} = \epsa$. This gives a learner $A'$ running in time $O(m + T)$ using $m$ samples that $(\eps(1-\eta) + \eta, \delta + \eps_{\additional})$-learns $\mcC$ with $\eta$-nasty noise. The expected error of this $A'$ is at most $\eps(1-\eta) + \eta + \delta + \delta_{\additional}$, since even when $A'$ fails to have $\eps(1-\eta) + \eta$ error, its error is upper bounded by $1$. Therefore, by \Cref{thm:boost-no-holdout}, there is a learner $A''$ using $mk$ samples and running in time $O(k(m+T))$ that $(\eps(1-\eta) + \eta + \delta + 2\epsa, \delta')$ learns $\mcC$ with $\eta$-nasty noise. The desired result follows by renaming $\epsa' = \epsa/2$.
\end{proof}

\subsection{Proof of \texorpdfstring{\Cref{prop:mal-harder-oblivious}}{Proposition~\ref{prop:mal-harder-oblivious}}}
\label{subsec:proof-mal-harder-oblivious}
Here, we provide an easy proof that if a learner succeeds with malicious noise, it also succeeds with Huber contamination.
\begin{proof}[Proof of \Cref{prop:mal-harder-oblivious}]
    This proof follows by showing, for any choices of the Huber adversary, a malicious adversary can create the same distribution over samples. Therefore, if an algorithm succeeds against all malicious adversaries, it must also succeed for all Huber adversaries.

    Consider any concept $c \in \mcC$ and let $\mcD_c$ be the distribution over $(\bx, c(\bx))$ where $\bx \sim \mcD$. The Huber adversary chooses any outlier distribution $\mcO$ over labeled samples and the learning algorithm receives i.i.d.~samples from the mixture $(1-\eta) \cdot \mcD_c + \eta \cdot \mcO$. 
    
    Consider the malicious adversary that, whenever the $\eta$ coin comes up heads, draws $(\bx, \by) \sim \mcO$ as its chosen corrupted pair. Clearly, the resulting distributions over datasets with the original oblivious Huber and this malicious adversary are the same. Therefore, since $A$ is promised to $(\eps, \delta)$-learn with $\eta$-malicious noise, on the samples generated by the Huber adversary, it must with probability at least $1-\delta$ output a hypothesis that has at most $\eps$-error w.r.t. $c$. We conclude that $A$ $(\eps,\delta)$-learns with $\eta$-Huber contamination.
\end{proof}

\subsection{Proof of \texorpdfstring{\Cref{lem:TV-harder-nasty}}{Lemma~\ref{lem:TV-harder-nasty}}}
This proof is split into two pieces. We define a variant of nasty noise, which we refer to as \emph{fixed-rate nasty noise} and use it as an intermediate step. One piece connects TV noise to fixed-rate nasty noise, and the other connects fixed-rate nasty noise to standard nasty noise.

\begin{definition}[Fixed-rate-nasty noise]
    \label{def:fixed-rate-nasty}
    In the $\eta$-fixed-rate nasty noise variant of PAC learning, a sequence of $n$ labeled examples that is given as input to the learning algorithm is obtained as follows:
    \begin{itemize}
        \item [(i$'$)]  First,  a ``clean'' labeled data set of $n$ examples $(\bx_1,c(\bx_1)),\dots,(\bx_n,c(\bx_n))$ is drawn,  where each $\bx_i$ is an i.i.d.~draw from ${\cal D}$.
        \item [(ii$'$)] Then, an arbitrary subset of \emph{exactly} $k=  \floor{\eta n}$ of the $n$ labeled pairs $(\bx_i,c(\bx_i))$ are removed and replaced with $k$ \emph{arbitrary pairs} $(x,y) \in X \times \bits$, and
        the resulting data set is presented to the learning algorithm.
    \end{itemize}
\end{definition}
In order to state the two pieces in the proof of \Cref{lem:TV-harder-nasty}, it will be helpful to have concise notation to denote a uniform subsample.

\begin{definition}[Subsampling filter]
    \label{def:subsampling-filter}
    For any $m \geq n$ we define the \emph{subsampling filter} $\sub_{m \to n}: X^m \to X^n$ as the (randomized) algorithm that given $S \in X^m$, returns a sample of $n$ points drawn uniformly without replacement from $S$.
\end{definition}

We now give the two pieces that, together, easily imply \Cref{lem:TV-harder-nasty}.
\begin{restatable}[TV noise to fixed-rate nasty noise]{proposition}{TVNasty}
    \label{prop:TV-harder-nasty}
    For any concept class $\mcC$ of functions $X \to \bits$ and distribution $\mcD$ over $X$, if an algorithm $A$ using $n$ samples $(\eps, \delta)$-learns $\mcC$ over distribution $\mcD$ with $\eta$-TV noise, then, $A' \coloneqq A \circ \submn$, where
    \begin{equation*}
        m \coloneqq O\paren*{\frac{n^4 \log (2|X|)^2}{(\delta_{\additional})^4}},
    \end{equation*}
    $(\eps,\delta + \delta_{\additional})$-learns $\mcC$ over distribution $\mcD$ with $\eta$-fixed-rate nasty noise.
\end{restatable}

\begin{restatable}[Fixed-rate to standard nasty noise]{proposition}{compareNasty}
    \label{prop:standard-nasty-harder}
    For any concept class $\mcC$ of functions $X \to \bits$ and distribution $\mcD$ over $X$ and $n < m$, let $A'$ be an algorithm of the form $A' = A \circ \submn$ for some algorithm $A$. If $A'$ $(\eps, \delta)$-learns $\mcC$ over distribution $\mcD$ with $\eta$-fixed-rate nasty noise, then, $A'$ $(\eps, \delta + \frac{n}{\sqrt{m}} + \frac{n}{m})$-learns $\mcC$ over $\mcD$ with $\eta$-(standard) nasty noise.
\end{restatable}
\Cref{prop:TV-harder-nasty,prop:standard-nasty-harder} together clearly imply \Cref{lem:TV-harder-nasty}, so all that remains is to prove them.
\subsubsection{Proof of \texorpdfstring{\Cref{prop:TV-harder-nasty}}{Proposition~\ref{prop:TV-harder-nasty}}}

This proof uses a recent result of Blanc and Valiant \cite{BV24} that connects various ``oblivious" statistical adversaries with their ``adaptive" counterparts. What they refer to as ``oblivious" adversaries are those that corrupt the distribution from which the learner receives i.i.d~samples (like TV noise) and ``adaptive" adversaries are  those that corrupt the sample directly (like nasty noise). In reading the following result, we encourage the reader to keep in mind the case where the set of corruptible inputs, $X_{\mathrm{corrupt}}$, consists of the entire domain $X$. In this case, the ``adaptive" adversary corresponds to fixed-rate nasty noise and the ``oblivious" adversary corresponds exactly to TV noise.
\begin{theorem}[Theorem 3 of \cite{BV24}]
    \label{thm:BV-main}
    For any domains $X_{\mathrm{corrupt}}\subseteq X$ with $|X|>1$, (possibly randomized) test function $f:X^n \to [0,1]$, distribution $\mcD$ over $X$, noise budget $\eta\in [0,1]$, $\eps > 0$, let
    \begin{equation*}
        m = O\paren*{\frac{n^4(\ln |X|)^2}{\eps^4}}.
    \end{equation*}
    Then, the following quantities are within $\pm \eps$ of one another.
    \begin{enumerate}
        \item The quantity
        \begin{equation*}
            \Ex_{\bS \sim \mcD^m}\bracket*{\sup_{S'\text{ is a valid corruption of }\bS}\set{\Ex[f \circ \submn(S')]}},
        \end{equation*}
        where $S'$ is a valid corruption of $S$ if the number of coordinates, $i \in [n]$, on which $S_i$ and $S'_i$ differ is at most $\floor{\eta m}$, and when they differ, $S_i \in X_{\mathrm{corrupt}}$.
        \item The quantity,
        \begin{equation*}
            \sup_{\mcD'\text{ is a valid corruption of }\mcD}\set{\Ex_{\bS' \sim (\mcD')^n}[f(\bS')]},
        \end{equation*}
        where $\mcD'$ is a valid corruption of $\mcD$ if there is a coupling of $\bx \sim \mcD$ and $\bx' \sim \mcD$ for which $\Pr[\bx \neq \bx'] \leq \eta$ and, when they differ, $\bx \in X_{\mathrm{corrupt}}$.
    \end{enumerate}
\end{theorem}
We now prove \Cref{prop:TV-harder-nasty}, restated below for the reader's convenience.
\TVNasty*
\begin{proof}
      Fix a concept $c \in \mcC$. For any sample $S \in X^n$, define
    \begin{equation*}
        f(S) \coloneqq \Pr[A(S)\text{ is not $\eps$ close to }c \text{~under~}\mcD],
    \end{equation*}
    where the probability is taken over the randomness of the learner $A$. Consider the augmented domain $X' \coloneqq X \times \bits$ and define:
    \begin{enumerate}
        \item $\mcD_c$ to be the distribution over $(\bx, c(\bx))$ where $\bx \sim \mcD$.
        \item $X_{\mathrm{corrupt}} = X'$.
    \end{enumerate}
    Observe, that in \Cref{thm:BV-main}, if $\mcD'_c$ is a valid corruption of $\mcD_c$, there is a coupling of $\bx \sim \mcD_c$ and $\bx' \sim \mcD'_c$ for which $\Pr[\bx \neq \bx'] \leq \eta$, and hence $\TV(\mcD_c, \mcD_c') \leq \eta$. Therefore, since $A$ $(\eps, \delta)$-learns $\mcC$ with $\eta$-TV
    noise, we have that
    \begin{equation*}
        \sup_{\mcD_c'\text{ is a valid corruption of }\mcD_c}\set{\Ex_{\bS' \sim (\mcD_c
        )^n}[f(\bS')]} \leq \delta.
    \end{equation*}
    Therefore, using that $\log|X'| =\log (2|X|))$
    and the choice of $m$ in \Cref{prop:TV-harder-nasty}, \Cref{thm:BV-main} gives that
    \begin{equation*}
        \Ex_{\bS \sim \mcD_c^m}\bracket*{\sup_{S'\text{ is a valid corruption of }\bS}\set{\Ex[f \circ \submn(S')]}} \leq \delta + \delta_{\additional}.
    \end{equation*}
    $S'$ is a valid corruption of $S$ iff they differ on at most $\floor{\eta m}$ coordinates. This exactly corresponds to the fixed-rate nasty adversary being able to corrupt $S$ to $S'$. Furthermore, on any sample $S'$, $\Ex[f \circ \submn(S')]$ corresponds to the probability that $A'$ fails to return a good hypothesis given the sample $S'$. Therefore, the above guarantees that $A'$ returns a good hypothesis with probability at least $1-(\delta + \delta_{\additional})$, as required.
\end{proof}

\subsection{Proof of \texorpdfstring{\Cref{prop:standard-nasty-harder}}{Proposition~\ref{prop:standard-nasty-harder}}}
We prove the below proposition, restated for the reader's convenience:
\compareNasty*
\begin{proof}
    Fix any concept $c$ and strategy for the (standard) nasty adversary on which $A'$ fails to produce a hypothesis with error at most $\eps$ with probability $\delta'$. We will show there is a corresponding strategy for the fixed-rate nasty adversary in which $A'$ fails in learning with probability at least $\delta' - \frac{n}{\sqrt{m}} - \frac{n}{m}$. The desired result then follows from the assumption that $A'$ fails with probability at most $\delta$ for any fixed-rate nasty adversary.

    Our fixed-rate adversary will do the following.
    \begin{enumerate}
        \item  Recall that, upon receiving a clean $m$-element sample $\bS$, the (standard) nasty adversary chooses a budget $\bz$ such that the marginal distribution of $\bz$ is $\Bin(m,\eta)$.
        \item If this budget $\bz$ is at most $\floor{m \eta}$, the fixed-rate adversary makes exactly the same corruptions as the (standard) nasty adversary makes to $\bS$.
        \item If this budget is more than $\floor{m \eta}$, the fixed-rate adversary makes the first $\floor{m\eta}$ corruptions the (standard) nasty adversary makes.
    \end{enumerate}
    Let $\bS'_{\fixed}$ and $\bS'_{\mathrm{standard}}$ refer to the corrupted samples made by the fixed-rate and standard nasty adversaries respectively, and $\Delta(\bS'_{\fixed}, \bS'_{\mathrm{standard}})$ refer to the number of coordinates on which these two samples differ. With this strategy,
    \begin{equation*}
        \Ex[\Delta(\bS'_{\fixed}, \bS'_{\mathrm{standard}})] \leq \Ex[\max(\bz - \floor{\eta m}, 0)] \leq \sqrt{\Var(\bz)} + \eta m - \floor{\eta m} \leq \sqrt{m} + 1.
    \end{equation*}
     Next, consider two corrupted samples, $S'_{\fixed}$ and $S'_{\mathrm{standard}}$, created by the fixed-rate and standard nasty adversaries respectively. Then, in order for $A'$ to differ on $S'_{\fixed}$ and $S'_{\mathrm{standard}}$, one of the $\Delta$ indices on which $S'_{\fixed}$ and $S'_{\mathrm{standard}}$ differ must be one of the $n$ indices chosen in the random subsample. Each of these $\Delta(S'_{\fixed}, S'_{\mathrm{standard}})$ indices makes it into the subsample with probability $\frac{n}{m}$, so
    \begin{equation*}
        \abs*{\Pr[A'(S'_{\fixed}) \text{ succeeds}] - \Pr[A'(S'_{\mathrm{standard}}) \text{ succeeds}]} \leq \frac{n \cdot \Delta(S'_{\fixed}, S'_{\mathrm{standard}})}{m}.
    \end{equation*}
    Combining this with the prior bound on the expected number of differences between $\bS'_{\fixed}$ and $\bS'_{\mathrm{standard}}$ gives the desired result.
    \end{proof}

\subsection{Applications of \texorpdfstring{\Cref{thm:main-distribution-free-informal}}{Theorem~\ref{thm:main-distribution-free-informal}}}
\label{sec:consequences}

In this subsection we record some consequences of \Cref{thm:dist-free-combined}, the formal version of \Cref{thm:main-distribution-free-informal}, for (i) upgrading \emph{nasty noise} tolerant weak learners to strong learners, and (ii) efficient learning in the presence of nontrivial-rate \emph{nasty noise}. In both of these cases, as described in \Cref{sec:intro}, analogous results were previously known for malicious noise but not for nasty noise, and \Cref{thm:dist-free-combined} lets us extend the results to the nasty noise setting.

\subsubsection{Upgrading weak learners to strong learners in the presence of noise.}
The smooth boosting algorithm described and analyzed in \cite{Servedio:03jmlr} implicitly yields the following result, showing that weak-learners that are tolerant to malicious noise can be upgraded to strong-learners that are tolerant to malicious noise:

\begin{claim} 
[Upgrading malicious-noise-tolerant weak learners to malicious-noise-tolerant strong learners; implicit in \cite{Servedio:03jmlr}]
\label{claim:S03}
Let ${\cal C}$ be a concept class over $\bits^d$, and let $\eta,\gamma$ be such that there is an efficient weak learning algorithm for ${\cal C}$, running in time $T=\poly(d,\log(1/\delta))$, that with probability at least $1-\delta$ achieves advantage $\gamma$ in the presence of malicious noise at rate $\eta$.
Then there is an efficient learning algorithm, running in time $\poly(1/\gamma,1/\eps,d,\log(1/\delta))$ that $(\eps,\delta \cdot \poly(1/\gamma,1/\epsilon))$-distribution-independently learns ${\cal C}$ in the presence of malicious noise at rate $\eta' = \Theta(\eta \eps).$
\end{claim}

By applying \Cref{thm:dist-free-combined} to the conclusion of \Cref{claim:S03}, we immediately get a stronger version of \Cref{claim:S03}, namely that weak learners which can tolerate \emph{malicious noise} can be upgraded to achieve strong learning even in the presence of \emph{nasty noise}.\footnote{Note that the most basic application of \Cref{thm:dist-free-combined} would obtain a version of \Cref{cor:S03} in which the nasty noise learner has error $\eps(1-\eta') + \eta'$ at noise rate $\eta' = \Theta(\eta \eps)$. It's straightforward to deduce the below version by changing the constant hidden by $\Theta(\cdot)$.}

\begin{corollary}
[Upgrading malicious-noise-tolerant weak learners to nasty-noise-tolerant strong learners]
\label{cor:S03}
Let ${\cal C}$ be a concept class over $\bits^d$, and let $\eta,\gamma$ be such that there is an efficient weak learning algorithm for ${\cal C}$, running in time $T=\poly(d,\log(1/\delta))$, that with probability at least $1-\delta$ achieves advantage $\gamma$ in the presence of malicious noise at rate $\eta$.
Then there is an efficient learning algorithm, running in time $\poly(1/\gamma,1/\eps,d,\log(1/\delta))$ that $(\eps,\delta \cdot \poly(1/\gamma,1/\epsilon))$-distribution-independently learns ${\cal C}$ in the presence of nasty noise at rate $\eta' = \Theta(\eta \eps).$
\end{corollary}

We make two remarks about \Cref{cor:S03}.  First, note that since any weak learner that can tolerate nasty noise can also tolerate malicious noise, \Cref{cor:S03} immediately implies the corresponding statement if the weak learner is assumed to tolerate nasty noise. Second, we note that the $\delta \cdot \poly(1/\gamma,1/\epsilon)$ confidence parameter of the learner provided by \Cref{cor:S03} can be upgraded using our efficient success probability amplification for nasty noise, \Cref{thm:boost-no-holdout}; the details are left to the interested reader.

\subsubsection{Efficient learning in the presence of non-trivial nasty noise.}

Let $m$ be the number of examples required to learn a concept class ${\cal C}$ in the original distribution-independent (noiseless) PAC learning model. It is clear that for either malicious noise or nasty noise, if the noise rate $\eta$ is much less than $1/m$, then with high probability in a sample of $m$ examples none of the examples will be noisy, so simply running the learning algorithm for the noiseless scenario will succeed with high probability.
Kearns and Li \cite{KearnsLi:93} showed that for the malicious noise model, it is possible to ``automatically'' handle a somewhat higher noise rate:

\begin{theorem}
[Non-trivial malicious noise tolerance for efficient learning algorithms] \label{thm:KL}
Let ${\cal C}$ be a concept class over $\bits^d$, and suppose that there is a polynomial-time algorithm $A$ that $(\eps/8,1/2)$-learns ${\cal C}$ (with no noise) using $m$ examples.
Then there is an algorithm that $(\eps,1-\delta)$-learns ${\cal C}$ in the presence of malicious noise at a rate $\eta=\Omega(\min(\eps,{\frac {\log m} m}))$, with running time that is polynomial in the running time of $A$ and $\log(1/\delta).$
\end{theorem}

Combining this with \Cref{thm:dist-free-combined}, we obtain the following analogous result for nasty noise:

\begin{theorem}
[Non-trivial nasty noise tolerance for efficient learning algorithms] 
Let ${\cal C}$ be a concept class over $\bits^d$, and suppose that there is a polynomial-time algorithm $A$ that $(\eps/8,1/2)$-learns ${\cal C}$ (with no noise) using $m$ examples.
Then there is an algorithm that $(\eps,1-\delta)$-learns ${\cal C}$ in the presence of nasty noise at a rate $\eta=\Omega(\min(\eps,\frac {\log m} m))$, with running time that is polynomial in the running time of $A$ and $\log(1/\delta).$
\end{theorem}

%% file: sections/StepsDiagram.tex
\tikzstyle{result} = [rectangle, rounded corners, 
minimum width=3cm, 
minimum height=1cm,
text centered, 
text width=5cm, 
draw=black]
\tikzstyle{arrow} = [thick,->,>=stealth]

\begin{figure}[!t]
\label{fig:structure-of-proof}
\begin{tikzpicture}[node distance=2cm]

\node (nastytop) [result] {{\footnotesize{$\eta$-nasty noise,\\$(1.01(\eps(1-\eta) + \eta),1.01\delta)$-learning,
\\$\mathrm{poly}(d,1/\epsilon,\log(1/\delta))$ overhead \\(\Cref{thm:dist-free-combined})}}};

\node (thmII) [result, below of=nastytop, xshift = -5cm, yshift = -0.2 cm] {{\footnotesize{$\eta$-nasty noise,
\\$(1.005(\eps(1-\eta) + \eta)+\delta,\delta)$-learning,
\\$\mathrm{poly}(d,1/\epsilon,\log(1/\delta))$ overhead
\\(\Cref{thm:dist-free-boost})}}};

%\node (amplif) [result, below of=thmII, xshift = -2cm] {Efficient success probability amplification 
%\\(\Cref{thm:boost-no-holdout})};

\node (thmI) [result, below of=thmII, xshift = 5cm, yshift = -0.4cm
] {{\footnotesize{$\eta$-nasty noise, 
\\$(\eps(1-\eta) + \eta,1.01\delta)$-learning, 
\\$\mathrm{poly}(d,1/\delta)$ overhead 
\\(\Cref{thm:dist-free-polynomial})
}}};

\node (fixedrate) [result, below of=thmI, yshift=-0.2cm
] {{\footnotesize{$\eta$-fixed rate nasty noise,
\\$(\eps(1-\eta) + \eta,O(\delta))$-learning,
\\$\poly(d,1/\delta)$ overhead}}};

\node (tv) [result, below of=fixedrate, yshift=0.2cm
%, yshift=-0.5cm
] {{\footnotesize{$\eta$-TV noise \\$(\eps(1-\eta) + \eta,\delta)$-learning}}};

\node (huber) [result, below of=tv,yshift=0.35cm] {{\footnotesize{$\eta$-Huber contamination, 
\\$(\epsilon,\delta)$-learning}}};
\node (malicious) [result, below of=huber, yshift=0.2cm] {{\footnotesize{$\eta$-malicious noise, \\$(\epsilon,\delta)$-learning
\\(initial assumption)}}};

\draw [arrow] (thmII) -- (nastytop);
\draw [arrow] (thmI) -- (thmII) node[midway, left] {{\footnotesize{Success probability 
amplification (\Cref{thm:boost-no-holdout})~~~}}};
%\draw [arrow] (amplif) -- (thmII);
\draw [arrow] (fixedrate) -- (thmI) node[midway,right]{{\footnotesize{\Cref{prop:standard-nasty-harder}}}};
\draw [arrow] (thmI) -- (nastytop);
\draw [arrow] (tv) -- (fixedrate) node[midway,right]{{\footnotesize{\Cref{prop:TV-harder-nasty}} using \cite{BV24}}};;
\draw [arrow] (huber) -- (tv) node[midway, right] {{\footnotesize{Proof of \Cref{thm:dist-free-polynomial}}}} ;
\draw [arrow] (malicious) -- (huber) node[midway, right]{{\footnotesize{\Cref{prop:mal-harder-oblivious}}}};

\end{tikzpicture}
\caption{The structure of the proof of \Cref{thm:dist-free-combined} (formal version of \Cref{thm:main-distribution-free-informal}).}
\label{fig:structure}
\end{figure}
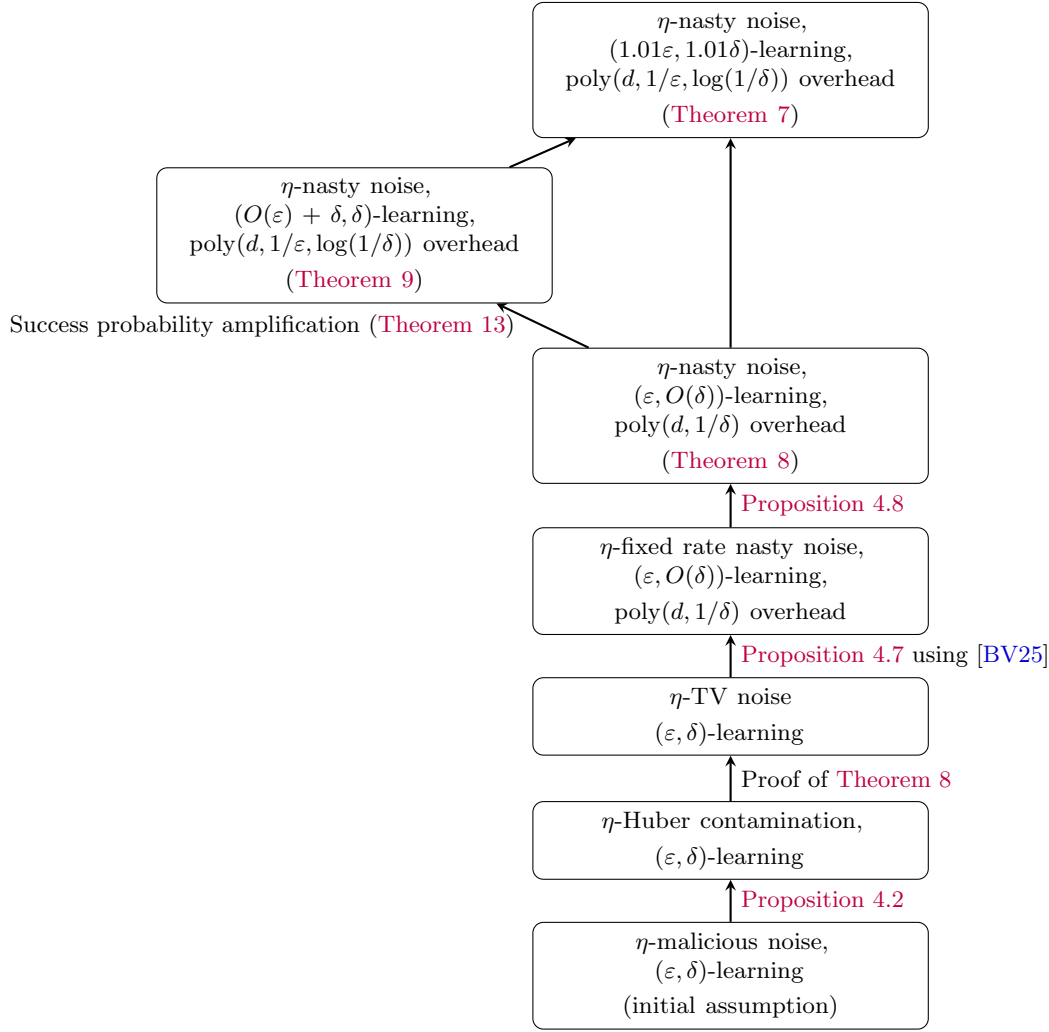

%% file: sections/Boosting.tex
%!TEX root = ../malicious-nasty-writeup.tex

\section{Amplifying the success probability of nasty-noise learners: Proof of \texorpdfstring{\Cref{thm:boost-no-holdout-informal}}{Theorem~\ref{thm:boost-no-holdout-informal}}}
\label{sec:amplify}
In this section, we prove the formal version of \Cref{thm:boost-no-holdout-informal}.

\begin{theorem}[Amplifying the success probability with nasty noise]
    \label{thm:boost-no-holdout}
Given parameters $\epsa,\delta > 0$, let
    \begin{equation*}
        k = O\paren*{\frac{\ln (1/\delta)}{(\epsa)^2}}.
    \end{equation*}
Given any learning algorithm $A$ running in time $T$ and taking in $n$ samples, there is a learner $A'$ (described in \Cref{fig:Boost}) taking $nk$ samples and running in time $O(Tk)$ with the following property: For any $\eps,\eta > 0$, concept class $\mcC$, and distribution $\mcD$, if $A$ learns $\mcC$ with expected error $\eps$ in the presence of $\eta$-nasty noise, then $A'$ $(\eps'\coloneqq \eps + \epsa, \delta)$-learns $\mcC$ with $\eta$-nasty noise.
\end{theorem}

As mentioned earlier, we remark that in \Cref{thm:boost-no-holdout} the new learner $A'$ does not depend on the underlying distribution $\mcD$. Therefore, \Cref{thm:boost-no-holdout} can be used to amplify success probability both for fixed-distribution and distribution-independent learners.

\begin{figure}[h] 
  \captionsetup{width=.9\linewidth}
    
    \begin{tcolorbox}[colback = white,arc=1mm, boxrule=0.25mm]
    \vspace{2pt} 
    \Amplify$(A,k, S_{\mathrm{big}}):$\vspace{6pt}
    
    \textbf{Input:} A learner $A$ taking in a size-$n$ data set, parameter $k \in \N$, and a size-$(nk)$ data set, $S_{\mathrm{big}}$.\vspace{6pt}

    \textbf{Output:} A hypothesis inheriting the accuracy of $A$ and with a $\approx 2^{-k}$ failure probability.\vspace{6pt}

    \begin{enumerate}
            \item[]\textbf{1. Split sample into groups:} Draw a uniform permutation $\bsigma:[nk] \to [nk]$ and define $\bsigma(S_{\mathrm{big}})$ as the size-$nk$ data set satisfying
            \begin{equation*}
                \bsigma(S_{\mathrm{big}})_i = (S_{\mathrm{big}})_{\bsigma(i)}.
            \end{equation*}
            Set the groups $\bS^{(1)}, \bS^{(2)}, \ldots, \bS^{(k)}$ to consist of the first $n$ elements of $\bsigma(S_{\mathrm{big}})$, the second $n$ elements, and so on.\vspace{6pt}
            \item[]\textbf{2. Generate hypotheses:} For each $i \in [k]$, let $\bh_i = A(\bS^{(i)})$. \vspace{6pt}
            \item[]\textbf{3. Combine hypotheses:} Let $\bh$ be the (randomized) hypothesis that, on input $x$, draws a uniform random $\bi \sim \{1,\dots,k\}$ and outputs $\bh_{\bi}(x)$. Output $\bh$.
    \end{enumerate}
    \end{tcolorbox}
\caption{An algorithm for amplifying the success probability of a learner.}
\label{fig:Boost}
\end{figure}

\subsection{The main amplification lemma}
\label{subsec:amplify-description}
The main ingredient underlying \Cref{thm:boost-no-holdout} is a way of generating $k$ candidate hypotheses that inherit the performance guarantees of the base learner.
\begin{lemma}[Generating candidate hypotheses]
    \label{lem:gen-hypotheses}
    For any algorithm $A$ taking in $n$ samples that learns a concept class $\mcC$ with expected error $\eps$ over distribution $\mcD$ with $\eta$-nasty noise, let $\Amplify$ be the algorithm in \Cref{fig:Boost}. Then, for any concept $c$, given a sample that was initially correctly labeled by $c$ and then corrupted by an $\eta$-nasty adversary, with probability at least $1 - \delta$ the $k$ hypothesis that $\Amplify$ generates satisfy
    \begin{equation*}
        \sum_{i \in [k]} \error_{\mcD}(\bh_i, c) \leq \eps k + O\paren*{\sqrt{k \ln(1/\delta)}}.
    \end{equation*}  
\end{lemma}

\paragraph{Proof structure for \Cref{lem:gen-hypotheses}.} It's not hard to show that each hypothesis  $\bh_1, \ldots, \bh_k$ generated in \Cref{fig:Boost} has an expected error of $\eps$. If these hypotheses were independent, then \Cref{lem:gen-hypotheses} would follow by a standard application of Hoeffding's inequality. However, as discussed before, these hypotheses need \emph{not} be independent, because the adversary chooses the corruption as a function of the entire data set of $nk$ points, and therefore the $k$ groups of points in \Cref{fig:Boost} can be dependent on another.

To prove \Cref{lem:gen-hypotheses}, we carefully track what quantities \emph{are} independent of one another. Let $\bS_{\mathrm{big}}$ be the ``clean" sample of $nk$ points and let 
$S_{\mathrm{big}}'$ be the corrupted version of this clean sample which is given as input to \Amplify. Then, let $\bS^{(1)}, \ldots, \bS^{(k)}$ and $(\bS^{(1)})', \ldots, (\bS^{(k)})'$ be the ``clean" and corrupted respectively groups of $k$ points created in \Cref{fig:Boost} (note that the pseudocode in \Cref{fig:Boost} only directly manipulates the corrupted points, but we will track the ``clean" points in the analysis as well). We show that:
\begin{enumerate}
    \item The random variables $\bS^{(1)}, \ldots, \bS^{(k)}$ are independent.
    \item It is possible to corrupt $\bS^{(i)}$ to $(\bS^{(i)})'$ by corrupting $\bz_i$ points in $\bS^{(i)}$ for $\bz_1,\ldots, \bz_k$ that are i.i.d.~according to $\Bin(n,\eta)$.
\end{enumerate}
Note that we make no assumption about the correlation structure between $\bS^{(i)}$ and $\bz_i$. The key quantity we define is $f(S, z)$, which captures the maximum expected error of $A$ (here the expectation is over internal randomness of $A$) given a sample $S'$ that is formed by corrupting $\leq z$ points of $S$. Intuitively, we think of the adversary as choosing a coupling between $\bS^{(1)}, \ldots, \bS^{(k)}$ and $\bz_1,\ldots, \bz_k \iid \Bin(n,\eta)$ to maximize $\sum_{i \in [k]} f(\bS^{(i)}, \bz_i)$. 

Note that since $A$ has expected error $\leq \eps$ in the presence of nasty noise, we know that for any given $i \in [k]$, there is no coupling of $\bS^{(i)}$ and $\bz_i$ that makes $\Ex[f(\bS^{(i)}, \bz_i)] > \eps$. Our goal is to use this to show that there is no coupling over all $k$ groups the adversary could make which causes $\sum_{i \in [k]} f(\bS^{(i)}, \bz_i)$ to often be large. We show this by contrapositive.
\begin{lemma}[Transfer of couplings]
    \label{lem:coupling-transfer}
    For $\eps, \delta > 0$, $k \in \N$, distributions $\mcA$ and $\mcB$ over domains $X_A$ and $X_B$ respectively, and function $f:X_A \times X_B \to [0,1]$, suppose there exists a coupling of $\ba_1, \ldots, \ba_k \iid \mcA$ and $\bb_1, \ldots, \bb_k \iid \mcB$ for which,
    \begin{equation*}
        \Pr\bracket*{\sum_{i \in [k]} f(\ba_i, \bb_i) \geq  \eps k + \sqrt{2k \ln(1/\delta)}} \geq \delta.
    \end{equation*}
     Then, there is a coupling of $\ba \sim \mcA$ and $\bb \sim \mcB$ for which $\Ex[f(\ba,\bb)] \geq \eps$.
\end{lemma}
The statement of \Cref{lem:coupling-transfer} is a bit surprising because $\delta$ can be much smaller than $\eps$. For instance, it could be the case that $\delta = 2^{-0.1k}$, while $\eps = 0.1$. In this setting, \Cref{lem:coupling-transfer} transfers a coupling over $k$-tuples in which an ``interesting" thing only happens with tiny probability, to a coupling over single draws with a large (constant-sized) expectation, i.e.~one in which something ``interesting" happening with constant probability.

An intuition for why a statement like \Cref{lem:coupling-transfer} can hold in the context of this example is that if $\ba_1, \ldots, \ba_k \iid \mcA$ are collectively  ``interesting" with probability $2^{-0.1k}$, then a draw $\ba \sim \mcA$ should be a member of some ``interesting" group with a constant probability. For example, if only half of the probability mass of $\mcA$ is on members of an ``interesting" group, then, only $(2^{-k})$-fraction of the probability mass of $\mcA^k$ is ``interesting." \Cref{lem:coupling-transfer} takes advantage of this exponential difference.

\subsection{Transferring of couplings: Proof of \texorpdfstring{\Cref{lem:coupling-transfer}}{Lemma~\ref{lem:coupling-transfer}}}

The key ingredient in the proof is the following technical result.
\begin{proposition}
    \label{prop:TV-small}
    For any distribution $\mcD$ and $k \in \N$, let $\bx_1, \ldots, \bx_k \iid \mcD$ and let $\be$ be any event (coupled with the draw of $\bx_1, \ldots, \bx_k$) that occurs with probability $\delta$. For $\mcD_{\be}$ the distribution of a uniform element of $\{\bx_1,\dots,\bx_k\}$ conditioned on the event $\be$ occurring, we have
    \begin{equation*}
        \TV(\mcD,\mcD_{\be}) \leq \sqrt{\frac{\ln(1/\delta)}{2k}}.
    \end{equation*}
\end{proposition}
Our proof of \Cref{prop:TV-small} will use Hoeffding's lemma.
\begin{fact}[Hoeffding's lemma \cite{Hoeffding:63}]
    \label{fact:Hoeffding's lemma}
    Let $\bx \in [0,1]$ be a random variable with mean $\mu$. Then, for any $\lambda \in \R$,
    \begin{equation*}
        \Ex\bracket*{\exp(\lambda(\bx - \mu)) \leq \exp\pbra{\frac{\lambda^2}{8}}}
    \end{equation*}
\end{fact}
\begin{proof}[Proof of \Cref{prop:TV-small}]
    By \Cref{def:dtv}, using $X$ to denote the domain of $\mcD$, we have
    \begin{equation*}
        \TV(\mcD,\mcD_{\be}) = \sup_{f:X \to [0,1]}\set*{\Ex\bracket*{\frac{1}{k} \sum_{i \in [k]}f(\bx_i)\mid \be } - \Ex\bracket*{\frac{1}{k} \sum_{i \in [k]}f(\bx_i)}}.
    \end{equation*}
    Fix any such $f$ and let $\mu = \Ex_{\bx \sim \mcD}[f(\bx)]$. Then, for any $\lambda > 0$, we can bound,
    \begin{align*}
        \Ex\bracket*{\frac{1}{k} \sum_{i \in [k]}f(\bx_i) - \mu\,\bigg|\, \be} &\leq \frac{\ln\paren*{\Ex\bracket*{\sum_{i \in [k]}\exp\paren*{(\lambda/k) \cdot (f(\bx_i ) - \mu)} \,\big|\,\be}}}{\lambda} \tag{Jensen's inequality}\\
        &\leq \frac{\ln\paren*{(1/\delta) \cdot \Ex\bracket*{\sum_{i \in [k]}\exp\paren*{(\lambda/k) \cdot (f(\bx_i) - \mu)}}}}{\lambda}  \tag{$\Pr[\be] = \delta$}\\
        &\leq \frac{\ln\paren*{(1/\delta) \cdot \Ex_{\bx \sim \mcD}\bracket*{\exp\paren*{(\lambda/k) \cdot (f(\bx) - \mu)}}^k}}{\lambda}  \tag{$\bx_1, \ldots, \bx_k \iid \mcD$}\\
        &\leq \frac{\ln\paren*{(1/\delta) \cdot \exp\paren*{\frac{\lambda^2}{8k}}}}{\lambda}  \tag{Hoeffding's lemma, \Cref{fact:Hoeffding's lemma}} \\
        &= \frac{\ln(1/\delta)}{\lambda} + \frac{\lambda}{8k}.
    \end{align*}
    Setting $\lambda = \sqrt{8 \ln(1/\delta)k}$, we can conclude that, for any $f:X\to [0,1]$,
    \begin{equation*}
        \Ex\bracket*{\frac{1}{k} \sum_{i \in [k]}f(\bx_i )\mid \be} - \Ex\bracket*{\frac{1}{k} \sum_{i \in [k]}f(\bx_i)} \leq \sqrt{\frac{\ln(1/\delta)}{2k}}. \qedhere
    \end{equation*}
\end{proof}

We are now ready to prove the main result of this subsection.
\begin{proof}[Proof of \Cref{lem:coupling-transfer}]
    Let $\be$ be the event occurring with probability at least $\delta$ indicating that $\sum_{i \in [k]} f(\ba_i, \bb_i) \geq  \eps k + \sqrt{2k \ln(1/\delta)}$.
Set $\mcA_{\be}$ ($\mcB_{\be}$, respectively) to be the distribution of $\ba_{\bi}$ where $\bi$ is uniform over $\{1,\dots,k\}$ (the distribution of $\bb_{\bi}$, respectively), conditioned on event $\be$ occurring. Then, by \Cref{prop:TV-small},
    \begin{equation*}
        \TV(\mcA, \mcA_{\be}) \leq \sqrt{\frac{\ln(1/\delta)}{2k}} \quad\quad\text{and}\quad\quad \TV(\mcB, \mcB_{\be}) \leq \sqrt{\frac{\ln(1/\delta)}{2k}}.
    \end{equation*}
    Let us condition on the event $\be$ occurring. Under this conditioning, it is possible to draw $\ba_{\mathrm{fresh}} \sim \mcA$ and $\bb_{\mathrm{fresh}} \sim \mcB$ coupled with $\ba_{\bi}$ and $\bb_{\bi}$ for which
    \begin{equation*}
        \Pr[\ba_{\mathrm{fresh}} \neq \ba_{\bi}] \leq\sqrt{\frac{\ln(1/\delta)}{2k}} \quad\quad\text{and}\quad\quad \Pr[\bb_{\mathrm{fresh}} \neq \bb_{\bi}] \leq\sqrt{\frac{\ln(1/\delta)}{2k}}.
    \end{equation*}
    Then, since we are conditioning on the event $\be$, by the hypothesis of \Cref{lem:coupling-transfer} we have
    \begin{equation*}
        \Ex[f(\ba_{\bi},\bb_{\bi})] \geq \frac{\eps k + \sqrt{2k \ln(1/\delta)}}{k}.
    \end{equation*}
    Then, since $f$ is bounded on $[0,1]$,
    \begin{align*}
        \Ex[f(\ba_{\mathrm{fresh}}, \bb_{\mathrm{fresh}})] &\geq \Ex[f(\ba_{\bi},\bb_{\bi})] - \Pr[\ba_{\mathrm{fresh}} \neq \ba_{\bi} \text{ or }\bb_{\mathrm{fresh}} \neq \bb_{\bi}] \\
        &\geq \Ex[f(\ba_{\bi},\bb_{\bi})] - 2\sqrt{\frac{\ln(1/\delta)}{2k}} \tag{Union bound} \\
        &\geq \eps.
    \end{align*}
    This is exactly the coupling whose existence we set out to show.
\end{proof}

\subsection{Generating candidate hypotheses: Proof of \texorpdfstring{\Cref{lem:gen-hypotheses}}{Lemma~\ref{lem:gen-hypotheses}}}

Before proving \Cref{lem:gen-hypotheses}, we prove the following simple claim about the distribution of groups of samples.
\begin{proposition}
    \label{prop:group-dists}
    For any distribution $\mcD$ and concept $c$, let $\mcD_c$ be the distribution of $(\bx, c(\bx))$ where $\bx \sim \mcD$. Suppose the algorithm $\Amplify$ in \Cref{fig:Boost} is given a sample $\bS'$ of $nk$ points generated by an $\eta$-nasty adversary where the clean distribution is $\mcD_c$, and let $(S^{(1)})', \ldots, (S^{(k)})'$ refer to the $k$ groups of samples created in Step~1 of \Cref{fig:Boost}. There exists random variables $\bS^{(1)}, \ldots, \bS^{(k)}$ and $\bz_1, \ldots, \bz_k$ with the following properties:
    \begin{enumerate}
        \item The marginal distribution of $\bS^{(1)}, \ldots, \bS^{(k)}$ is that of $k$ independent draws from $(\mcD_c)^n$.
        \item The marginal distribution of $\bz_1, \ldots, \bz_k$ is that of $k$ independent draws from $\Bin(n, \eta)$.
        \item With probability $1$, for all $i \in [k]$, the number of coordinates on which $\bS^{(i)}$ and $(S^{(i)})'$ differ is at most $\bz_i$.
    \end{enumerate}
\end{proposition}
\begin{proof}
    We begin by recalling the process by which the nasty adversary generates the input to the algorithm (cf.~\Cref{def:variable-nasty-noise}). First, a clean sample $\bS_{\mathrm{big}} \sim (\mcD_c)^{nk}$ is drawn. Then, the adversary chooses some $\bZ \subseteq [nk]$ with the only constraint being that the marginal distribution of $|\bZ|$ is $\Bin(nk, \eta)$. Finally, the adversary can arbitrarily corrupt the coordinates of $\bS_{\mathrm{big}}$ in $\bZ$ to form $S'_{\mathrm{big}}$ which is given to $A$.

    Then, $\Amplify$ generates $(\bS^{(1)})', \ldots, (\bS^{(k)})'$ by drawing a uniform permutation $\bsigma:[nk] \to [nk]$ and setting $(\bS^{(1)})'$ to the first $n$ samples in $\bsigma(\bS'_{\mathrm{big}})$, $(\bS^{(2)})'$ to the next $n$ samples in $\bsigma(\bS'_{\mathrm{big}})$, and so on. In this analysis, we set,
    \begin{enumerate}
        \item $\bS^{(1)}$ to be the first $n$ elements of $\bsigma(\bS_{\mathrm{big}})$, $\bS^{(2)}$, the next $n$, and so on.
        \item $\bz_1$ to be the number of $i \in \bZ$ for which $\sigma(i)$ is between $1$ and $n$, $\bz_2$ the number of $i \in \bZ$ for which $\sigma(i)$ is between $n+1$ and $2n$, and so on.
    \end{enumerate}
    Now we prove that all three claimed properties hold. For the first, we observe that applying a random permutation (or indeed any permutation) to points drawn iid does not change their distribution, so the distribution of $\bsigma(\bS_{\mathrm{big}})$ is just that of $nk$ independent samples from $\mcD_c$. It therefore has the right distribution after being split into $k$ groups.

    For the second property, for all $i \in [nk]$, let $\ba_i$ be the indicator that there is a $j \in \bZ$ for which $\bsigma(j) = i$. We will show that the distribution of $\ba_1, \ldots, \ba_{nk}$ is that of $nk$ independent draws from $\Ber(\eta)$, which implies that $(\bz_1, \ldots, \bz_k)$ has the right marginal distribution. To show this, observe that for any $a \in \zo^{nk}$, writing $w(a) \coloneqq \sum_{i \in [nk]} a_i$, we have that
    \begin{align*}
        \Pr[\ba = a] &= \Pr_{\bZ}[\Pr[\ba = a]\mid\bZ] \\
        &= \Pr_{\bZ}\bracket*{\frac{\Ind[|\bZ| = w(a)]}{\binom{nk}{w(a)}}} \tag{$\bsigma$ is a uniform permutation}\\
        &= \eta^{w(a)}(1-\eta)^{nk - w(a)} \tag{$|\bZ| \sim \Bin(nk,\eta)$}
    \end{align*}
    which exactly corresponds to $\ba_1, \ldots, \ba_{nk}$ being $nk$ independent draws from $\Ber(\eta$).

    The final property follows from how we defined $\bz_1, \ldots, \bz_k$ and the fact that $(\bS_{\mathrm{big}})_j$ can only differ from $(\bS'_{\mathrm{big}})_j$ if $j \in \bZ$. 
\end{proof}

We are ready to prove the main result of this subsection.
\begin{proof}[Proof of \Cref{lem:gen-hypotheses}]
    For any size-$n$ sample $S$ and number of corruptions $z$, let
    \begin{equation*}
        f(S, z) = \sup_{S'\text{ differs from S on }\leq z\text{ coordinates}}\set*{\Ex[\error_{\mcD}(A(S'), c)]},
    \end{equation*}
    where the randomness in the expectation is over the randomness of the learner $A$.  Since we are given that $A$ learns with expected error $\eps$ even with $\eta$-nasty noise, we know that for any coupling of $\bS \sim (\mcD_{c})^n$ and $\bz \sim \Bin(n, \eta)$, we have $\Ex[f(S,z)] < \eps$. Now, let $\bS^{(1)}, \ldots, \bS^{(k)}$ and $\bz_1, \ldots, \bz_k$ be as in \Cref{prop:group-dists}. By the contrapositive formulation of \Cref{lem:coupling-transfer}, we have
    \begin{equation*}
        \Pr\bracket*{\sum_{i \in k} f(\bS^{(i)}, \bz_i) \geq  \eps k + \sqrt{2k \ln(2/\delta)}} < \frac{\delta}{2}.
    \end{equation*}
    Then, defining,
    \begin{equation*}
        \boldsymbol{\mathrm{err}}_i \coloneqq \error_{\mcD}(\bh_i, c)],
    \end{equation*}
  (the bold font on $\bh_i$ is because it depends on the randomness of $\bS^{(i)}$ and $\bz_i$)
    we have that $\boldsymbol{\mathrm{err}}_i$ is a random variable bounded on $[0,1]$ with mean at most $f(\bS^{(i)}, \bz_i)$. By Hoeffding's inequality (\Cref{fact:hoeffding-inequality}),
    \begin{equation*}
        \Pr\bracket*{\sum_{i \in k} \boldsymbol{\mathrm{err}}_i \geq \sum_{i \in [k]} f(\bS^{(i)}, \bz_i) + \sqrt{\frac{\ln(2/\delta)k}{2}}} \leq \delta/2.
    \end{equation*}
    Combining these with a union bound, we have that
    \begin{equation*}
        \Pr\bracket*{\sum_{i \in k} \boldsymbol{\mathrm{err}}_i\geq \eps k + \sqrt{\frac{9\ln(2/\delta)k}{2}}} < \delta. \qedhere
    \end{equation*}
\end{proof}

\subsection{Using the candidate hypotheses: Proof of \texorpdfstring{\Cref{thm:boost-no-holdout}}{Theorem~\ref{thm:boost-no-holdout}}}

\begin{proof}[Proof of \Cref{thm:boost-no-holdout}]
    Based on how we set $k$ in \Cref{thm:boost-no-holdout} and the guarantee of \Cref{lem:gen-hypotheses}, with probability at least $1-\delta$, the hypotheses generated have total error at most
    \begin{equation*}
        \eps k + O\paren*{\sqrt{k \ln(1/\delta)}} \leq k\cdot(\eps + \epsa).
    \end{equation*}
    The hypothesis $\Amplify$ outputs has error equal to the average error of the $k$ hypotheses. Using the above bound, this is at most $\eps + \epsa$ with probability at least $1 - \delta$.
\end{proof}

%% file: sections/strong-separation-malicious-nasty.tex
%!TEX root = ../malicious-nasty-writeup.tex

\section{A strong separation between malicious and nasty noise for fixed-distribution learning: Proof of \texorpdfstring{\Cref{thm:fixed-separation-informal}}{Theorem~\ref{thm:fixed-separation-informal}}} \label{sec:strong-sep}

Our main result in this section is \Cref{thm:fixed-separation}, which is a detailed version of the informal \Cref{thm:fixed-separation-informal} from the Introduction. This theorem shows that if we consider general fixed-distribution learning algorithms that need not be ICE algorithms, then in contrast with \Cref{thm:ICE-malicious-nasty}, there can be an arbitrarily large constant factor separation between even the strong malicious noise rate (\Cref{remark:strong-malicious}) and the nasty noise rate that efficient algorithms can tolerate, which of course implies the same separation between \emph{standard} malicious noise and nasty noise.
\begin{theorem} 
[Uniform-distribution learning with low-rate nasty noise can be hard even when uniform-distribution learning with high-rate strong malicious noise is easy]
\label{thm:fixed-separation}
For any constant ratio $r > 1$, there exist two noise rates $0 < \eta_N < \eta_M < 0.1$ that satisfy $\eta_M \ge r\eta_N$ and the following. 
If one-way functions exist, then for all sufficiently large $d$ (relative to $\eta_N,\eta_M$), there is a concept class ${\cal C} = {\cal C}_{\eta_N,\eta_M}$ over $\{-1, +1\}^d$ such that the following uniform-distribution learning results hold for ${\cal C}$:

\begin{itemize}

\item There is no $\poly(d)$-time uniform-distribution PAC learning algorithm for ${\cal C}$ in the presence of nasty noise at rate $\eta_N$ that achieves error rate  
with probability even as large as any $1/\poly(d)$; but

\item There is a $\poly(d/\eps)$-time uniform-distribution PAC learning algorithm for ${\cal C}$ in the presence of strong malicious noise at rate $\eta_M$ that achieves error rate $4\eta_M(1 + o(1)) + \eps$ 
with probability $1 - \exp(-\Omega(d))$.

\end{itemize}

\end{theorem}

In our proof of \Cref{thm:fixed-separation}, we will let $\eta_N = 2^{-8r}$ and $\eta_M = H(\eta_N)/(H(\eta_N) + 2)$ where $H(\eta_N) = -\eta_N\log_2(\eta_N) - (1 - \eta_N)\log_2(1 - \eta_N)$ is the binary entropy function. We have
\[\eta_M \le \frac{H(\eta_N)}2 < -\eta_N\log_2(\eta_N) < 2^{-4r} < 0.1\]
and
\[\eta_M \ge \frac{\eta_M}{2(1 - \eta_M)} = \frac{H(\eta_N)}4 \ge \frac{-\eta_N\log_2(\eta_N)}4 \ge 2r\eta_N.\]

\subsection{Erasure-list-decodable code}
In our construction of the concept class $\calC_{\eta_N, \eta_M}$ in \Cref{thm:fixed-separation}, we will make use of a low-hamming-weight subcode of an efficiently erasure-list-decodable binary code. Formally, we need the following lemma.
(Note that our choice of $\eta_N,\eta_M$ specified above satisfies the conditions of the lemma.)

\begin{lemma}\label{lem:eeldbc-with-hamming}
    For any constant $\eta_N, \eta_M$ such that $0 < \eta_N < \eta_M < 1 - 1/(1 + H(\eta_N)) < 0.5$, there exist constants $\rho, \tau, \lambda \in (0, 1)$ such that for sufficiently large $w$, there exists a $(\tau, \Theta(1))$-efficiently erasure-list-decodable binary code mapping $(\rho w)$-bit messages to $w$-bit codewords, such that at least $2^{\lambda w}$ many codewords have Hamming weight 
    at most $\zetaN w$, and moreover,
    \[\frac{\eta_M}{(1 - \eta_M)\tau} < 0.002 + 0.998 \cdot \frac{\eta_M}{(1 - \eta_M)H(\eta_N)} < 1.\]
\end{lemma}

The proof of
\Cref{lem:eeldbc-with-hamming}  will use \Cref{lem:ding-jin-xing} combined with \Cref{lem:random-code-hamming}.
\begin{lemma}[{\cite[Theorem 3.3]{DingJX14}}]\label{lem:ding-jin-xing}
    For every constant $\rho \in (0, 1)$ and $\xi \in (0, 1 - \rho)$, for sufficiently large $w$, with probability at least $1 - 2^{-\Omega(w)}$, a random binary linear code of mapping $(\rho w)$-bit messages to $w$-bit codewords is $(1 - \rho - \xi, 2^{O(1/\xi)})$-efficiently erasure-list-decodable.
\end{lemma}

\begin{lemma}\label{lem:random-code-hamming}
    For every constant $\zetaN \in (0, 1/2)$ and $\rho \in (1 - H(\zetaN), 1)$, with constant probability, a random binary linear code mapping $(\rho w)$-bit messages to $w$-bit codewords satisfies that at least $2^{(\rho + H(\zetaN) - 1 - o(1))w}$ many codewords have Hamming weight at most $\zetaN w$.
\end{lemma}
We use the Paley-Zygmund inequality in the proof of \Cref{lem:random-code-hamming}:
\begin{fact}[Paley-Zygmund inequality]
    Let $\bZ \ge 0$ be a random variable, $\theta \in [0, 1]$, then
    \[\Pr[\bZ > \theta\Ex[\bZ]] \ge (1 - \theta)^2\frac{\Ex[\bZ]^2}{\Ex[\bZ^2]}.\]
\end{fact}
\begin{proof}[Proof of \Cref{lem:random-code-hamming}]
The proof is essentially a calculation; the main idea is that since any two distinct code words in a random linear code are independent, we can exactly compute $\E[\bZ^2]$, where $\bZ$ is the number of ``light'' codewords in the random linear code, and this lets us apply Paley-Zygmund.

We proceed to the details. Let $\bC$ be a random linear code mapping $(\rho w)$-bit messages to $w$-bit codewords. Suppose $\bC_i$ is the codeword in $\bC$ corresponding to the message $e_i$ whose $i$-th bit is $-1$ and other bits are $+1$, then the $\bC_i$'s are independent and uniformly random in $\{-1, +1\}^w$. 
    
    Let $\alpha$ be the fraction of $w$-bit strings with Hamming weight at most $\zetaN w$, so $\alpha = \frac1{2^w}\sum_{i = 0}^{\zetaN w}\binom wi = 2^{-(1 - H(\zetaN) + o(1))w}$. Let $\bZ$ be the number of codewords in $\bC$ with Hamming weight at most $\zetaN n$. We have
    \[\bZ = \sum_{S \subseteq [\rho w]}\Indicator\left[\ham\left(\bigoplus_{i \in S}\bC_i, (+1)^w\right) \le \zetaN w\right].\]
    Therefore,
    \[\Ex[\bZ] = \sum_{S \subseteq [\rho w]}\Pr\left[\ham\left(\bigoplus_{i \in S}\bC_i, (+1)^w\right) \le \zetaN w\right] = (2^{\rho w} - 1)\alpha + 1.\]
    and
    \[\begin{split}
        \Ex[\bZ^2] ={}& \sum_{S_1, S_2 \subseteq [\rho w]}\Pr\left[\left(\ham\left(\bigoplus_{i \in S_1}\bC_i, (+1)^w\right) \le \zetaN w\right) \wedge \left(\ham\left(\bigoplus_{i \in S_2}\bC_i, (+1)^w\right) \le \zetaN w\right)\right]\\
        ={}& \sum_{\begin{subarray}{c}
                S_1, S_2 \subseteq [\rho w]\\
                S_1 \ne S_2            \end{subarray}}
            \Pr\left[\left(\ham\left(\bigoplus_{i \in S_1}\bC_i, (+1)^w\right) \le \zetaN w\right) \wedge \left(\ham\left(\bigoplus_{i \in S_2}\bC_i, (+1)^w\right) \le \zetaN w\right)\right] + \\
        &\sum_{\begin{subarray}{c}
                S_1, S_2 \subseteq [\rho w]\\
                S_1 = S_2            \end{subarray}}
            \Pr\left[\left(\ham\left(\bigoplus_{i \in S_1}\bC_i, (+1)^w\right) \le \zetaN w\right) \wedge \left(\ham\left(\bigoplus_{i \in S_2}\bC_i, (+1)^w\right) \le \zetaN w\right)\right]\\
        ={}&(2^{\rho w} - 1)(2^{\rho w} - 2)\alpha^2 + 2(2^{\rho w} - 1)\alpha + (2^{\rho w} - 1)\alpha + 1.
    \end{split}\]
    Note that $\rho > 1 - H(\zetaN)$ and thus $2^{\rho w}\alpha = 2^{\Theta(w)}$, so by Paley-Zygmund inequality with $\theta = 1/2$,
    \[\Pr[\bZ > 2^{\rho w - 1}\alpha] \ge \frac14 \cdot \frac1{\Theta(1) + 3 \cdot 2^{-\rho w}\alpha^{-1}} \ge \Theta(1).\qedhere\]
\end{proof}

\begin{proof}[Proof of \Cref{lem:eeldbc-with-hamming}]
    Note that $\eta_M < 1 - 1/(1 + H(\eta_N))$ implies $\eta_M/(1 - \eta_M) < H(\eta_N)$. Let
    \begin{align*}
        \rho &= 1 - 0.999 \cdot H(\eta_N) - 0.001 \cdot \frac{\eta_M}{1 - \eta_M} > 1 - H(\eta_N)\\
        \xi &= 0.001 \cdot H(\eta_N) - 0.001 \cdot \frac{\eta_M}{1 - \eta_M}\\
        \tau &= 1 - \rho - \xi = 0.998 \cdot H(\eta_N) + 0.002 \cdot \frac{\eta_M}{1 - \eta_M} > \frac{\eta_M}{1 - \eta_M}  \\
        \lambda &= \frac12(\rho + H(\zetaN) - 1).
    \end{align*}
    We then apply \Cref{lem:ding-jin-xing}, \Cref{lem:random-code-hamming}, and union bound.
\end{proof}

\subsection{The concept class $\calC$}
Let $\calU$ be the uniform distribution over $X$. Recall that we choose $\eta_N = 2^{-8r}$ and $\eta_M = H(\eta_N)/(H(\eta_N) + 2)$. Let $\rho, \tau, \lambda \in (0, 1)$ be the constants from \Cref{lem:eeldbc-with-hamming}. Let $\kappa = 0.998 \cdot \frac{\eta_M}{(1 - \eta_M)\tau} + 0.002$, then $\frac{\eta_M}{(1 - \eta_M)\tau} < \kappa < 0.004 + 0.996 \cdot \frac{\eta_M}{(1 - \eta_M)H(\eta_N)} = 0.502 < 1$.

For convenience we work over a domain $X$ of size $2^d/(1 - \kappa) = O(2^d)$. 
We view the domain as being partitioned into two subdomains, which we denote $X_\key$ and $X_\val$; we refer to $X_\key$ as the \emph{key side} and to $X_\val$ as the \emph{value side}. The key side contains $2^d\kappa/(1 - \kappa)$ elements of $X$ (say, the lexicographically first $2^d\kappa/(1 - \kappa)$ elements) and the value side contains the next $2^d$ elements of $X$. Note that a random draw from $\calU$ lands on the key side with probability $\kappa$.

Let $w$ be a sufficiently large 
$\Theta(d)$ and let $\nu$ be a sufficiently small $1/\poly(d)$. 
Let $\Ext: \{0, 1\}^w \times \{0, 1\}^u \to \{0, 1\}^\ell$ be the efficiently-computable $(\lambda w, \nu)$-extractor from \Cref{thm:seeded-extractor} where $u = O(\log(w/\nu)) = O(\log(w))$ and $\ell = \Theta(w)$.

By \Cref{lem:eeldbc-with-hamming}, let $C$ be a $(\tau, \Theta(1))$-efficiently erasure-list-decodable binary code mapping $(\rho w)$-bit messages to $w$-bit codewords, with associated encoding and decoding functions $\Enc, \Dec$, 
such that at least $2^{\lambda w}$ many codewords have Hamming weight at most $\zetaN w$.
For $p \in [2^{\lambda w}]$, let $W_p$ denote the lexicographically $p$-th codeword with Hamming weight at most $\zetaN w$.

We view the key side as being partitioned into $w$ blocks $X_\key = X_1 \sqcup \cdots \sqcup X_w$, where each $|X_i| = \frac1w|X_\key|$ (say, block $X_i$ contains the $i$-th block of $\frac1w|X_\key|$ many lexicographically consecutive strings from $X_\key)$.

Let $\calF = \{f_k: \{-1, +1\}^d \to \{-1, +1\}\}_{k \in \{-1, +1\}^\ell}$ be a pseudorandom function family.

Now we can describe the concept class $\calC$ over $X$. There are $2^{\lambda w + u}$ concepts $c_{p, q}$ parameterized by $p \in [2^{\lambda w}]$ and $q \in \{-1, +1\}^u$. 
The concept class $c_{p, q}$ is defined as follows:
\begin{itemize}
    \item \textbf{Key side:} Given $x \in X_\key$, let $j = j(x) \in [w]$ be the block such that $x \in X_j$. The value of $c_{p, q}(x)$ is $(W_p)_j$, the $j$-th bit of $W_p$. (So for each $j \in [w]$, every input example in block $X_j$ on the key side evaluates to the same output bit under $c_{p, q}$. Intuitively, the key side can be thought of as representing the $w$-bit string $W_p$ in a very redundant way.)
    \item \textbf{Value side:} Recall that there are $2^d$ strings in $X_\val$. We associate the lexicographically $i$-th string in $X_\val$ with the lexicographically $i$-th string in $\{-1, +1\}^d$. For $x \in X_\val$, the value of $c_{p, q}(x)$ is $f_{\Ext(W_p, q)}(x)$.
\end{itemize}

The structure of the domain and the way that concepts label domain elements is illustrated in \Cref{fig:domain-sep,fig:key-sep,fig:value-sep}.

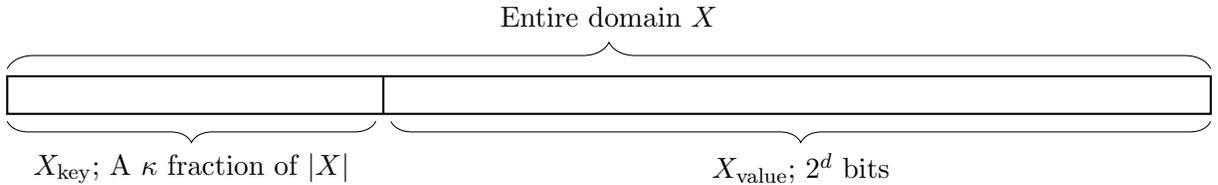
\begin{figure}[!h]

\begin{center}
    \begin{tikzpicture}
    \draw[thick] (0, 0) rectangle (16, 0.5);

    \draw[thick] (5, 0) -- (5, 0.5);

    \draw[decorate, decoration={brace, amplitude=10pt}] (0, 0.6) -- (16, 0.6) 
        node[midway, above, yshift=12pt] { Entire domain $X$};

    \draw[decorate, decoration={brace, amplitude=8pt, mirror}] (0, -0.1) -- (4.9, -0.1) 
        node[midway, below, yshift=-8pt] {$X_{\mathrm{key}}$; A $\kappa$ fraction of $|X|$};

    \draw[decorate, decoration={brace, amplitude=8pt, mirror}] (5.1, -0.1) -- (16, -0.1) 
        node[midway, below, yshift=-8pt] {$X_{\mathrm{value}}$; $2^d$ bits};

\end{tikzpicture}
\end{center}
\caption{Division of the domain into a key side and value side}
\label{fig:domain-sep}
\end{figure}
\medskip
\begin{figure}[!h]

    \begin{center}
        \begin{tikzpicture}
 
    \draw[very thick] (0, 0) rectangle (16, 0.5);

    \draw[very thick] (3, 0) -- (3, 0.5);

    \draw (0.5, 0) -- (0.5, 0.5);
    \node at (0.25, 0.25) {$1$};
    \draw (1, 0) -- (1, 0.5);
    \node at (0.75, 0.25) {$1$};
    \node at (1.75, 0.25) {$\cdots$};
    \draw (2.5, 0) -- (2.5, 0.5);
    \node at (2.75, 0.25) {$1$};

    \draw[very thick] (6, 0) -- (6, 0.5);

    \draw (3.5, 0) -- (3.5, 0.5);
    \node at (3.25, 0.25) {$0$};
    \draw (4, 0) -- (4, 0.5);
    \node at (3.75, 0.25) {$0$};
    \node at (4.75, 0.25) {$\cdots$};
    \draw (5.5, 0) -- (5.5, 0.5);
    \node at (5.75, 0.25) {$0$};

    \draw[very thick] (13, 0) -- (13, 0.5);
     \draw (13.5, 0) -- (13.5, 0.5);
    \node at (13.25, 0.25) {$1$};
    \draw (14, 0) -- (14, 0.5);
    \node at (13.75, 0.25) {$1$};
    \node at (14.75, 0.25) {$\cdots$};
    \draw (15.5, 0) -- (15.5, 0.5);
    \node at (15.75, 0.25) {$1$};

    \node at (8.3, 1.55) {$W_p = 1\,0\,\ldots\,1$};
    \draw[->] (1.5, 0.5) -- (8.2, 1.37);
    \draw[->] (4.5, 0.5) -- (8.5, 1.37);
    \draw[->] (14.5, 0.5) -- (9.45, 1.37);
    
    \draw[decorate, decoration={brace, amplitude=8pt, mirror}] (0.1, -0.1) -- (2.9, -0.1) 
        node[midway, below, yshift=-8pt] {$B_1$};

    \draw[decorate, decoration={brace, amplitude=8pt, mirror}] (3.1, -0.1) -- (5.9, -0.1) 
        node[midway, below, yshift=-8pt] {$B_2$};

    \draw[decorate, decoration={brace, amplitude=8pt, mirror}] (13.1, -0.1) -- (15.9, -0.1) 
        node[midway, below, yshift=-8pt] {$B_w$};

    \node at (9.5, -0.5) {\huge$\cdots$};

    \draw[decorate, decoration={brace, amplitude=5pt}] (8.2, 1.8) -- (9.45, 1.8) 
        node[midway, above, yshift=3pt] {$w=\mathrm{poly}(d)$ bits long};
\end{tikzpicture}
    \end{center}
    \caption{How the key side $X_{\mathrm{key}}$ is labeled by a concept $c_{p, q}$.}
    \label{fig:key-sep}
\end{figure}
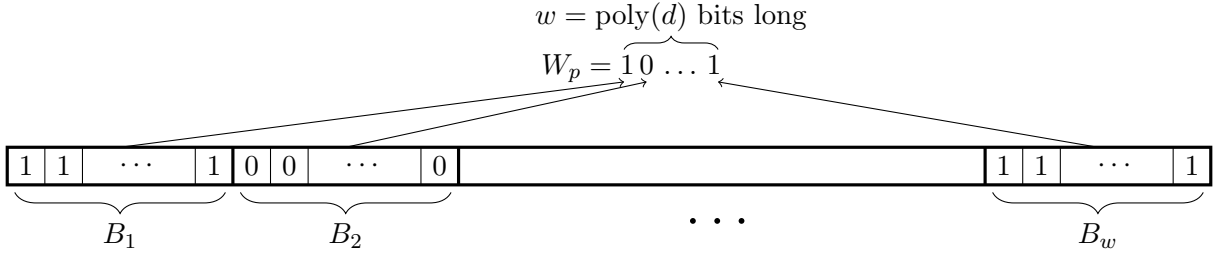
\begin{figure}[!h]

\begin{center}
    \begin{tikzpicture}

    \draw[very thick] (0, 0) rectangle (16, 0.55);

     \draw[thick] (7.3, 0) -- (7.3, 0.55);
    \draw[thick] (9.7, 0) -- (9.7, 0.55);
    \node at (8.5, 0.275) {$f_{\Ext(W_p, q)}(x')$};

    \node at (8.5, 0.85) {input $x' \in \{0,1\}^d$};

    \draw[decorate, decoration={brace, amplitude=8pt, mirror}] (0.1, -0.1) -- (15.9, -0.1) 
        node[midway, below, yshift=-8pt] {$2^d$ bits; these give the truth table of $f_{\Ext(W_p, q)}$};
\end{tikzpicture}
\end{center}
\caption{How the value side $X_{\mathrm{value}}$ is labeled by a concept $c_{p, q}$.}
\label{fig:value-sep}
\end{figure}
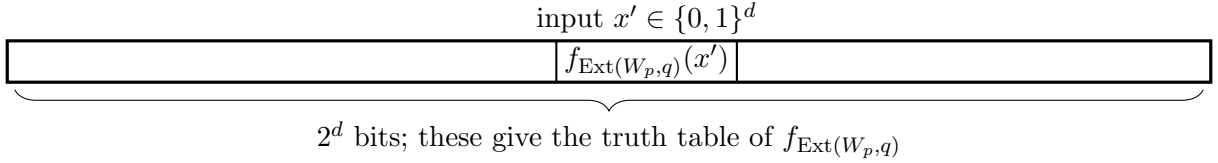

\subsection{Proof of part (1) of \texorpdfstring{\Cref{thm:fixed-separation}}{Theorem~\ref{thm:fixed-separation}}: no algorithm can efficiently learn $\calC$ with $\eta_N$-rate nasty noise}

Let $A$ be any $\poly(d)$-time algorithm which uses $n \le \poly(d)$ many examples. 

An outline of our argument is as follows. We first describe an adversary strategy. We then argue that the adversary strategy is feasible within the ``noise budget'' $\eta_N$. Next, we argue that if the nasty-noise adversary carries out the strategy, then  
(since $\calF$ is a PRF) 
the algorithm $A$ has probability less than any $1/\poly(d)$ 
of constructing a hypothesis $h$ that has accuracy as high as $\frac12 + \frac \kappa2 + 1/\poly(d)$. This gives part (1) of \Cref{thm:fixed-separation}.

\noindent\textbf{The adversary strategy.} Intuitively, the adversary strategy works as follows: for each block on the key side in which every example is labeled $-1$, the adversary flips every example in the block to $+1$. Note that the string $W_p$ has Hamming weight at most $\zetaN w$, so in expectation the adversary has to flip a $\zetaN$ fraction of examples on the key side, which is a $\zetaN \kappa$ fraction of all examples and is within the ``noise budget'' $\eta_N$. (The rationale for doing this is that every example on the key side in the corrupted data set will have label $+1$ while the distribution of the key side examples stays the same, i.e.~uniform random, so the learning algorithm will gain no information from the examples on the key side. Thus, the learner has to learn the concept $c_{p, q}$ using only examples from the value side, but this is cryptographically hard since on the value side $c_{p, q}$ is the pseudorandom function $f_{\Ext(W_p, q)}$, and the key $\Ext(W_p, q)$ comes from a distribution with total variation distance $\nu$ from the uniform distribution on $\{-1, +1\}^\ell$.)

Formally, the adversary strategy is as follows.
\begin{enumerate}
    \item Let $\bS = (\bx_1, \dots, \bx_n)$ be the initial multiset of $n$ unlabeled ``clean'' examples drawn i.i.d. from $\calU$, and let $\bS = \bS_1 \sqcup \cdots \sqcup \bS_w \sqcup \bS_\val$ be the partition of $\bS$ where $\bS_j = \bS \cap X_j$ for each $j \in [w]$ and $\bS_\val = \bS \cap X_\val$.
    
    \item For each $j \in [w]$, let $\bx_{j_1}, \dots, \bx_{j_{|\bS_j|}}$ be the elements of $\bS_j$. For each $j \in [w]$, if $(W_p)_j = -1$ (which means $\bx_{j_t}$ has label $-1$ for all $t \in [|\bS_j|]$), the adversary puts $j_1, \dots, j_{|\bS_j|}$ into the ``"corruption set" $\bZ' \subseteq [n]$, and replaces the sample $(\bx_{j_t}, -1)$ with $(\bx_{j_t}, +1)$ for all $t \in [|\bS_j|]$.
\end{enumerate}

Let $\bS_\nasty$ be the $n$-element data set which is the result of applying the nasty adversary strategy to $\bS$.

\noindent\textbf{An $\eta_N$-rate nasty noise adversary can always carry out the adversary strategy.} Recall that an $\eta_N$-rate nasty noise adversary can corrupt a subset $\bZ$ of the $n$ indices in its sample as long as the marginal distribution of $|\bZ|$ is equal to $\Bin(n, \eta_N)$. Also recall that for any $j \in [w]$, each example drawn from $\calU$ lies in $X_j$ with probability $\frac \kappa w$, and there are at most $\zetaN w$ many $j$ such that $(W_p)_j = -1$. Thus by inspection of the adversary strategy, we see that the distribution of the size $|\bZ'|$ of the corruption set is stochastically dominated by $\Bin(n, \zetaN \kappa)$, which is in turn dominated by $\Bin(n, \eta_N)$ since $\kappa < 1$.

\noindent\textbf{The end of the argument.} By inspection of the idealized adversary strategy, it is clear that all of the examples in $\bS_\nasty$ on the key side are labeled $+1$, and moreover the key-side examples in $\bS_\nasty$ are uniformly distributed over $X_\key$. Intuitively, this means that there is ``no useful information'' for algorithm $A$ in the key-side examples, and consequently algorithm $A$ must learn using only the value-side examples, but this is a cryptographically-hard problem, which is essentially an immediate consequence of the definition of pseudorandom functions and the property of $\Ext(p, q)$ having a distribution close to uniform. To make this more precise, we argue as follows.

Let $\bT$ be the data set of $n$ i.i.d. noiseless labeled examples $(\tilde\bx^1, c_{p, q}(\tilde\bx^1)), \dots, (\tilde\bx^n, c_{p, q}(\tilde\bx^n))$ where each $\tilde \bx^i$ is drawn uniformly from $X_\val$. Note that $c_{p, q}(\tilde\bx^i) = f_{\Ext(p, q)}(\tilde \bx^i)$. Let $\bT'$ be the data set of $(\tilde \bx^1, f_k(\tilde \bx^1)), \dots, (\tilde \bx^n, f_k(\tilde \bx^n))$ where $k$ is drawn uniformly from $\{-1, +1\}^\ell$. Since the distribution of $\Ext(p, q)$ has total variation distance at most $\nu$ from the distribution of $k$, the total variation distance between $\bT$ and $\bT'$ is at most $\nu$.

Via a standard argument, it follows directly from the definition of a PRF  that no polynomial-time algorithm $A'$, given the data set $\bT'$, can with probability even $1/\poly(d)$ construct a hypothesis which has accuracy $1/2 + 1/\poly(d)$ on value-side examples. Thus, no polynomial-time algorithm $A'$, given the data set $\bT$, can with probability $1/\poly(d) + \nu$ construct a hypothesis which has accuracy $1/2 + 1/\poly(d)$ on value-side examples. On the other hand, we have the following claim:

\begin{claim}\label{claim:simulate-S-nasty-sep}
    A polynomial-time algorithm $\tilde A$ can, given the $n$-element data set $\bT$, construct a data set $\bT_\nasty$ which is distributed precisely like $\bS_\nasty$.
\end{claim}

If $A$ can with probability $1/\poly(d) + \nu$ construct a hypothesis which has accuracy $(1 + \kappa)/2 + 1/\poly(d) < 0.51$ over $\calU$ given $\bS_\nasty$, then an algorithm $A'$ that given $\bT$ simulates $\tilde A$ on $\bT$ to construct the data set $\bT_\nasty$ as in \Cref{claim:simulate-S-nasty-sep} and then simulates $A$ on $\bT_\nasty$, can construct a hypothesis which has accuracy at least
\[\frac{(1 + \kappa)/2 + 1/\poly(d) - \kappa}{1 - \kappa} \ge \frac12 + \frac1{\poly(d)},\]
contradicting with the argument above.

\begin{proof}[Proof of \Cref{claim:simulate-S-nasty-sep}]
    $\tilde A$ carries out the following process: For $i = 1, \dots, n$, let $\bb_i$ be an independent random variable which takes each value in $[w]$ with probability $\kappa/w$ (recall that this is precisely the probability that a uniform random example lands in any particular block on the key side), and takes a special value $\bot$ with the remaining $1 - a$ probability. Let $\bT_j \subseteq [n]$ be the elements $i$ such that $\bb_i = j$. The elements of $\bT_j$ will correspond precisely to the elements $j_1, \dots, j_{|\bS_j|}$ in step 2 of the adversary strategy for which $\bx_{j_1}, \dots, \bx_{j_{|\bS_j|}}$ belong to $\bS_j$.
    
    Given the outcomes of $\bb_1, \dots, \bb_n$, algorithm $\tilde A$ uses those outcomes and the elements of $\bT$ to create the data set $\bT_\nasty$. In more detail, for each $i \in [n]$ for which $\bb_i = \bot$, $\tilde A$ uses a ``fresh'' labeled example from $\bT$ to play the role of $(\bx_i, c_{p, q}(\bx_i))$ (these correspond to the elements of $\bS_\val$). For each $j \in [w]$, for the $|\bT_j|$ elements of $\bT_j$, $\tilde A$ creates $\bT_j$ examples in the form of $(\bx, +1)$, where each example $\bx$ is an independent uniform random element of $X_j$.
    
    It is clear from inspection that the distribution of $\bT_\nasty$ is identical to the distribution of $\bS_\nasty$.
\end{proof}

This concludes the proof of part (1) of \Cref{thm:fixed-separation}.

\subsection{Proof of part (2) of \texorpdfstring{\Cref{thm:fixed-separation}}{Theorem~\ref{thm:fixed-separation}}: an algorithm can efficiently learn with $\eta_M$-rate strong malicious noise}

We first describe a learning algorithm which is called $A_\mal$: it will be clear that $A_\mal$ uses $\poly(d)$ many examples and runs in $\poly(d)$ time. We then argue that $A_\mal$ learns $\calC$ under the uniform distribution $\calU$ to accuracy $4\eta_M(1 + o(1)) + \eps$ and confidence $1 - \exp(-\Omega(d))$ in the presence of strong malicious noise at rate $\eta_M$.

Throughout the subsequent discussion we let $c = c_{m^*, q^*}$ be the target concept from $\calC$ (which is of course unknown to the learning algorithm).

\subsubsection{The algorithm $A_\mal$, and some intuition for it}

Let $n$ be a suitably large polynomial in $d/\eps$. Let $S = ((x_1, y_1), \dots, (x_n, y_n))$ be the data set of labeled examples from $X \times \{-1, +1\}$ that algorithm $A_\mal$ is given as the output of an $\eta$-noise-rate strong malicious noise process. Let $S_1 \sqcup \cdots \sqcup S_w \sqcup S_\val$ be the partition of $S$ given by $S_i = S \cap X_i$, $S_\val = S \cap X_\val$. We define the quantity
\[D := \frac{(1 - \eta_M)\kappa n}{w},\]
which is the expected number of clean examples from block $i$. Let $\Delta = n^{0.51}$, which should be thought of as an ``allowable'' amount of error between the actual values of various quantities and their expected values.

Algorithm $A_\mal$ is given in \Cref{fig:Amal-sep}.

\begin{figure}[H] 
    \captionsetup{width=.9\linewidth}
    
    \begin{tcolorbox}[colback = white,arc=1mm, boxrule=0.25mm]
        \vspace{2pt} 
        $A_\mal(S):$\vspace{6pt}
        
        \textbf{Input:} A data set $S=((x_1,y_1),\dots,(x_n,y_n))$ that is the output of an $\eta_M$-noise-rate strong malicious noise process on a ``clean'' $n$-element data set of examples drawn from ${\cal U}$ and labeled by $c$.\vspace{6pt}
        
        \textbf{Output:} A hypothesis $h$ for $c$ that with probability $1 - \exp(-\Omega(d))$ has error rate at most $4\eta_M(1 + o(1)) + \eps$.\vspace{6pt}
        
        \begin{enumerate}
            \item[]\textbf{1. Compute each key bit (if possible):} For every $i \in [w]$, let $s_i^{+1}$, $s_i^{-1}$ be the number of examples in $S_i$ with label $+1$, $-1$, respectively. Set
            \[z_i = \begin{cases}
                +1 & s_i^{-1} < D - \Delta \le s_i^{+1}\\
                -1 & s_i^{+1} < D - \Delta \le s_i^{-1}\\
                ? & \text{otherwise}
            \end{cases}.\]
            
            \vspace{6pt}
            \item[]\textbf{2. Decode a Boolean vector:} Run $\Enc(\Dec(z))$ to obtain $m^{(1)}, \dots, m^{(L')}$, where $L' \le 2^{1/\xi} = O(1)$ and each $m^{(i)}$ is a string in $\{-1, +1\}^w$ such that for all $t \in [w]$, $m^{(i)}_t \ne z_t$ if and only if $z_t ={} ?$.
            
          	\vspace{6pt}
            \item[]\textbf{3. Hypothesis testing to choose a final hypothesis:} For each $i \in [L']$, $q \in \{-1, +1\}^u$, perform hypothesis testing of the hypothesis $c_{m^{(i)}, q}$ using $S$ (more precisely, compute the empirical fraction $\hat{\eps}_{i, q}$ of examples $(x, y) \in S$ for which $c_{m^{(i)}}(x) \ne y$, and outputs as the final hypothesis $h$ the $c_{m^{(i)}, q}$ for which $\hat{\eps}_i$ is smallest, breaking ties arbitrarily).
        \end{enumerate}
    \end{tcolorbox}
    \caption{The algorithm $A_\mal$ for learning $\calC$ in the presence of $\eta_M$-rate strong malicious noise.}
    \label{fig:Amal-sep}
\end{figure} 

\subsubsection{Useful notation}
While the analysis is not particularly complicated, it requires keeping track of a number of different quantities; doing this is made easier with some systematic notation, which we provide in this subsection.

\begin{itemize}
    
    \item $\bS^\original = (\bx_1,c(\bx_i)),\dots,(\bx_n,c(\bx_n))$ is the original data set of $n$ i.i.d.~``clean'' labeled examples drawn from ${\cal D}$ in the first phase of the strong malicious noise process.
    
    \item $\bZ_{\mal} \subseteq [n]$ is the set of indices that are randomly selected for corruption in the second phase of the strong malicious noise process (recall \Cref{remark:strong-malicious}).
    
    \item $\bS^{\original,\survivenoise}$ is the elements of $\bS^\original$ that are not replaced by the $\eta_M$-noise-rate strong malicious adversary (i.e. ``survive'') into $S$.
    
    \item $S^{\new}=((x'_{j},y'_{j}))_{j \in \bZ_\mal}$ is the subsequence of noisy examples that the malicious adversary uses to replace the elements of $\bS^{\original,\replaced}$.  So the original $n$-element data set $S$ that is input to the algorithm is 
    \[
    S = \bS^{\original,\survivenoise} \sqcup S^{\new}.
    \]
    
    \item For each of the above sets, we append a subscript $_i$ to denote the subset consisting of those examples which lie in the $i$-th block of the key side, and we use a subscript $_\val$ to denote the subset of those examples which lie on the value side.
    So, for example, $S^\new_i$ is the subset of labeled examples that the malicious adversary introduces which belong to block $i$.
    
\end{itemize}
Let $\balpha_{i} \coloneq |\bS^{\original, \survivenoise}_i|$.

\subsubsection{Useful claims about typical events}\label{sec:useful-typical-sep}
We require two claims, both of which essentially say that with high probability over the draw of a ``clean'' random sample and the random choice of which examples can be corrupted by the strong malicious noise process,  
various sets have sizes that are close to their expected sizes.  These claims are easy consequences of standard tail bounds, namely \Cref{fact:hoeffding-inequality,} and \Cref{fact:multiplicative-Chernoff-bound}.
We record the statements that we will need below.

First some setup:  Observe that for each $i \in [w]$ we have 
\[
\E[\balpha_i]=D,
\]
and that moreover we have
\[
\E[|\bZ_\mal|] = \eta_M n, \quad
\]
Recall that we defined
\[
\Delta := n^{0.51}.
\]
We write
\[
x \approx_{\Delta} y
\]
as shorthand for $x \in [y - \Delta, y + \Delta]$.

\begin{lemma}
    [Sets typically have the ``right'' size]
    \label{lem:typical-sep}
    For a suitably large value of $n=\poly(d,1/\eta_M,1/\eps)=\poly(d/\eps)$, with probability $1-\exp(-\Omega(d))$ 
    over the draw of $(\bS^\original,\bZ_\mal)$ and the internal randomness of $A_\mal$, both of the following statements hold:
    
    \begin{itemize}
        
        \item [(A)] $\balpha_i \approx_{\Delta} D$ for each $i \in [w]$; 
        
        \item [(B)] $|\bZ_\mal| \approx_{\Delta} \eta_M n$ (and hence  $|S^\new| \approx_{\Delta}  \eta_M n$, since $|S^\new|=|\bZ_\mal|$).
        
    \end{itemize}
    
\end{lemma}

\subsubsection{Correctness of $A_\mal$}  In the rest of the argument, we assume that both of the ``typical'' events (A)-(B) described in \Cref{sec:useful-typical-sep} take place.  We proceed to argue that under these conditions, 

\begin{itemize}
    \item [(I)] With probability $1 - \exp(-\Omega(d))$, the string $z$ computed in step 1 satisfies that for all $t \in [w]$, $m^*_t \ne z_t$ if and only if $z_t = {}?$, and the number of $t \in [w]$ such that $z_t = {}?$ is upper-bounded by $\frac{\eta_M}{\kappa(1 - \eta_M)}(1 + o(1))w$.
    
    \item [(II)] With probability $1 - \exp(-\Omega(d))$, $m^\ast$ is one of the $w$-bit strings $m^{(1)},\dots,m^{(L')}$ that is obtained in step~2 of Algorithm $A_\mal$ (recall that the target concept is $c_{m^*, q^*}$); and
    
    \item [(III)] Given that (I) occurs, with probability at least $1-\exp(-\Omega(d))$ the hypothesis $c_{m^{(i)}}$ that Algorithm $A_\mal$ outputs has error at most $\eps + 4\eta_M(1+o(1))$.
\end{itemize}
The correctness of $A_\mal$ is a direct consequence of (I), (II) and (III). We prove these items in reverse order.

\textbf{Proof of (III).}
This is a fairly standard hypothesis testing argument. Fix any possible string $m \in \{0, 1\}^w$ and vector $q \in \{-1, 1\}^u$. Let the value $\hat{\eps}_{m,q}$ be the empirical fraction of examples $(x, y) \in S$ for which $c_{m, q}(x) \ne y$, let $\tilde \beps_{m,q}$ be the empirical fraction of examples $(x, y) \in \bS_\original$ for which $c_{m, q}(x) \ne y$, and let $\eps_{m,q}$ be the true error rate of hypothesis $c_{m, q}$ on noiseless uniform examples. A standard additive Hoeffding bound give us that with probability $1 - \exp(-\Omega(d))$, the value of $\tilde \beps_{m,q}$ is within an additive $\eps$ of $\eps_{m,q}$.
By a union bound over all pairs $(m, q) \in \{-1, +1\}^w \times \{-1, +1\}^u$, with probability $1 - \exp(-\Omega(w)) \cdot 2^w \cdot 2^u = 1 - \exp(-\Omega(w))$ we have this closeness for every $(m, q)$. Additionally, we have
\[|\hat{\eps}_{m,q} - \tilde \beps_{m,q} |n = \left|\hat{\eps}_{m,q}|S| - \tilde \beps_{m,q}|\bS^\original|\right| \le |S \mathop{\triangle} \bS^\original| \le 2|\bZ_\mal|.\]
By \Cref{lem:typical-sep} (B) we have that $|\bZ_\mal| \approx_{\Delta} \eta_M n$, so for every $(m, q)$ we have that $|\hat{\eps}_{m,q} - \tilde \beps_{m,q}| \le \frac1n(\eta_M n + \Delta) = 2\eta_M(1 + o(1))$.

Let $\eps_{i, q} \coloneq \eps_{m^{(i)}, q}$ be the true error rate of hypothesis $c_{m^{(i)}, q}$ on noiseless examples, and recall that $\hat{\eps}_{i, q}$ is the empirical error $c_{m^{(i)}, q}$ has on $S$. For the value $i^*$ such that $c_{m^{(i^*)},q^*}= c_{m^*,q^*}$ (there is such an $i^*$ because of (II)), we have $\eps_{i^*,q^*}=\tilde \beps_{i^*,q^*} = 0$, so $\hat{\eps}_{i^*,q^*} \le  2\eta_M(1 + o(1))$. Hence in step 3 the final hypothesis $c_{m^{(i)},q}$ that is output must have $\hat{\eps}_{i,q} \le 2\eta_M(1 + o(1))$, which means that $\tilde{\beps}_{i,q} \le 4\eta_M(1 + o(1))$ and therefore $\eps_{i,q} \le \eps + 4\eta_M(1 + o(1))$ as claimed.\qed

\textbf{Proof of (II).} By (I), with probability $1 - \exp(-\Omega(d))$, $z$ can be obtained from $m^*$ by changing at most $\frac{\eta_M}{\kappa(1 - \eta_M)}(1 + o(1))w < \tau w$ entries in $m^*$ to $?$. Therefore, by the erasure-list-decodability of the code $C$ as given by \Cref{lem:eeldbc-with-hamming}, $m^*$ is one of the $w$-bit strings $m^{(1)}, \dots, m^{(L')}$ obtained by the list-decoding algorithm.\qed

\textbf{Proof of (I).} By \Cref{lem:typical-sep} we have that with probability $1 - \exp(-\Omega(d))$, $\balpha_t \approx_\Delta D$ for each $t \in [w]$ and $|S_\new| \approx_\Delta \eta_M n$. Therefore, for each $t$, $s_t^{m^*_t} \ge D - \Delta$ (recall that $s_t^{+1}, s_t^{-1}$ are number of examples in $S_t$ with label $+1, -1$ respectively). It follows that $z_i$ is either set to $m^*_t$ or $?$ in step 1 (i.e. $m_t^* \ne z_t$ if and only if $z_t = ?$), and moreover, when $z_t$ is set to $?$, it holds that $|S_t^\new| \ge s_t^{-m^*_t} \ge D - \Delta$. Since $S^\new = S_1^\new \sqcup \cdots \sqcup \cdots S_w^\new$, the number of $?$ is upper-bounded by
\[\frac{|S^\new|}{D - \Delta} \ge \frac{\eta_M n - \Delta}{(1 - \eta_M)\kappa n/w - \Delta} \ge \frac{\eta_M}{\kappa(1 - \eta_M)}(1 + o(1))w.\qed\]

%% file: sections/fixed-distribution-ICE.tex
%!TEX root = ../malicious-nasty-writeup.tex

\section{Fixed-distribution malicious-noise-tolerant ICE algorithms and nasty noise: Proof of \texorpdfstring{\Cref{thm:ICE-malicious-nasty-informal}}{Theorem~\ref{thm:ICE-malicious-nasty-informal}}}
\label{sec:ICE}

\subsection{ICE algorithms} \label{sec:ICE-algorithms}
If a learning algorithm is given a data set which contains two copies $(x,-1)$ and $(x,1)$ of the same example point but with different labels, a natural way to proceed is to simply disregard the two conflicting examples. The following definition formalizes this notion:

\begin{definition} [ICE-algorithm] \label{def:ICE-algorithm}
A learning algorithm $A$ that uses $n$ labeled examples is said to \emph{Ignore Contradictory Examples} (equivalently, is said to be an \emph{ICE-algorithm}) if it works in the following two-stage way:
Let $S = (x_1,y_1),\dots,(x_n,y_n) \in (X \times \bits)^n$ be the sequence of examples that is given as input to $A$.

\begin{enumerate}

\item 
In the first stage, algorithm $A$  repeats the following step until no more repetitions are  possible:

\begin{itemize}
\item If there is a pair 
of examples $(x_i,-1)$, $(x_j,1)$ in $S$ with $x_i=x_j$ (i.e. a pair of \emph{contradictory examples}), then each of these examples is removed from $S$.
\end{itemize}

Let $S' \subseteq S$ be the result of removing all pairs of contradictory examples in this way.

\item In the second stage, $A$ calls some other learning algorithm $A'$ on $S'$ and returns the hypothesis $h: X \to \bits$ that $A'$ returns.

\end{enumerate}

\end{definition}

Note that if $A$ is an ICE-algorithm that is run in the standard noise-free PAC setting, then $A$ will simply run the algorithm $A'$ on the original data set $S$, since there will be no contradictory examples.
Note further that if $A$ is an ICE-algorithm that uses $n$ labeled examples and is run in the presence of malicious or nasty noise at rate $\eta$, where $n \gg 1/\eta$, then with very high probability in the second stage the algorithm $A'$ will be run on a data set $S'$ of at least (approximately) $(1-\eta)n$ labeled examples. $S'$ may still contain noisy examples, but it will not contain any pair of examples whose labels explicitly contradict each other.

As discussed in the introduction, ICE-algorithms capture a natural and intuitive way to deal with noisy data.
In the rest of this section we show that 
in the fixed distribution setting, if an efficient \emph{ICE algorithm} can tolerate malicious noise at rate $\eta$, then there is an efficient algorithm that can tolerate nasty noise at rate (almost) $\eta/2$ (\Cref{thm:ICE-malicious-nasty}). Then, in \Cref{sec:ICE-bad-news}, we show this factor of 1/2 in \Cref{thm:ICE-malicious-nasty} is the best possible (\Cref{thm:ICE-bad-news}).

\subsection{From malicious to nasty noise for ICE-algorithms}

The main result of this subsection, \Cref{thm:ICE-malicious-nasty}, shows that in the distribution-dependent setting, any ICE algorithm $A$ that succeeds with $\eta$-malicious noise can be upgraded to an algorithm $A'$ that succeeds with $\kappa\eta$-nasty noise, where $\kappa$ can be any constant less than $1/2.$ (This is the detailed version of the informal \Cref{thm:ICE-malicious-nasty-informal} from the Introduction.) The overhead in sample complexity and runtime of this conversion is only polynomial in the original sample complexity and runtime, polynomial in $1/\eps$, and logarithmic in $1/\delta$, and moreover the accuracy and confidence of the nasty noise learner are very nearly as good as the accuracy and confidence of the original malicious noise learner.

\begin{theorem} [Efficient malicious-noise-tolerant fixed-distribution learning by ICE algorithms implies efficient nasty-noise-tolerant fixed-distribution learning]
\label{thm:ICE-malicious-nasty}
Fix any distribution ${\cal D}$ over domain $X$ and any constant $\kappa<1/2$. Let ${\cal C}$ be a concept class over $X$.
For any $\eta < 1/2$, suppose that $A$ is an efficient ICE-algorithm which uses $n\geq C \log(1/\delta)/\eta$ examples (for a suitable constant $C$ depending on $\kappa$) and $(\eps,\delta)$-learns ${\cal C}$ in the presence of malicious noise at rate $\eta$. Then, for some
\begin{equation*}
    N \leq \poly(n, \log(2|X|+1), 1/\eps, \log(1/\delta)), 
\end{equation*}
there is an efficient algorithm $A'$ using $N$ examples that $(1.02\eps, 1.02\delta)$-learns ${\cal C}$ over distribution ${\cal D}$ in the presence of nasty noise at rate $\kappa\eta$.\footnote{Similar to \Cref{thm:dist-free-combined}, it will be clear from the proof that each occurrence of 1.02 could be replaced by $1+\tau$ for any fixed constant $\tau > 0$.}

\end{theorem}

Before proceeding with the detailed proof of \Cref{thm:ICE-malicious-nasty}, we first outline the main strand of the argument (the actual argument is a bit more involved for technical reasons that are discussed in the detailed proof).  The first step is to use \Cref{lem:compare-mal} to pass from the initial $\eta$-rate malicious noise tolerant learner $A$ to a learning algorithm $A'$ that can tolerate \emph{strong} malicious noise at rate $\eta$ (recall that strong malicious noise was defined in \Cref{remark:strong-malicious}).  The second step is a variant of \Cref{thm:ICE-malicious-nasty} which applies to \emph{strong} malicious noise, which we will prove later:

\begin{theorem} [Variant of \Cref{thm:ICE-malicious-nasty} for strong malicious noise]
\label{thm:ICE-strong-malicious-nasty}
Fix any distribution ${\cal D}$ over domain $X$ and any constant $\kappa<1/2$. Let ${\cal C}$ be any concept class over $X$. For any $\eta < 1/2$,
suppose that $A$ is an efficient ICE-algorithm which uses $n\geq C \log(1/\delta)/\eta$ examples (for a suitable constant $C$ depending on $\kappa$) and $(\eps,\delta)$-learns ${\cal C}$ in the presence of \emph{strong} malicious noise at rate $\eta$.
Then for some value $n' \leq n$, there is an efficient algorithm $A'$ which uses $n'$ examples and $(\eps, 1.005\delta)$-learns ${\cal C}$ over distribution ${\cal D}$ in the presence of nasty noise at rate $\kappa\eta$.
\end{theorem}

It may appear that we are done at this point, but some more ingredients are required in order to achieve the strong quantitative parameters of \Cref{thm:ICE-malicious-nasty} in all parameter regimes. The next step is the simple observation that any $(\eps',\delta')$-algorithm for learning in the presence of nasty noise is trivially also an $(\eps'+\delta')$-expected-error  algorithm for learning in the presence of nasty noise (cf.~the discussion at the end of \Cref{sec:expected-error}).  Then the final step is to apply \Cref{thm:boost-no-holdout}, which efficiently amplifies the success probability of expected-error learning algorithms in the nasty noise setting.

We first, in \Cref{sec:proofofthm:ICE-malicious-nasty}, give the proof of \Cref{thm:ICE-malicious-nasty} using the ingredients described above. Then we prove \Cref{thm:ICE-strong-malicious-nasty} in \Cref{sec:proofofthm:ICE-strong-malicious-nasty}. Finally, we present the equivalence between standard and strong malicious noise in \Cref{sec:standard-strong-equivalent}.

\subsection{Proof of \texorpdfstring{\Cref{thm:ICE-malicious-nasty}}{Theorem~\ref{thm:ICE-malicious-nasty}}}
\label{sec:proofofthm:ICE-malicious-nasty}

As discussed in the proof overview given in the previous subsection, the first step is to apply \Cref{lem:compare-mal}; we do this with the $\delta_{\additional}$ parameter set to $0.00005\eps$.  
We get from \Cref{lem:compare-mal} that there is an algorithm $A'$ which uses $m = \poly(n,\log(2|X|+1),1/\eps)$ examples and $(\eps,\delta + 0.00005\eps)$-learns ${\cal C}$ over distribution ${\cal D}$ in the presence of strong malicious noise at rate $\eta.$

Next, we apply \Cref{thm:ICE-strong-malicious-nasty} to algorithm $A'$.  This gives us that for some value $m' = \poly(n,\log(2|X|+1),1/\eps)$, there is an efficient algorithm $A''$ which uses $m'$ examples and $(\eps, 1.005(\delta + 0.00005\eps))$-learns ${\cal C}$ over distribution ${\cal D}$ in the presence of nasty noise at rate $\kappa \eta$. The argument now splits into two cases depending on the relative sizes of $\delta$ and $\eps$:

\medskip
\noindent {\bf Case 1:} $\delta > 0.004\eps.$ In this case we have that $1.005 \cdot 0.00005\eps < 0.00006\eps < 0.015\delta$, so $A''$ is already an $(\eps,1.02\delta)$-learner for ${\cal C}$ over distribution ${\cal D}$ in the presence of nasty noise at rate $\kappa \eta$, as desired, and by inspection the sample complexity and running time are as claimed in \Cref{thm:ICE-malicious-nasty}.

\medskip
\noindent {\bf Case 2 (the main strand of the argument):} $\delta \leq 0.004\eps.$
In this case, we next make the simple observation (cf.~the discussion at the end of \Cref{sec:expected-error}) that $A''$ is an $(\eps + 1.005(\delta + 0.00005\eps))$-expected-error  algorithm for learning in the presence of nasty noise, i.e.~an $(1.00006\eps + 1.005\delta)$-expected-error  algorithm for learning in the presence  of nasty noise. Finally, we apply \Cref{thm:boost-no-holdout} with its $\eps_{\additional}$ parameter set to $0.01\eps$ and its $\delta'$ parameter set to $\delta$.  We get that the resulting algorithm is a $(1.02\eps,\delta)$-learner for ${\cal C}$ over distribution ${\cal D}$ in the presence of $\kappa \eta$-rate nasty noise, and it is immediate from \Cref{thm:boost-no-holdout} that the running time and sample complexity are as claimed. This concludes the proof of \Cref{thm:ICE-malicious-nasty}, modulo the proof of \Cref{thm:ICE-strong-malicious-nasty}.
\qed

\subsection{Proof of \texorpdfstring{\Cref{thm:ICE-strong-malicious-nasty}}{Theorem~\ref{thm:ICE-strong-malicious-nasty}}}
\label{sec:proofofthm:ICE-strong-malicious-nasty}
{\bf Overview of the argument.} 
Before giving the detailed proof we give the high-level idea of the argument, which is quite simple.
Consider a corruption $S_{\nasty}$ of an initial $n'$-element ``clean'' data set $\bS$
that can be achieved by a $\kappa\eta$-rate nasty noise adversary. 
Let $S_{\new}$ be the $k$ ``noisy'' examples that the nasty adversary introduces, and let $S_{\replaced}$ be the $k$ ``clean'' examples from $S$ that the adversary chooses to replace with $S_{\new},$ so $S_{\nasty} = (\bS \setminus S_{\replaced}) \cup S_{\new}$.
We show that for some choice of the parameter $n' \leq n$ mentioned above, with high probability an $\eta$-noise-rate malicious adversary can corrupt a ``clean'' $n$-element sample $\bS^\ast$,
turning it into an $n$-element data set $S_{\mal}$, in such a way that after contradictory examples are removed from $S_{\mal}$, the result is precisely $S_{\nasty}$. 
(Intuitively, this is done by using half of the $\eta$ ``malicious noise budget'' to add the $k$ new points in $S_{\new}$ that the nasty adversary introduces, and using half of the $\eta$ malicious noise budget to generate contradictory versions of the $k$ points in $S_{\mathrm{replace}}$, thereby causing those points to be removed by the ICE-algorithm. This can be accomplished within budget, with high probability, because the nasty noise rate is only $\kappa \eta$, 
and $\kappa$ is a constant strictly less than $1/2$.)  
Since the ICE-algorithm $A$ successfully learns in the presence of malicious noise, it successfully learns given $S_{\mal}$; this means that with overall high probability, the algorithm $A'$ (recall \Cref{def:ICE-algorithm}) generates a high-accuracy hypothesis given $S_{\nasty}$ --- i.e.~it successfully learns in the presence of $\kappa\eta$-rate nasty noise.

\begin{proof}
Consider an execution of the $\eta$-noise-rate strong malicious adversary. Let $\bS^\ast$
be an initial sequence of $n$ noiseless labeled examples that is generated in the first phase of such an execution (recall \Cref{remark:strong-malicious}).
Let $\bZ_\mal
\subseteq [n]$ 
be the randomly chosen set of indices such that the ``strong malicious noise coin''  came up heads in the second phase, so each $j \in [n^\ast]$ independently belongs to $\bZ_\mal$ with probability $\eta$.  
Let $\bS$ be the sequence of $n-|\bZ_\mal|$ noiseless examples obtained by removing the elements of $\bZ_\mal$ from the $\bS^\ast$ sequence.

Since the size of $|\bZ_\mal|$ is distributed according to $\Bin(n,\eta)$, by Chebychev's inequality there is an absolute constant $c_1>0$ such that with probability at least $0.999$, the size $|\bZ_\mal|$ of $\bZ_\mal$ lies in the interval $[\eta n - c_1 \sqrt{\eta n},\eta n + c_1 \sqrt{\eta n}]$; we refer to such an outcome of $\bZ_\mal$ as a \emph{typical} outcome.
Since $A$ is an $(\eps,\delta)$-PAC learner in the presence of strong malicious noise at rate $\eta$, we have that with probability at most $\delta$ over the draw of $\bS^*$ and the outcome of $\bZ_\mal$ (and any internal randomness of $A$, which we denote by \$), the error of the hypothesis $h$ that $A$ generates is greater than $\eps$.
It follows that 
\[
\Prx_{\bS^*,\bZ_\mal,\$}[\text{error}(h) > \eps \ | \ 
\bZ_\mal \text{~is typical}] \leq {\frac 1 {0.999}}\delta < 1.002 \delta.
\]
Consequently there must be some value $m\in [\eta n - c_1 \sqrt{\eta n},\eta n + c_1 \sqrt{\eta n}]$ such that 
\begin{equation} \label{eq:2delta}
\Prx_{\bS^*,\bZ_\mal,\$}[\text{error}(h) > \eps \ | \ |\bZ_\mal|=m] \leq 1.002 \delta.
\end{equation}
We fix the value $n'$ to be 
\[
n' := n - 2 \lfloor m/2 \rfloor.
\]  
We observe that conditioned on $|\bZ_\mal|=m$, $\bS$ is distributed precisely as a sequence of $n' \geq (1 - {\frac 3 2} \eta)n$ i.i.d.~noiseless labeled examples drawn from ${\cal D}$, where the inequality holds by the upper bound on $m$ (recalling again that $n \geq C/\eta$).

We now change perspective, and view $\bS = ((\bx_i,c(\bx_i))_{i \in [n']}$ as the initial ``clean'' data set 
for a $\kappa\eta$-rate nasty adversary (recall Step~(1) in \Cref{def:variable-nasty-noise}).
Let $\bZ_\nasty$ be a random variable denoting the subset of $[n']$ corresponding to the indices that a $\kappa\eta$-rate nasty adversary can corrupt.  $\bZ_\nasty$ depends on $\bS$, but by the definition of nasty noise, the marginal distribution of $|\bZ_\nasty|$ is $\Bin(n',\kappa\eta)$. So using a multiplicative Chernoff bound with the lower bounds $n \geq C\log(1/\delta)/\eta$ and $n' \geq (1-{\frac 3 2}\eta)n$, with overall probability at least $1 - 0.001\delta$ over the draw of $(\bZ_\nasty,\bS)$, we have that
$|\bZ_\nasty|\leq {\frac 1 2} (\kappa \eta n' + {\frac 1 2} \eta n') < \lfloor m/2 \rfloor$.
We say that such an outcome of $(\bZ_\nasty,\bS)$ is a \emph{malleable} outcome of $(\bZ_\nasty,\bS)$.

Fix any malleable outcome of $(\bZ_\nasty,\bS)$, and consider an execution of the $\kappa\eta$-rate nasty adversary on data set $\bS$ when it can corrupt the set $\bZ_\nasty$. 
Let
$S_{\replaced}=((\bx_{i_1},c(\bx_{i_1})),\dots,$ $(\bx_{i_k},c(\bx_{i_k})))$ be the sequence of $k:=|\bZ_\nasty| < \lfloor m/2 \rfloor$ examples from $\bS$ that the nasty adversary chooses to replace,  let $S_{\new}=((x'_{i_1},y'_{i_1}),\dots,(x'_{i_k},y'_{i_k}))$ be the corresponding sequence of noisy examples that are used to replace them, and let $S'$ be the final data set of $n'$ examples resulting from all this.

We return to the strong malicious adversary setting, where the initial data set is $\bS^\ast$ (recall that $|\bS^\ast|=n$).
Recall that the strong malicious adversary is given ``full knowledge'' of the outcome of $\bZ_\mal$ and of $\bS^\ast$, and hence ``full knowledge'' of $\bS$, and that for each $j \in \bZ_\mal$ it can replace the $j$-th labeled example 
with an arbitrary labeled example from $X \times \bits.$  
Now for the crux of the argument:  when $|\bZ_\mal|=m$, 
the strong malicious adversary can operate as follows:

\begin{itemize}

\item [(1)] It uses $k < \lfloor m/2 \rfloor$ elements of $\bZ_\mal$ to construct contradictory examples for the $k$ elements of $\bS_{\replaced}$; i.e.~it constructs an example $(\bx_{i_j},\overline{c(\bx_{i_j})})$ for each $j \in [k]$.

\item [(2)] It uses $k<\lfloor m/2 \rfloor$ elements of $\bZ_\mal$ to introduce the $k$ elements $(x'_{i_1},y'_{i_1}),\dots,(x'_{i_k},y'_{i_k})$ of $S_{\new}$

\item [(3)] Finally, it uses $2(\lfloor m/2 \rfloor - k)$ elements of $\bZ_\mal$  to construct $\lfloor m/2 \rfloor - k$ pairs of contradictory examples.  (If $m$ is odd, it doesn't use the one remaining element of $\bZ_\mal$.)

\end{itemize}
Call the resulting $n$-element data set $S''$.

Now consider the execution of the ICE-algorithm $A$ on the $n$-element data set $S''$. Since $A$ is an ICE-algorithm, the $\lfloor m/2 \rfloor$ contradictory pairs ($k$ of which are from (1), and $\lfloor m/2 \rfloor -k$ of which are from (3)) will be removed, and the resulting final data set is precisely $S'$.
Combining \Cref{eq:2delta} with the $0.001\delta$ failure probability of not achieving a malleable outcome, it follows that $A$ $(\eps, 1.005\delta)$-learns in the presence of nasty noise at rate $\kappa\eta$.
\end{proof}

\input{sections/MalEquivalent}

\section{(Near)-Optimality of \texorpdfstring{\Cref{thm:ICE-malicious-nasty-informal}}{Theorem~\ref{thm:ICE-malicious-nasty-informal}}: Proof of \texorpdfstring{\Cref{thm:ICE-bad-news-informal}}{Theorem~\ref{thm:ICE-bad-news-informal}}}
\label{sec:ICE-bad-news}

In this section, we prove the formal version of \Cref{thm:ICE-bad-news-informal} stated below.

\begin{theorem} [The factor 1/2 between nasty and malicious noise rates in \Cref{thm:ICE-malicious-nasty} is best possible]
\label{thm:ICE-bad-news}
Let ${\cal U}$ denote the uniform distribution over $\bits^d$.  
Fix any $\eta \in [1/\poly(d),0.1]$
and any constant $1/2 < \kappa < 1$. Assume that one-way functions exist.  
Then there is a concept class ${\cal C} = {\cal C}_{\eta,\kappa}$ of concepts over $\bits^d$ with the following properties:

\begin{itemize}

\item [(1)] Any $\poly(d)$-time 
algorithm for learning ${\cal C}$ under ${\cal U}$ in the presence of nasty noise at rate $\kappa\eta$ cannot achieve accuracy better than ${\frac 1 2} + \kappa \eta$ with probability even as large as any $1/\poly(d)$; but

\item [(2)] There is an ICE-algorithm which runs in $\poly(d/\eps)$ time and $(\eps + 8\eta(1+o(1)),\exp(-\Omega(d)))$-learns ${\cal C}$ under ${\cal U}$ in the presence of strong malicious noise at rate $\eta$.
\end{itemize}

\end{theorem}

\subsection{Required technical ingredients}

We require the following result from the theory of error-correcting codes:

\begin{theorem} [see \cite{GuruswamiSudan00,GuruswamiRudra11} and Section~5 of \cite{Guruswami09IWCC}] \label{thm:list-decodable-binary-bit-flip}
Let $\tau>0$ be any (arbitrarily small) constant.  There exists a $(p={\frac 1 2} - \tau,L=\poly(d/\tau)=\poly(d))$-efficiently bit-flip-list-decodable binary code $C$ mapping $
d$-bit messages to $w=\poly(d/\tau)=\poly(d)$-bit codewords.
\end{theorem}
We write $\Enc$ for the encoding function $\Enc: \bits^d \to \bits^w$ and
$\Dec$ for the decoding function $\Dec: \bits^w \to (\bits^d)^{\leq L}$ of the code given by \Cref{thm:list-decodable-binary-bit-flip}.

\subsection{The concept class ${\cal C}_{\eta,\kappa}$} \label{sec:cc}

Let $\kappa' = {\frac {\kappa + 1/2} 2}$, and recall that since $\kappa>1/2$ we have $1/2 < \kappa' < \kappa$.
For convenience we work over a domain $X$ of size $2^d/(1-2\kappa'\eta)=O(2^d)$. 
We view the domain $X$ as being partitioned into two subdomains, which we denote $X_{\key}$ and $X_{\val}$; we refer to $X_{\key}$ as the \emph{key side} and to $X_{\val}$ as the \emph{value side}.  The key side contains $2\kappa'\eta 2^{d'}$ elements of $X$ (say, the lexicographically first $2\kappa'\eta 2^{d'}$ elements) and the value side contains the next 
Note that a random draw from ${\cal U}$ lands on the key side with probability $2\kappa'\eta$.

Fix $\tau = {\frac {\kappa - 1/2} {8}}$, and let $C$ be the $({\frac 1 2} - \tau,\poly(d))$-efficiently bit-flip-list-decodable binary code mapping $d$-bit messages to $w=\poly(d)$-bit codewords, with associated encoding and decoding functions $\Enc$, $\Dec$, whose existence is given by \Cref{thm:list-decodable-binary-bit-flip}.

We view the key side as being partitioned into $w$ \emph{blocks} $X_{\key} = X_1 \sqcup \cdots \sqcup X_w$, where each $|X_i| = {\frac 1 w} |X_{\key}|$ (say, block $X_i$ contains the $i$-th block of ${\frac 1 {w}} |X_{\key}|$  many lexicographically consecutive strings from $X_{\key}$). 

Let ${\cal F} = \{f_k: \bits^d \to \bits\}_{k \in \bits^d}$ be a pseudorandom function family.

Now we can describe the concept class ${\cal C}={\cal C}_{\eta,\kappa}$ over $X$.  There are $2^d$ functions, one for each possible key $k \in \bits^d$ for the PRF  ${\cal F}$. The $k$-th concept in ${\cal C}$, denoted $c_k$, is defined as follows:

\begin{itemize}

\item {\bf Key side:}  Given $x \in X_{\key}$, let $j=j(x) \in [w]$ be the block such that $x \in X_j$.  The value of $c_k(x)$ is $\Enc(k)_j$, the $j$-th bit in the encoding of the key $k$.  (So for each $j \in [w]$, every input example in block $X_j$ on the key side evaluates to the same output bit under $c_k$. Intuitively, the key side can be thought of as representing the $w$-bit string $\Enc(k)$ in a very redundant way.)

\item {\bf Value side:} Recall that there are $2^d$ strings in $X_{\val}$. Fix an efficiently computable and efficiently invertible bijection between $X_{\val}$ and $\bits^d$ (for concreteness, say the lexicographically $i$-th string in $X_{\val}$ is associated with the lexicographically $i$-th string in $\bits^d$); given $x \in X_{\val}$, write $x'$ for the corresponding element of $\bits^d$ under this bijection.  For $x \in X_{\val}$, the value of $c_k(x)$ is $f_{k}(x')$. (Intuitively, the $2^d$ outputs of $c_k$ across $X_{\val}$ correspond to the $2^d$-bit truth table of the pseudorandom function $f_k$.)

\end{itemize}

The structure of the domain and the way that concepts label domain elements is illustrated in \Cref{fig:domain,fig:key,fig:value}.

\begin{figure}[!h]

\begin{center}
    \begin{tikzpicture}
    \draw[thick] (0, 0) rectangle (16, 0.5);

    \draw[thick] (5, 0) -- (5, 0.5);

    \draw[decorate, decoration={brace, amplitude=10pt}] (0, 0.6) -- (16, 0.6) 
        node[midway, above, yshift=12pt] { Entire domain $X$};

    \draw[decorate, decoration={brace, amplitude=8pt, mirror}] (0, -0.1) -- (4.9, -0.1) 
        node[midway, below, yshift=-8pt] {$X_{\mathrm{key}}$; A $2\kappa'\eta$ fraction of $|X|$};

    \draw[decorate, decoration={brace, amplitude=8pt, mirror}] (5.1, -0.1) -- (16, -0.1) 
        node[midway, below, yshift=-8pt] {$X_{\mathrm{value}}$; $2^d$ bits};

\end{tikzpicture}
\end{center}
\caption{Division of the domain into a key side and value side}
\label{fig:domain}
\end{figure}
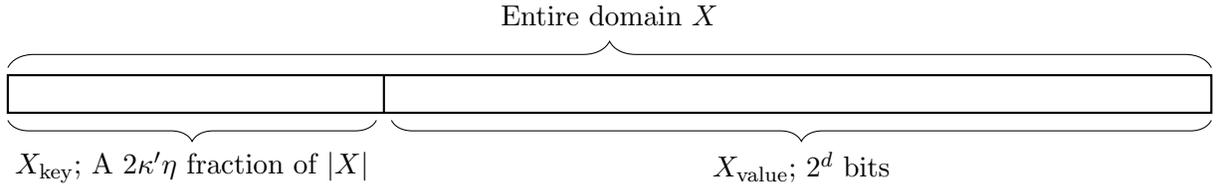
\medskip
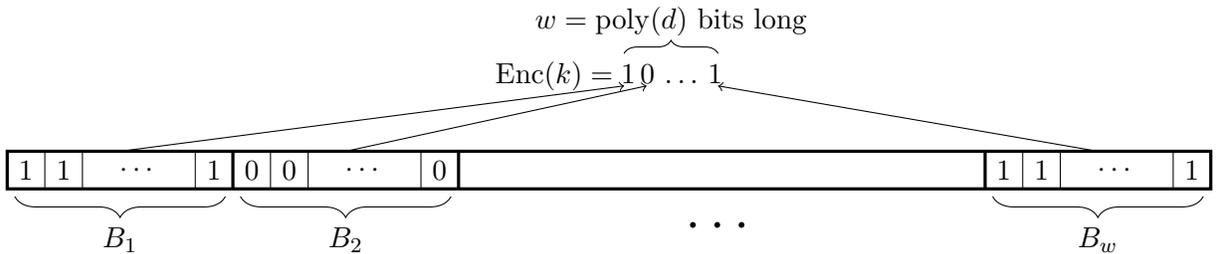
\begin{figure}[!h]

    \begin{center}
        \begin{tikzpicture}
 
    \draw[very thick] (0, 0) rectangle (16, 0.5);

    \draw[very thick] (3, 0) -- (3, 0.5);

    \draw (0.5, 0) -- (0.5, 0.5);
    \node at (0.25, 0.25) {$1$};
    \draw (1, 0) -- (1, 0.5);
    \node at (0.75, 0.25) {$1$};
    \node at (1.75, 0.25) {$\cdots$};
    \draw (2.5, 0) -- (2.5, 0.5);
    \node at (2.75, 0.25) {$1$};

    \draw[very thick] (6, 0) -- (6, 0.5);

    \draw (3.5, 0) -- (3.5, 0.5);
    \node at (3.25, 0.25) {$0$};
    \draw (4, 0) -- (4, 0.5);
    \node at (3.75, 0.25) {$0$};
    \node at (4.75, 0.25) {$\cdots$};
    \draw (5.5, 0) -- (5.5, 0.5);
    \node at (5.75, 0.25) {$0$};

    \draw[very thick] (13, 0) -- (13, 0.5);
     \draw (13.5, 0) -- (13.5, 0.5);
    \node at (13.25, 0.25) {$1$};
    \draw (14, 0) -- (14, 0.5);
    \node at (13.75, 0.25) {$1$};
    \node at (14.75, 0.25) {$\cdots$};
    \draw (15.5, 0) -- (15.5, 0.5);
    \node at (15.75, 0.25) {$1$};

    \node at (8, 1.5) {$\mathrm{Enc}(k) = 1\,0\,\ldots\,1$};
    \draw[->] (1.5, 0.5) -- (8.2, 1.37);
    \draw[->] (4.5, 0.5) -- (8.5, 1.37);
    \draw[->] (14.5, 0.5) -- (9.45, 1.37);
    
    \draw[decorate, decoration={brace, amplitude=8pt, mirror}] (0.1, -0.1) -- (2.9, -0.1) 
        node[midway, below, yshift=-8pt] {$B_1$};

    \draw[decorate, decoration={brace, amplitude=8pt, mirror}] (3.1, -0.1) -- (5.9, -0.1) 
        node[midway, below, yshift=-8pt] {$B_2$};

    \draw[decorate, decoration={brace, amplitude=8pt, mirror}] (13.1, -0.1) -- (15.9, -0.1) 
        node[midway, below, yshift=-8pt] {$B_w$};

    \node at (9.5, -0.5) {\huge$\cdots$};

    \draw[decorate, decoration={brace, amplitude=5pt}] (8.2, 1.8) -- (9.45, 1.8) 
        node[midway, above, yshift=3pt] {$w=\mathrm{poly}(d)$ bits long};
\end{tikzpicture}
    \end{center}
    \caption{How the key side $X_{\mathrm{key}}$ is labeled by a concept $c_k$ for a key $k$.}
    \label{fig:key}
\end{figure}
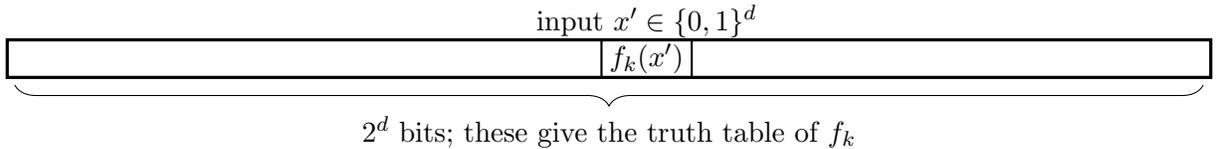
\begin{figure}[!h]

\begin{center}
    \begin{tikzpicture}

    \draw[very thick] (0, 0) rectangle (16, 0.5);

     \draw[thick] (7.9, 0) -- (7.9, 0.5);
    \draw[thick] (9.1, 0) -- (9.1, 0.5);
    \node at (8.5, 0.25) {$f_k(x')$};

    \node at (8.5, 0.75) {input $x' \in \{0,1\}^d$};

    \draw[decorate, decoration={brace, amplitude=8pt, mirror}] (0.1, -0.1) -- (15.9, -0.1) 
        node[midway, below, yshift=-8pt] {$2^d$ bits; these give the truth table of $f_k$};
\end{tikzpicture}
\end{center}
\caption{How the value side $X_{\mathrm{value}}$ is labeled by a concept $c_k$ for a key $k$.}
\label{fig:value}
\end{figure}

\subsection{Proof of part (1) of \texorpdfstring{\Cref{thm:ICE-bad-news}}{Theorem~\ref{thm:ICE-bad-news}}:  no 
algorithm can efficiently learn ${\cal C}$ with $\kappa\eta$-rate nasty noise, for $\kappa>1/2$} \label{sec:proof-1}

Let $A$ be any $\poly(d)$-time 
algorithm which uses $n \leq \poly(d)$ many examples. 
We recall that by \Cref{rem:too-many-examples}, we may assume that $n \geq Cw/\eta$ for a suitably large absolute constant $C$ (we will use this later).

An outline of our argument is as follows:  We first describe an ``idealized'' adversary strategy.
We then argue that while an actual $\kappa\eta$-rate nasty-noise adversary cannot always carry out the idealized adversary strategy, it can carry out the idealized adversary strategy with very high probability, namely at least $1-\exp(-d)$ (this probability is over the random examples drawn by algorithm $A$).
Next, we argue that conditioned on the idealized adversary strategy being carried out, 
and since $\calF$ is a PRF,
algorithm $A$ has probability less than any $1/\poly(d)$ of 
constructing a hypothesis $h$ that has accuracy as high as $2\kappa'\eta + (1-2\kappa'\eta)({\frac 1 2} + 1/\poly(d)) = {\frac 1 2} + \kappa' \eta + 1/\poly(d) < {\frac 1 2 + \kappa \eta}.$
This gives part (1) of \Cref{thm:ICE-bad-news}.

\noindent{\bf The idealized adversary strategy.}
Intuitively, the idealized adversary strategy works as follows: for each block on the key side, the adversary selects half\footnote{Some care must be taken because of even/odd issues; see step~(2)(b) of the actual idealized adversary description below.} of the examples in that block and converts them into contradictory examples for the other half of the examples in that block. (The rationale for doing this is that given such a data set of corrupted examples, 
the learning algorithm will gain no information at all from the examples on the key side, so it will have to learn the concept $c_k$ using only examples from the value side; but this is cryptographically hard since on the value side $c_k$ is simply the pseudorandom function $f_k$.)

Formally, the idealized adversary strategy is as follows.

\begin{enumerate} 

\item 
Let $\bS=(\bx_1,\dots,\bx_n)$ be the initial multiset of $n$ unlabeled ``clean'' examples drawn i.i.d.~from ${\cal U}$, and let $\bS = \bS_1 \sqcup \cdots \sqcup \bS_w \sqcup \bS_{\val}$ be the partition of $\bS$ where $\bS_j = \bS \cap X_j$ for each $j \in [s]$ and $\bS_{\val}=\bS \cap X_{\val}$.

\item 
 For each $j \in [w]$, let $\bx_{j_1},\dots,\bx_{j_{|\bS_j|}}$ be the elements of $\bS_j.$
 For each $j \in [w]$: 

\begin{enumerate}

\item If $|\bS_j|$ is even, then the idealized adversary puts $j_1,\dots,j_{|\bS_j|/2}$ into the ``corruption set'' $\bZ' \subseteq [n].$ (Note that all of these examples have the same true label, which is $\Enc(k)_j$.)
For $1 \leq t \leq |\bS_j|/2$, the sample $(\bx_{j_t},c(\bx_{j_t})) = (\bx_{j_t},\Enc(k)_j)$ is replaced with $(\bx_{j_{t+|\bS_j|/2}},\overline{\Enc(k)_j})$.  Observe that with this choice of replacement examples, each replacement example $(\bx_{j_{t+|\bS_j|/2}},\overline{\Enc(k)_j})$ is a contradictory example with the example $(\bx_{j_{t+|\bS_j|/2}},c(\bx_{j_{t+|\bS_j|/2}}))=
(\bx_{j_{t+|\bS_j|/2}},\Enc(k)_j)$. 
\item If $|\bS_j|$ is odd,
then the idealized adversary puts $j_1,\dots,j_{\lceil |\bS_j|/2 \rceil}$ into $\bZ'$ (as in the previous bullet, all of these examples have the same true label $\Enc(k)_j$).  For $1 \leq t \leq \lfloor |\bS_j|/2 \rfloor$, the sample $(\bx_{j_t},c(\bx_{j_t})) = (\bx_{j_t},\Enc(k)_j)$ is replaced with $(\bx_{j_{t+\lceil |\bS_j|/2 \rceil}},\overline{\Enc(k)_j})$.
Similar to the even case, with this choice of replacement examples, for $1 \leq t \leq \lfloor |\bS_j|/2 \rfloor$ each replacement example $(\bx_{j_{t+\lceil |\bS_j|/2 \rceil}},\overline{\Enc(k)_j})$ is a contradictory example with the example $(\bx_{j_{t+\lceil |\bS_j|/2 \rceil}},c(\bx_{j_{t+\lceil |\bS_j|/2 \rceil}})) = (\bx_{j_{t+\lceil |\bS_j|/2 \rceil}},\overline{\Enc(k)_j})$.
Finally, the labeled example $(\bx_{j_{\lceil |\bS_j|/2 \rceil}},c(\bx_{j_{\lceil |\bS_j|/2 \rceil}}))=(\bx_{j_{\lceil |\bS_j|/2 \rceil}},\Enc(k)_j)$ is replaced with a correctly labeled uniform random example from the value side, i.e.~with $(\bx,c_k(\bx))$ where $\bx$ is uniform random over the value side $X_\val$.

\end{enumerate}

\end{enumerate}

Let $\bS_\nasty$ be the $n$-element data set which is the result of applying the idealized nasty adversary strategy to $\bS$.

\noindent {\bf A $\kappa\eta$-rate nasty noise adversary can carry out the idealized adversary strategy with extremely high probability.}
Recall that in actuality, a $\kappa \eta$-rate nasty noise adversary can corrupt a subset $\bZ$ of the $n$ indices in its sample as long as the marginal distribution of $|\bZ|$ is equal to $\Bin(n,\kappa\eta)$.
By inspection of the idealized adversary strategy, recalling that each example drawn from ${\cal U}$ lies on the key side with probability $2\kappa'\eta$, we see that the size $|\bZ'|$ of the corruption set for the idealized adversary is stochastically dominated by the random variable $\by$ whose distribution is given by ${\frac 1 2}\Bin(n,2\kappa' \eta) + w.$  (The $\Bin(n,2\kappa'\eta)$ comes from the fact that each of the $n$ examples drawn from ${\cal U}$ independently lands on the key side with probability $2\kappa' \eta$, and the extra ``$+w$'' comes from the fact that each of the $w$ blocks could require (at most) one extra corruption because of $|\bS_j|$ being odd.)
Next we will use the following claim:

\begin{claim} \label{claim:sd}
The distribution ${\frac 1 2}\Bin(n,2\kappa' \eta) + w$  is $(1-\zeta)$-stochastically dominated by $\Bin(n,\kappa \eta)$, where $\zeta = \exp(-d)$.
\end{claim}
\begin{proof}
This follows easily from multiplicative Chernoff bounds and the bound
$n \geq Cw/\eta$ (recall the beginning of \Cref{sec:proof-2}). In more detail, let $\kappa'' = \kappa' + {\frac {\kappa - \kappa'} 3}$ and let $\kappa''' = \kappa' + {\frac {2(\kappa-\kappa')} 3}$.  The expected value of a draw from ${\frac 1 2} \Bin(n,2\kappa'\eta)$ is $\kappa'\eta n$, so a multiplicative Chernoff bound using the lower bound on $n$ (and recalling that $w \geq d$) gives us that with probability at least $1-\zeta/2$, a draw from ${\frac 1 2} \Bin(n,2\kappa'\eta)$ takes value at most $\kappa'' \eta n$.  Similarly, a multiplicative Chernoff bound gives that with probability at least $1-\zeta/2$, a draw from $\Bin(n,\kappa \eta)$ takes value at least $\kappa''' \eta n$.  Since the lower bound on $n$ implies that $\kappa'' \eta n + w < \kappa''' \eta n$, the claim follows.
\end{proof}

Given \Cref{claim:sd}, it follows that $|\bZ'|$ is $(1-\exp(-d))$-stochastically dominated by $\Bin(n,\kappa \eta)$. This means that a $\kappa\eta$-rate nasty-noise adversary can carry out the idealized adversary strategy with probability $1-\exp(-d)$ over the choice of the $n$ random examples $\bS$ drawn by the ICE-algorithm.

\noindent {\bf The end of the argument.}  Let us say that an outcome of the clean sample $\bS$ is
\emph{vulnerable} if the $\kappa \eta$-rate nasty-noise adversary can carry out the idealized adversary strategy described above; so we have argued above that $\Pr[\bS$ is vulnerable$] \geq 1-\exp(-d)$.

By inspection of the idealized adversary strategy, it is clear that if $\bS$ is vulnerable, then all of the examples in $\bS_\nasty$ which are on the key side come in contradictory pairs, and that moreover the key-side examples in $\bS_\nasty$ are uniformly distributed over $X_\key$.
Intuitively, this means that there is ``no useful information'' for algorithm $A$ in the key-side examples, so consequently algorithm $A$ must learn using only the value-side examples; but this is a well-known cryptographically hard problem (essentially as an immediate consequence of the definition of pseudorandom functions).
To make this more precise, we argue as follows.

Let $\tilde{A}$ be a polynomial-time algorithm which is given as input a data set $\bT$ of $n$ i.i.d.~noiseless labeled examples $(\tilde{\bx}^1,c(\tilde{\bx}^1)),\dots,(\tilde{\bx}^{n},c(\tilde{\bx}^{n}))$ where each $\tilde{\bx}^i$ is drawn uniformly from $X_\val$.

Via a standard argument, it follows directly from the definition of a PRF that no polynomial-time algorithm $\tilde{A}$, given such a data set $\bT$, can with probability even $1/\poly(d)$ construct a hypothesis which has accuracy $1/2 + 1/\poly(d)$ on value-side examples.
On the other hand, we have the following claim:

\begin{claim} \label{claim:tilde-simulate-A}
A polynomial-time algorithm $\tilde{A}$ can, given the $n$-element data set $\bT$, construct a data set $\bT_\nasty$ which is distributed precisely like $\bS_\nasty$.
\end{claim}
\begin{proof}
$\tilde{A}$ carries out the following process:  For $i=1,\dots,n$, let $\bb_i$ be an independent random variable which takes each value in $[w]$ with probability  $2\kappa'\eta/w$ (recall that this is precisely the probability that a uniform random example lands in any particular block on the key side), and takes a special value $\bot$ with the remaining $1-2\kappa'\eta$ probability.  
Let $\bT_j \subseteq [n]$ be the elements $i$ such that $\bb_i=j$.
The elements of $\bT_j$ will correspond precisely to the elements $j_1,\dots,j_{|\bS_j|}$ in step~2 of the idealized adversary strategy for which $\bx_{j_1},\dots,\bx_{j_{|\bS_j|}}$ belong to $\bS_j$.

Given the outcomes of $\bb_1,\dots,\bb_n$, algorithm $\tilde{A}$ uses those outcomes and the elements of $\bT$ to carry out the idealized adversary strategy and create the data set $\bT_\nasty$. In more detail, for each $i \in [n]$ for which $\bb_i=\bot$, 
$\tilde{A}$ uses a ``fresh'' labeled example from $\bT$ to play the role of $(\bx_i,c(\bx_i))$ (these correspond to the elements of $\bS_\val$).  For each $j \in [w]$, for the $|\bT_j|$ elements of $\bT_j$,

\begin{itemize}

\item If $|\bT_j|$ is even, then $\tilde{A}$ creates $|\bT_j|/2$ contradictory pairs of labeled examples $(\bx,-1),(\bx,1)$, where each pair is obtained by independently choosing the example $\bx$ to be a uniform random element of $X_j$.

\item If $|\bT_j|$ is odd, then $\tilde{A}$ creates $\lfloor |\bT_j|/2 \rfloor$ contradictory pairs of labeled examples as described above, and $\tilde{A}$ uses one more ``fresh'' labeled example from $\bT$ to play the role of  the correctly labeled uniform random example from the value side that is used to replace $(\bx_{j_{\lceil |\bS_j|/2 \rceil}},c(\bx_{j_{\lceil |\bS_j|/2 \rceil}}))$ in $\bS_j$, as described in the last sentence of step~2(b) of the idealized adversary strategy

\end{itemize}

It is clear from inspection that the distribution of $\bT_\nasty$ is identical to the distribution of $\bS_\nasty$.
\end{proof}

Since $\Pr[\bS$ is vulnerable$] \geq 1-\exp(-d)$, it follows that with overall probability at least $1 - \exp(-d) - 1/\poly(d)$, the accuracy of $A$'s final hypothesis is at most 
\[
(1-2\kappa'\eta)\pbra{{\frac 1 2} + 1/\poly(d)} + 2 \kappa' \eta= {\frac 1 2} + \kappa' \eta + 1/\poly(d) < {\frac 1 2} + \kappa \eta,
\] where the $2 \kappa' \eta$ factors on the left-hand side are because a random example drawn from ${\cal U}$ lands on the key side with probability $2 \kappa' \eta$.  This concludes the proof of part (1) of \Cref{thm:ICE-bad-news}.

\subsection{Proof of part (2) of \texorpdfstring{\Cref{thm:ICE-bad-news}}{Theorem~\ref{thm:ICE-bad-news}}:  an ICE-algorithm can efficiently learn with $\eta$-rate strong malicious noise} \label{sec:proof-2}

We first describe a learning algorithm which is called $A_\mal$; it will be clear that $A_\mal$ is an ICE-algorithm and that it uses $\poly(d)$ many examples and runs in $\poly(d)$ time. We then argue that $A_\mal$ learns ${\cal C}$ under the uniform distribution $\calU$ to accuracy $1-\eps - 8\eta(1+o(1))$ and confidence $1-\exp(-\Omega(d))$ in the presence of strong malicious noise at rate $\eta.$

Throughout the subsequent discussion we let $c = c_{m^\star}$ 
be the target concept from ${\cal C}$ (which is of course unknown to the learning algorithm).

\subsubsection{A useful tool:  List-decoding ``continuous received words'' via randomized rounding}

We define the randomized function $\Round: \R \to \bits$ as follows:
\begin{align*}
& \bullet \ \ \ \text{for~}v \in [-1,1], \quad \Round(v) = 
\begin{cases}
1 & \text{with probability~}{\frac {1+v} 2}\\
-1 &\text{with probability~}{\frac {1-v} 2},
\end{cases}\\
& \bullet \ \ \ \text{for~}|v|>1, \quad  \ \ \ \ \ \Round(v) = \sign(v).
\end{align*}
Note that $\E[\Round(v)]=v$ if $|v| \leq 1$.

 We abuse notation and for $v = (v_1,\dots,v_w) \in \R^w$, we write $\Round(v)$ to denote the vector-valued random function
\[
\Round(v)=(\Round(v_1),\dots,\Round(v_w)) \in \bits^w.
\]

We will need the following basic technical lemma:

\begin{lemma} \label{lem:rounding}
Let $v \in \R^w$, $u \in \bits^w$ be such that $\|v-u\|_1 \leq (1-\kappa)w$ where $\kappa>0$ is a positive constant.  Then with probability $1 - \exp(-\Omega(w))$ we have 
\[
\ham(\Round(v),u) \leq {\frac 1 2} \pbra{1 - {\frac \kappa 2}}w.
\]
\end{lemma}
\begin{proof}
For each $i \in [w]$ the random variables $\ham(\Round(v_i),u_i)$ are independent Bernoulli random variables whose sum is $\ham(\Round(v),u)$. Moreover, for each $i \in [w]$ we have 
\[
\E[\ham(\Round(v_i),u_i)] \leq {\frac 1 2} |v_i - u_i| \quad \text{(with equality if $|v| \leq 1$)},
\] 
so $\E[\ham(\Round(v),u)]\leq{\frac 1 2} \|v-u\|_1 \leq {\frac 1 2} (1-\kappa) w.$
The lemma now follows from a standard multiplicative Chernoff bound.
\end{proof}

\Cref{lem:rounding} can be interpreted in the following way:  we can think of a vector $v \in \R^w$ as a ``continuous received word'' for a bit-flip error-correcting code, and of $u \in \bits^w$ as an (unknown) codeword such that the $\ell_1$ distance $\|v-u\|_1$ from $v$ to $u$ is bounded below $w$ by some $\Omega(w)$ amount.
\Cref{lem:rounding} gives a very simple randomized procedure, which with extremely high probability, converts $v$ into a Boolean vector $\Round(v)$ for which the Hamming distance $\ham(\Round(v),u)$ is bounded below ${\frac w 2}$ by some $\Omega(w)$ amount. This Hamming distance bound is useful because it will enable us to apply list-decoding to $\Round(v).$

\subsubsection{The algorithm $A_\mal$ and some intuition for it}

Let $n$ be a suitably large polynomial in $d/\eps$.
We define the quantity
\begin{equation}
\label{eq:R}
R := {\frac {2\kappa' \eta} w} n,
\end{equation}
which will play an important role in our algorithm and analysis. Note that for each $i \in [w]$, the value $R$ is the expected number of correctly-labeled examples from block $i$ which would appear in the 
set 
$S'$ (see Step~1 of Algorithm~$A_\mal$) if there were no noise.

Let $S=((x_1,y_1),\dots,(x_n,y_n))$ be the data set of labeled examples from $X \times \bits$ that algorithm $A_\mal$ is given as the output of an $\eta$-noise-rate strong malicious noise process.  Let  $S_1 \sqcup \cdots \sqcup S_w \sqcup S_{\val}$ be the partition of $S$ given by $S_i = S \cap X_i$, $S_\val = S \cap X_\val$.
Algorithm $A_\mal$ is given in \Cref{fig:Amal}.

\begin{figure}[H] 
  \captionsetup{width=.9\linewidth}
    
    \begin{tcolorbox}[colback = white,arc=1mm, boxrule=0.25mm]
    \vspace{2pt} 
$A_\mal(S):$\vspace{6pt}
    
    \textbf{Input:} A data set $S=((x_1,y_1),\dots,(x_n,y_n))$ that is the output of an $\eta$-noise-rate strong malicious noise process on a ``clean'' $n$-element data set of examples drawn from ${\cal U}$ and labeled by $c=c_{m^\star}$.\vspace{6pt}

    \textbf{Output:} A hypothesis $h$ for $c$ that with probability $1 - \exp(-\Omega(d))$ has error rate at most $\eps + 8\eta(1+o(1))$.\vspace{6pt}

    \begin{enumerate}
            \item[]\textbf{1. Remove contradictory examples:} 
            Remove contradictory examples from $S$ as described in Step~1 of \Cref{def:ICE-algorithm}. Let $S' = ((x_{i_1},y_{i_1}),\dots,(x_{i_{n'}},y_{i_{n'}}))$ be the subsequence of $n' \leq n$ labeled examples from $S$ consisting of the examples that survive this removal process, and let $S'_1 \sqcup \cdots \sqcup S'_w \sqcup S'_{\val}$ be the partition of $S'$ defined by $S'_\ell = \{(x_{i_j},y_{i_j}) \in S': x_{i_j} \in X_\ell\}$, $S'_\val = \{(x_{i_j},y_{i_j}) \in S' : x_{i_j} \in X_\val\}.$
            
             \vspace{6pt}
            \item[]\textbf{2. Compute a (real-valued) guess for each key bit:}
            For each $i \in [w],$ let $v_i \in \R$ be the value 
            $v_i := \keybitguess(S'_i)$ 
            (see \Cref{eq:keybitguess}).
This defines a vector $v=(v_1,\dots,v_w) \in [-1,1]^w.$  
             \vspace{6pt}
	\item[]\textbf{3. Round the guess vector to a Boolean vector:}
Let $\bz=\Round(v)=
(\Round(v_1),\dots,\Round(v_w))$ be the result of applying $\Round$ to each coordinate of $v$.
	\vspace{6pt}
	\item[]\textbf{4. Decode the Boolean vector:}
Run the decoding algorithm $\Dec$ on $\bz$ to obtain $m^{(1)},\dots,m^{(L')}$, where each $m^{(i)}$ is a string in $\bits^d$ and $L' \leq L = \poly(d)$.	\vspace{6pt}
	\item[]\textbf{5. Hypothesis testing to choose a final hypothesis:}
For each $i \in [L']$, perform hypothesis testing of the hypothesis $c_{m^{(i)}}$ using 
$S'$ 
(more precisely, compute the empirical fraction $\widehat{\eps}_i$ of examples $(x,y) \in S'$ for which $c_{m^{(i)}}(x) \neq y$, and outputs as the final hypothesis $h$ the $c_{m^{(i)}}$ for which $\widehat{\eps}_i$ is smallest, breaking ties arbitrarily).     
    \end{enumerate}
    \end{tcolorbox}
\caption{The algorithm $A_\mal$ for learning ${\cal C}_{\eta,\kappa}$ in the presence of $\eta$-rate strong malicious noise.}
\label{fig:Amal}
\end{figure}

We now describe what is the value $\keybitguess(\cdot)$ that $A_\mal$ constructs in Step~2.  Given a data set $S'_{i}$ that contains labeled examples from $X_i$, for $b \in \bits$ let $n_{i,b}$ be the number of examples in $S'_{i}$ with label $b$.
The value of $\keybitguess(S'_{i})$ is
\begin{equation} \label{eq:keybitguess}
\keybitguess(S'_{i}) := {\frac {n_{i,1} - n_{i,-1}} {R(1-\eta)}} \in \R.
\end{equation}
We give intuition for this definition below: 

\begin{itemize}

\item Recall that if there were no noise (i.e. $\eta=0$), then all examples in $S'_{i}$ would have the same (correct) label bit $b=\Enc(m^\star)_i$, and we would have $n_{i,-b}=0$ and $n_{i,b} \approx R$. So in this case we would have $\keybitguess(S'_{i}) \approx b/(1-\eta)=b$, i.e.~we would have a very good ``real-valued guess'' of the correct value of bit $b$.  (Intuitively, the $1-\eta$ in the denominator is because in the presence of noise at rate $\eta$, we only expect there to be $R(1-\eta)$ non-noisy examples in block $i$.)

\item Because of the initial step of ignoring contradictory examples, intuitively the malicious adversary must introduce $\rho R$ many mislabeled examples 
into block $i$ in order to decrease the value of $n_{i,b}$ from $R$ to $(1-\rho)R$.
Similarly, the adversary must introduce an additional $\rho R$ many mislabeled examples into block $i$ to increase the value of $n_{i,-b}$ from 0 to $\rho R$.

\end{itemize}

\begin{remark} \label{remark:keybitguess}
We remark that our randomized rounding procedure $\Round(\cdot)$ ``fully takes advantage'' of the fact that $\keybitguess(\cdot)$ outputs a real value to force the adversary to spend many corruptions in order to influence the final bit that is generated for each block.  This can be seen by contrasting the performance of $\Round(\cdot)$ with the performance of an alternative deterministic rounding scheme  in the following example. 

Consider the natural deterministic rounding procedure $\mathrm{DetRound}: \R \to \bits$ which simply outputs $\mathrm{DetRound}(v)=\sign(v)$.  Suppose that the noiseless data set contained $R$ examples in block $i$ that are correctly labeled with bit $b$.  

\begin{itemize}

\item If we were using $\mathrm{DetRound}$ rather than $\mathrm{Round}$, then an adversary could cause $\mathrm{DetRound}$ to output the wrong bit for block $i$ with probability 1 by generating only $R+1$ noisy examples. This is done by having $R$ of the noisy examples be contradictory examples for the correctly labeled examples, and having one additional example in block $i$ labeled with the incorrect bit $-b$.  Since contradictory examples are removed, for block $i$ we would have $n_{i,b}=0,n_{i,-b}=1$, and the output of $\keybitguess$ on that block would be $-b/R$, resulting in an output of $-b$ from $\mathrm{DetRound}.$

\item In contrast, in order to force $\Round$ to output the wrong bit with probability 1, an adversary would need to generate $2R$ noisy examples:  $R$ contradictory examples for the correctly labeled examples to get rid of them and cause $n_{i,b}=0$, and an additional $R$ incorrectly labeled examples to cause $n_{i,-b}=1$.  (If the adversary used only $R+1$ noisy examples as in the previous bullet, then $\Round$ would output the correct bit with probability almost 1/2.)

\end{itemize}

\end{remark}

\subsubsection{Useful notation} \label{sec:useful-notation}

While the analysis is not particularly complicated, it requires keeping track of a number of different quantities; doing this is made easier with some systematic notation, which we provide in this subsection.

\noindent {\bf Notation pertaining to the original data set (before contradictory examples are removed):}

\begin{itemize}

\item $\bS^\original = (\bx_1,c(\bx_i)),\dots,(\bx_n,c(\bx_n))$ is the original data set of $n$ i.i.d.~``clean'' labeled examples drawn from ${\cal D}$ in the first phase of the strong malicious noise process.

\item $\bZ_{\mal} \subseteq [n]$ is the set of indices that are randomly selected for corruption in the second phase of the strong malicious noise process (recall \Cref{remark:strong-malicious}).

\item $\bS^{\original,\replaced}=((\bx_{j},c(\bx_{j})))_{j \in \bZ_\mal}$ is the subsequence of the original examples that are replaced by the $\eta$-noise-rate strong malicious adversary.  

\item $\bS^{\original,\survivenoise}$ is the elements of $\bS^\original$ that ``survive'' into $S$, so $\bS^{\original,\survivenoise} = \bS^\original \setminus \bS^{\original,\replaced}.$

\item $S^{\new}=((x'_{j},y'_{j}))_{j \in \bZ_\mal}$ is the subsequence of noisy examples that the malicious adversary uses to replace the elements of $\bS^{\original,\replaced}$.  So the original $n$-element data set $S$ that is input to the algorithm is 
\[
S = \bS^{\original,\survivenoise} \sqcup S^{\new}.
\]

\item For each of the above sets, we append a subscript $_i$ to denote the subset consisting of those examples which lie in the $i$-th block of the key side, and we use a subscript $_\val$ to denote the subset of those examples which lie on the value side.
So, for example, $S^\new_i$ is the subset of labeled examples that the malicious adversary introduces which belong to block $i$.

\end{itemize}

\noindent {\bf Breaking down the new (noisy) examples in more detail.} We distinguish three different kinds of noisy examples (elements of $S^\new$) that the malicious adversary may introduce:
\[
S^\new_i =  S^{\new,\correct}_i \sqcup S^{\new,\incorrect,\contradictory}_i \sqcup S^{\new,\incorrect,\noncontradictory}_i
\]
where

\begin{itemize}

\item $S^{\new,\correct}_i$ consists of the examples in $S^\new_i$ which have the correct label $\Enc(m^\ast)_i$.

\item $S^{\new,\incorrect,\contradictory}_i$ consists of the examples in $S^\new_i$ which have the incorrect label $-\Enc(m^\ast)_i$ and moreover contradict some other example in $S_i$ in the sense of \Cref{def:ICE-algorithm}.

\item $S^{\new,\incorrect,\noncontradictory}_i$ consists of the examples in $S^\new_i$ which have the incorrect label $-\Enc(m^\ast)_i$ and do not contradict some other example in $S_i$.

\end{itemize}

We give concise variable names for the sizes of some of the sets defined above that will be of particular interest:
\begin{align*}
\balpha_i &:= |\bS^{\original,\survivenoise}_i| \tag{note bold font because this is a random variable}\\
\beta_i &:= |S^{\new,\correct}_i|\\
\gamma_i &:= |S^{\new,\incorrect,\contradictory}_i|\\
\delta_i &:= |S^{\new,\incorrect,\noncontradictory}_i|.
\end{align*}
so 
\begin{align}
|S_i| &= \balpha_i + \beta_i + \gamma_i + \delta_i \nonumber\\
|\bZ_\mal| = |S^\new| &= \sum_{i=1}^w \beta_i + \gamma_i + \delta_i.
\label{eq:Zmal}
\end{align}

\noindent {\bf Notation capturing the effect of ignoring contradictory examples.}
As described in step~1 of the algorithm, the first thing that the ICE-algorithm $A_\mal$ must do is remove contradictory examples to construct the $n'$-element subsequence $S' = ((x_{i_1},y_{i_1}),\dots,(x_{i_{n'}},y_{i_{n'}}))$ of $S$.  
The effect of ignoring contradictory examples is to remove all $\gamma_i$ of the examples in $S^{\new,\incorrect,\contradictory}_i$, as well as $\gamma_i$ of their doppelgangers in $\bS^{\original,\survivenoise}_i$ and $S^{\new,\correct}_i$.
Some useful notation:

\begin{itemize}

\item $\bS^{\original,\survivenoise,\surviveice}_i$ consists of the examples in $\bS^{\original,\survivenoise}_i$ that are not removed in Step~1 of $A_\mal.$

\item $S^{\new,\correct,\surviveice}_i$ consists of the examples in $S^{\new,\correct}_i$ that are not removed in Step~1 of $A_\mal$.

\end{itemize}

Writing
\begin{align*}
\balpha'_i &:= |\bS^{\original,\survivenoise,\surviveice}_i|\\
\beta'_i &:= |S^{\new,\correct,\surviveice}_i|,
\end{align*}
we have
\begin{align}
\balpha'_i + \beta'_i &= \balpha_i + \beta_i - \gamma_i, \label{eq:pickle}\\
S'_i &= \bS^{\original,\survivenoise,\surviveice}_i \sqcup S^{\new,\correct,\surviveice}_i \sqcup S^{\new,\incorrect,\noncontradictory}_i, \nonumber\\
|S'_i| &= \balpha'_i + \beta'_i + \delta_i = \balpha_i + \beta_i - \gamma_i + \delta_i. \nonumber
\end{align}

\subsubsection{Useful claims about typical events} \label{sec:useful-typical}

We require two claims, both of which essentially say that with high probability over the draw of a ``clean'' random sample and the random choice of which examples can be corrupted by the strong malicious noise process, 
various sets have sizes that are close to their expected sizes.  These claims are easy consequences of standard tail bounds, namely \Cref{fact:hoeffding-inequality,} and \Cref{fact:multiplicative-Chernoff-bound}.
We record the statements that we will need below.

First some setup:  Observe that for each $i \in [w]$ we have 
\[
\E[\balpha_i]=R(1-\eta),
\]
and that moreover we have
\[
\E[|\bZ_\mal|] = \eta n, \quad
\]
We define the quantity
\[
\Delta := n^{0.51},
\]
which should be thought of as an ``allowable'' amount of error between the actual values of various quantities and their expected values. We write
\[
x \approx_{\Delta} y
\]
as shorthand for $x \in [y - \Delta, y + \Delta]$.

\begin{lemma}
[Sets typically have the ``right'' size]
\label{lem:typical}
For a suitably large value of $n=\poly(d,1/\eta,1/\eps)=\poly(d/\eps)$, with probability $1-\exp(-\Omega(d))$ over the draw of $(\bS^\original,\bZ_\mal)$ and the internal randomness of $A_\mal$, both of the following statements hold:

\begin{itemize}

\item [(A)] $\balpha_i \approx_{\Delta} 2R(1-\eta)$ for each $i \in [w]$; 

\item [(B)] $|\bZ_\mal| \approx_{\Delta} \eta n$ (and hence  $|S^\new| \approx_{\Delta}  \eta n$, since $|S^\new|=|\bZ_\mal|$).

\end{itemize}

\end{lemma}

\subsubsection{Correctness of $A_\mal$}  In the rest of the argument, we assume that both of the ``typical'' events (A)-(B) described in \Cref{sec:useful-typical} take place.  We proceed to argue that under these conditions, 

\begin{itemize}
    \item [(I)] the real-valued vector $v \in \R^w$ computed in step~2 of Algorithm $A_\mal$ has $\|v - \Enc(m^\ast)\|_1 \leq (1-4\tau)w$ (recall that $c_{m^\ast}$ is the target concept);
    
    \item [(II)] With probability $1 - \exp(-\Omega(d))$, the true secret key $m^\ast$ is one of the $d$-bit strings $m^{(1)},\dots,m^{(L')}$ that is obtained in step~4 of Algorithm $A_\mal$; and
    
    \item [(III)] Given that (I) and (II) occur, with probability at least $1-\exp(-\Omega(d))$ the hypothesis $c_{m^{(i)}}$ that Algorithm $A_\mal$ outputs has error at most $\eps + 8\eta(1+o(1))$.
\end{itemize}

The correctness of $A_\mal$ is a direct consequence of (I), (II) and (III).
We prove these items in reverse order.

\medskip

\noindent {\bf Proof of (III).}
This is a fairly standard hypothesis testing argument.
Fix any possible message string $m \in \zo^d$ (recall that there is a one-to-one correspondence between message strings $m$ and elements of the concept class $c_{m}$).  Let the value $\hat{\eps}_m$ be the empirical fraction of examples $(x,y) \in S'$ for which $c_{m'}(x) \neq y$, let $\tilde{\beps}_m$ be the empirical fraction of examples $(x, y) \in \bS_{\original}$ for which $c_{m}(x) \ne y$, and let $\eps_m$ be the true error rate of hypothesis $c_{m}$ on noiseless uniform examples.
A standard additive Hoeffding bound give us that with probability $1 - \exp(-\Omega(w))$, the value of $\tilde \beps_m$ is within an additive $\eps$ of $\eps_m$. By a union bound over all $d$-bit strings $m'$, with probability $1 - \exp(-\Omega(w)) \cdot 2^d = 1 - \exp(-\Omega(w))$ we have this closeness for every $m' \in \{-1, +1\}^d$.  
Additionally, we have
\[|\hat{\eps}_m n' - \tilde\beps_m n| = \left|\hat{\eps}_m|S'| - \tilde{\beps}_m|\bS^\original|\right| \le |S' \mathop{\triangle} \bS^\original| \le |S' \mathop{\triangle} S| + |S \mathop{\triangle} \bS^\original| \le 4|\bZ_\mal|.\]
By (B) we have that $|\bZ_\mal| \approx_{\Delta} \eta n$, so for every $m' \in \{-1, +1\}^d$ we have that $|\hat{\eps}_mn' - \tilde \beps_m n| \leq 4(\eta n  + \Delta)$.

Let $\eps_{i}\coloneq \eps_{m^{(i)}}$ be the true error rate of hypothesis $c_{m^{(i)}}$ on noiseless examples, and recall that $\hat{\eps}_{i}$ is the empirical error $c_{m^{(i)}}$ has on $S'$. For the value $i^\ast$ such that $c_{m^\ast}=c_{m^{(i^\ast)}}$, we have $\eps_{i^\ast}=\tilde \beps_{i^\ast}=0$, so $\hat{\eps}_{i^\ast} \le \frac1{n'}(4(\eta n + \Delta))$.
Hence in step~5 the final hypothesis $c_{m^{(i)}}$ that is output must have $\hat{\eps}_{i} \leq \frac1{n'}( 4(\eta n + \Delta))$, which means that $\tilde{\beps}_i \le \frac8n(\eta n + \Delta) = 8\eta(1 + o(1))$ and therefore $\eps_{i} \leq \eps + 8\eta(1+o(1))$ as claimed. \qed

\medskip

\noindent {\bf Proof of (II).} 
By (I), we have that $\|v - \Enc(m^\ast)\|_1 \leq (1-4\tau)w$.
By \Cref{lem:rounding}, we have that with probability $1 - \exp(-\Omega(w)) = 1-\exp(-\Omega(d))$, the Hamming distance between $\bz = \Round(v)$ and the codeword $\Enc(m^\ast)$ is at most ${\frac 1 2}(1 - 2\tau)={\frac 1 2} - \tau$.
Hence by \Cref{thm:list-decodable-binary-bit-flip}, when step~4 applies the decoding algorithm $\Dec$ to $\bz$, with probability $1-\exp(-\Omega(d))$ one of the $m^{(i)}$ message strings that it generates will be exactly $m^\ast.$\qed

\medskip

\noindent {\bf Proof of (I).}
Recall that coordinate $i$ of the vector $v$ computed in step~2 of $A_\mal$ is
\[
v_i = \keybitguess(S'_{i}) = {\frac {n_{i,1} - n_{i,-1}} {R(1-\eta)}}.  \tag{by \Cref{eq:keybitguess}}
\]
Recalling the notation from \Cref{sec:useful-notation}, we have that
\[
n_{i,\Enc(m^\ast)_i} = \balpha'_{i} + \beta'_{i}, 
\quad \quad
n_{i,-\Enc(m^\ast)_i} = \delta_{i}
\]
from which we get that
\[
v_i = \Enc(m^\ast)_i \cdot \pbra{{\frac {\balpha'_{i} + \beta'_{i} - \delta_{i}}{R(1-\eta)}}} \text{~and hence}
\]
\begin{align*}
|v_i - \Enc(m^\ast)_i| &= \left|\Enc(m^\ast)_i \cdot \pbra{{\frac {\balpha'_{i} + \beta'_{i} - \delta_{i}}{R(1-\eta)}} - 1}\right|\\
&=\left|
{\frac {\balpha'_{i} + \beta'_{i} - \delta_{i}}{R(1-\eta)}}-1
\right|\\
&= 
\left|
{\frac {\balpha_{i} + \beta_{i} - \gamma_i - \delta_{i}}{R(1-\eta)}} -1
\right| \tag{by \Cref{eq:pickle}}\\
&\approx_{{\frac {\Delta}{2R(1-\eta)}}}
\left|
{\frac { \beta_{i} - \gamma_i - \delta_{i}}{R}} 
\right|. \tag{using (A)}\\
\end{align*}
Summing over all $i \in [w]$, we get that
\begin{align*}
\|v-\Enc(m^\ast)\|_1 &\leq {\frac {\Delta}{2R(1-\eta)}} + 
\sum_{i=1}^w \left|
{\frac { \beta_{i} - \gamma_i - \delta_{i}}{R}} 
\right| \\
&\leq 
{\frac {\Delta}{2R(1-\eta)}} + 
{\frac 1 {R}} \sum_{i=1}^w \left(\beta_{i} + \gamma_i + \delta_{i}\right)\\
&=
{\frac {\Delta}{2R(1-\eta)}} + {\frac {|\bZ_\mal|}{R}}
\tag{by \Cref{eq:Zmal}}\\
&\leq 
{\frac {\Delta}{2R(1-\eta)}}
+ {\frac {\eta n} {R}}
+ {\frac \Delta {R}} \tag{using (B)}
\end{align*}
Recalling \Cref{eq:R}, we have that
\[
{\frac {\eta n}{R}}={\frac w {2\kappa'}}= {\frac w {\kappa+{\frac 1 2}}}
\]
so
\begin{align*}
\|v-\Enc(m^\ast)\|_1 &\leq o(1) + {\frac w {\kappa + {\frac 1 2}}}
\tag{using $\Delta=o(R)$}\\
&<
w(1-4\tau),
\end{align*}
where the last inequality uses our choice of $\tau = {\frac {\kappa-1/2}{8}}$ and the fact that $c$ is a constant strictly between $1/2$ and 1.
This concludes the proof of (I) and with it the proof of part (2) of \Cref{thm:ICE-bad-news}. \qed

%% file: sections/MalEquivalent.tex
%!TEX root = ../malicious-nasty-writeup.tex

\subsection{Standard and strong malicious noise are equivalent}
\label{sec:standard-strong-equivalent}

It's easy to see that a learner which can tolerate strong malicious noise can also tolerate standard malicious noise, since the standard malicious noise adversary is a restricted version of the strong malicious noise adversary. In this section we prove a converse, which we will use in the proof of \Cref{thm:ICE-malicious-nasty}.
Recall, from \Cref{def:subsampling-filter}, that $\submn$ refers to taking a size-$n$ uniform subsample of a size-$m$ data set.
\begin{lemma}[Standard and strong malicious noise are equivalent]
    \label{lem:compare-mal}
     For any $\eps, \delta, \eta > 0$ and concept class $\mcC$ of functions $X \to \bits$, suppose that an algorithm $A$ using $n$ samples $(\eps, \delta)$-learns $\mcC$ over distribution $\mcD$ with (standard) malicious noise at rate $\eta$. Given $\delta_{\additional}>0$, let
    \begin{equation*}
        m \coloneqq O\paren*{\frac{n^4 \log(2|X| + 1)^2}{(\delta_{\additional})^4}}
    \end{equation*}
and let $A' \coloneqq A \circ \submn$. Then $A'$ $(\eps, \delta + \delta_{\additional})$-learns $\mcC$ over distribution $\mcD$ with strong malicious noise at rate $\eta$.
\end{lemma}
\begin{proof}
    Fix any concept $c \in \mcC$. We will apply \Cref{thm:BV-main} in order to prove that the probability $A'$ learns $c$ to accuracy at least $1 - \eps$ with $\eta$-strong-malicious noise is at least $1- \delta-\delta_{\additional}$. Doing so requires specifying all the parameters of \Cref{thm:BV-main}:
    \begin{enumerate}
        \item We set the domain to
        \begin{equation*}
            X' \coloneqq (X \times \bits) \cup \varnothing,
        \end{equation*}
        and the corruptible portion of it to just $X_{\mathrm{corrupt}} = \set{\varnothing}$.
        \item We set the budget in \Cref{thm:BV-main} (what is called ``$\eta$" in the theorem statement) to $1$.
        \item We apply \Cref{thm:BV-main} to the base distribution $\mcD_{c, \eta}$, which is the distribution that with probability $1 - \eta$, outputs $(\bx, c(\bx))$ for $\bx \sim \mcD$, and with probability $\eta$, outputs $\varnothing$.
        \item We set the test function $f:(X')^n \to [0,1]$ as 
        \begin{equation}
            \label{eq:def-test-additive}
            f(S) \coloneqq \begin{cases}
                0&\text{ if }\varnothing \in S \\
                \Pr[A(S) \text{ is not $\eps$-close to $c$ under ${\cal D}$}]&\text{otherwise,}
            \end{cases}
        \end{equation}
        where the probability is taken over the randomness of the learner $A$.
        % \item $\mcD_{c}$ to be the distribution of $(\bx,c(\bx))$ for $\bx \sim \mcD$.
        % \item $\mcD_{c, \eta}$ to be the distribution that, with probability $1 - \eta$, outputs $(\bx, c(\bx))$ for $\bx \sim \mcD$, and with probability $\eta$, outputs $\varnothing$.
    \end{enumerate}
    % Then, define the test function $f:(X')^n \to [0,1]$ as,
    % \begin{equation}
    %     \label{eq:def-test-additive}
    %     f(S) \coloneqq \begin{cases}
    %         0&\text{ if }\varnothing \in S \\
    %         \Pr[A(S) \text{ is $\eps$ close to $c$}]&\text{otherwise,}
    %     \end{cases}
    % \end{equation}
    % where the probability is taken over the randomness of the learner $A$. We'll apply \Cref{thm:BV-main} with its budget (the quantity named ``$\eta$" in that theorem) set to $1$. 
    With this setup, we claim that
    \begin{equation}
        \label{eq:oblivious-ub-additive}
        \sup_{\mcD'\text{ is a valid corruption of }\mcD_{c, \eta}}\set{\Ex_{\bS' \sim (\mcD')^n}[f(\bS')]} \leq \delta.
    \end{equation}
    To see this, consider any $\mcD'$ that is a valid corruption of $\mcD_{c,\eta}$. Then, there is a coupling of $\bx \sim \mcD_{c,\eta}$ and $\bx' \sim \mcD'$ for which, with probability $1$, either $\bx = \bx'$ or $\bx = \varnothing$ (since $X_{\mathrm{corrupt}}$ contains only $\varnothing$). Let $\mcO$ be the distribution of $\bx'$ conditioned on $\bx = \varnothing$. We can write $\mcD'$ as the mixture
    \begin{equation*}
        \mcD' = (1-\eta) \mcD_{c} + \eta \mcO,
    \end{equation*}
    where $\mcD_c$ is the distribution that outputs $(\bx, c(\bx))$ for $\bx \sim \mcD$. Recall, from \Cref{prop:mal-harder-oblivious}, that since $A$ $(\eps, \delta)$-learns $\mcC$ with $\eta$- malicious noise, it also $(\eps, \delta)$-learns $\mcC$ with $\eta$-Huber contamination. This exactly corresponds to $\Ex_{\bS' \sim (\mcD')^n}[f(\bS')]$ being at most $\delta$ as long as $\mcO$, and therefore $\mcD'$, never (i.e. with probability $0$) outputs $\varnothing$.

    Next, consider the case where $\mcD'$ outputs $\varnothing$ with nonzero probability. We can instead consider some $\mcD'_{\text{no }\varnothing}$ which is identical to $\mcD'$ except that whenever $\mcD'$ outputs $\varnothing$, $\mcD'_{\text{no }\varnothing}$ outputs an arbitrary element of $X \times \bits$. If $\mcD'$ is a valid corruption of $\mcD_{c,\eta}$, then so is $\mcD'_{\text{no }\varnothing}$ (since $\varnothing$ is a corruptible element and the corruption budget is $1$). Furthermore, based on how we defined $f$ in \Cref{eq:def-test-additive},
    \begin{equation*}
        \Ex_{\bS' \sim (\mcD')^n}[f(\bS')] \leq \Ex_{\bS' \sim (\mcD'_{\text{no }\varnothing})^n}[f(\bS')],
    \end{equation*}
so by appealing to the earlier case we get that \Cref{eq:oblivious-ub-additive} holds in general. 

Now we can apply \Cref{thm:BV-main}, which for our choice of $m$ gives that,
    \begin{equation}
        \label{eq:apply-BV-additive}
         \Ex_{\bS \sim (\mcD_{c,\eta})^m}\bracket*{\sup_{S'\text{ is a valid corruption of }\bS}\set{\Ex[f \circ \submn(S')]}} \leq \delta + \delta_{\additional}.
    \end{equation}
    Now, consider any strategy for the strong malicious adversary. Recall that, for this adversary, we first draw $\bS \sim (\mcD_c)^m$ and $\bZ_{\mal} \subseteq [m]$ where each $i \in [m]$ is independently included in $\bZ_{\mal}$ with probability $\eta$. Given these, the adversary chooses an arbitrary sample $S' \in (X \times \bits)^m$ with the constraint that $\bS_i \neq S_i'$ can only occur if $i \in \bZ_{\mal}$.

    Consider the sample $\bS_{\varnothing} \in (X')^{m}$ where, for each $i \in [m]$,
    \begin{equation*}
        (\bS_{\varnothing})_i \coloneqq \begin{cases}
            \varnothing&\text{if $i \in \bZ_{\mal}$}.\\
            \bS_i&\text{otherwise}.
        \end{cases}
    \end{equation*}
    Then, the distribution of $\bS_{\varnothing}$ is exactly $(\mcD_{c,\eta})^m$. Furthermore, if the strong malicious adversary can corrupt $(\bS, \bZ_{\mal})$ to $\bS'$, then $\bS'$ is a valid corruption of $\bS_{\varnothing}$. Therefore, \Cref{eq:apply-BV-additive} says that the probability $A'$ fails with $\eta$-strong malicious noise is at most $\delta + \delta_{\additional}.$
\end{proof}

%% file: sections/AmplificationCounterexample.tex
\section{Counterexample to standard success amplification in the presence of nasty noise}
\label{sec:counterexample-amplify}
In this section, we give a counterexample, showing that the standard approach to amplifying the success probability of learners in the noiseless setting does not extend to the nasty-noise setting. Recall, as discussed in \Cref{sec:amplify}, that algorithm works as follows: Given an algorithm $A$ that uses $n$ samples and parameters $k$, $n_{\test}$, that approach draws a size-$nk + n_{\test}$ dataset. It splits that dataset into $k$ pieces of size $n$, on which it uses $A$ to generate $k$ hypotheses, and one test set of size $n_{\test}$. Then, it uses the last piece to test the empirical error of each hypothesis and outputs the hypothesis minimizing empirical error.

We formalize this approach in the pseudocode below. 

\begin{figure}[h] 
  \captionsetup{width=.9\linewidth}
    
    \begin{tcolorbox}[colback = white,arc=1mm, boxrule=0.25mm]
    \vspace{2pt} 
    \BadAmplify$(A,k, n_{\test}, S_{\mathrm{big}}):$\vspace{6pt}
    
    \textbf{Input:} A learner $A$ taking in a size-$n$ data set, parameter $k \in \N$, and a size-$(nk + n_{\test})$ data set, $S_{\mathrm{big}}$.\vspace{6pt}

    \textbf{Output:} A hypothesis $h$. \vspace{6pt}

    \begin{enumerate}
            \item[]\textbf{1.\,\,Split sample into groups:} Draw a uniform permutation $\bsigma:[nk + n_{\test}] \to [nk + n_{\test}]$ and define $\bsigma(S_{\mathrm{big}})$ as the size-$(nk + n_{\test})$ data set satisfying
            \begin{equation*}
                \bsigma(S_{\mathrm{big}})_i = (S_{\mathrm{big}})_{\bsigma(i)}.
            \end{equation*}
            Set the groups $\bS^{(1)}, \bS^{(2)}, \ldots, \bS^{(k)}$ to consist of the first $n$ elements of $\bsigma(S_{\mathrm{big}})$, the second $n$ elements, and so on, and let $\bS^{(\test)}$ be the remaining elements.\vspace{6pt}
            \item[]\textbf{2.\,\,Generate hypotheses:} For each $i \in [k]$, let $\bh_i = A(\bS^{(i)})$. \vspace{6pt}
            \item[]\textbf{3.\,\,Select a hypothesis:} Test the accuracy of $\bh_1, \ldots, \bh_k$ on $\bS^{(\test)}$. Output the hypothesis with the highest accuracy, breaking ties uniformly at random.
    \end{enumerate}
    \end{tcolorbox}
\caption{An algorithm that successfully amplifies the success probability of noiseless learners but fails to do so in the nasty-noise setting.}
\label{fig:BoostFail}
\end{figure} 

We note that the only difference between $\BadAmplify$ and $\Amplify$ (as defined in \Cref{fig:Boost}) is that $\BadAmplify$ attempts to perform hypothesis selection. While this difference may seem to be very small, we show that $\BadAmplify$ fails to provide any useful amplification in the nasty noise setting regardless of how $k$ and $n_{\test}$ are set.
\begin{theorem}
    \label{thm:bad-amplify}
     For any $\eps> 0$, $0.01 \leq \eta \leq 0.49$, and integers $n,k,n_{\test}\geq 1$ there is a concept class $\mcC$, distribution $\mcD$, and algorithm $A$ using $n$ samples that learns $\mcC$ to expected error $\eps$ with $\eta$-nasty noise over distribution $\mcD$ for which the following holds. There is a strategy for the $\eta$-rate nasty-noise adversary for which the algorithm $\BadAmplify$ from \Cref{fig:BoostFail} outputs a hypothesis with error $0.99$ with probability at least $\eps - 2^{-\Omega(n)}$.
\end{theorem}

As alluded to earlier, a standard analysis shows that if algorithm $A$ is an $\eps$-expected-error learner for a concept class ${\cal C}$ in the noiseless setting, then for $k=O(\log(1/\delta))$,
$n_{\test}=\poly(\log(\log(1/\delta)),1/\eps)$, algorithm $\BadAmplify(A,k,n_{\test},\bS)$ generates a hypothesis with error $O(\eps)$ with probability at least $1-\delta$ when run on a dataset $\bS$ of $(nk+n_{\test})$ many i.i.d.~noiseless examples labeled according to some concept in ${\cal C}$.  Thus \Cref{thm:bad-amplify} indeed demonstrates the failure of the standard approach in the presence of nasty noise.

To prove \Cref{thm:bad-amplify}, we first design a learner $A$ that achieves expected error $\eps$ in, for our purposes, the worst possible way: With probability $\eps$, it outputs a hypothesis with error close to $1$, and otherwise it outputs a hypothesis with no error. This will remain true for the $k$ hypotheses generated by $\BadAmplify$. On average, $k\eps$ of those hypotheses will have error $\geq 0.99$. Then, we design a strategy for the nasty noise adversary so that essentially \emph{all} of the $k$ hypotheses have $0$ on the test set, even those with high error. As a result, $\BadAmplify$ is likely to pick a hypothesis with high error.

\begin{proof}
    First, we design the domain, concept class, distribution, and base learner $A$.
    
    \begin{enumerate}
        \item The domain will be,
    \begin{equation*}
        X \coloneqq X_{\sm} \cup X_{\lar} \quad \text{where } X_{\sm} \coloneqq [100(nk + n_{\test})] \quad\text{and}\quad X_{\lar} \coloneqq 2^{X_{\sm}}.
    \end{equation*}
     It will be technically convenient to view each element of $X_{\lar}$ as specifying a subset of $X_{\sm}$.
     \item  The concept class $\mcC$ contains only two functions: The constant $+1$ and constant $-1$ function.
     \item The distribution $\mcD$ is uniform over $X_{\sm}$.
     \item The learner $A$ operates as follows: Given $n$ labeled points, if at least $n/2$ of the points have the label $1$, it sets $b = 1$. Otherwise, it sets $b = -1$. Then, it flips a coin that is heads with probability $p \coloneqq 1 - \eps + 2^{-\Omega(n)}$.
     \begin{enumerate}
         \item If this coin comes up heads, it outputs the hypothesis that labels every point as $b$.
         \item If this coin comes up tails, let $x_{\lar}$ be the first element of $X_{\lar}$ in its sample (and defaulting to $x_{\lar} = \varnothing$ if there is no such element). $A$ outputs the hypothesis $h$ with the following behavior (viewing $x_{\lar}$ as a subset of $X_{\sm}$):
         \begin{equation*}
             h(x) = \begin{cases}
                 b&x \in x_{\lar}\text{ or }x = x_{\lar}\\
                 -b&\text{otherwise.}
             \end{cases}
         \end{equation*}
     \end{enumerate}
    \end{enumerate}
    
    We begin by proving that $A$ learns $\mcC$ to expected error $\eps$ with $\eta$-nasty noise over distribution $\mcD$. For any concept $c \in \mcC$, $c$ is the constant $b^{\star}$ function for some $b^{\star} \in \bits$. The number of points labeled with $-b^{\star}$ that $A$ receives with the $\eta$-rate nasty adversary is upper bounded by $\Bin(n, \eta)$. For $\eta \leq 0.49$, by a Chernoff bound, $A$ will therefore correctly determine $b =b^{\star}$ with probability at least $1 - 2^{-\Omega(n)}$. Furthermore, whenever $A$ correctly determines $b$, it outputs a hypothesis with zero error with probability  $p \coloneqq 1 - \eps + 2^{-\Omega(n)}$. Therefore, $A$ outputs a zero-error hypothesis with probability at least $1 -\eps$, so it has an expected error of at most $\eps$.

    Next, we design a nasty adversary for which $\BadAmplify$ outputs a bad hypothesis with probability at least $\eps - 2^{-\Omega(n)}$
    (and thus it indeed fails to amplify success probability).  Given a clean sample $S = (x_1, b^{\star}), \ldots, (x_{nk + n_{\test}}, b^{\star})$, the nasty adversary does the following:
    \begin{enumerate}
        \item Let $X_{\mathrm{appear}} = \set{x_1, \ldots, x_{nk + n_{\test}}}$ be the subset of $X_{\sm}$ that appears in the sample.
        \item Draw a noise budget $\bz \sim \Bin(nk  + n_{\test}, \eta)$.
        \item Let $x_{\lar} \in X_{\lar}$ be the point representing $X_{\mathrm{appear}}$ and change the first $\bz$ points in the sample all to the same point $(x_{\lar}, b^{\star})$.
    \end{enumerate}

    We now analyze the hypotheses that $\BadAmplify$ creates with this strategy. A few straightforward observations:
    \begin{enumerate}
        \item Every point in the sample has the correct label, so every copy of $A$ will correctly determine $b = b^{\star}$.
        \item If any of the $\bz$ corruptions appear in the group $\bS^{(i)}$, then $\bh_i$ will have perfect accuracy on every point in $S_{\lar}$, and by extension, all points in $\bS^{(\test)}$. This is because the corruptions encode the entire sample and ensure the learner gets all points correct.
        \item $X_{\mathrm{appear}}$ contains at most $nk + n_{\test}$ points. Since $\mcD$ is uniform over $X_{\sm}$ which contains $100(nk + n_{\test})$ points, any hypothesis for which $p$-biased coin comes up tails will have error at least $0.99$ with respect to $\mcD$.
    \end{enumerate}
    Recall, by \Cref{prop:group-dists}, that the number of corrupted examples in each group is distributed according to $\Bin(n,\eta)$ and, since $\eta \geq 0.01$, is therefore nonzero with probability $1 - 2^{-\Omega(n)}$. As a result, the mean fraction of groups that simultaneously have error at least $0.99$ with respect to $\mcD$ but error $0$ with respect to $\bS_{\test}$ is at least $\eps - 2^{-\Omega(n)}$. Since $\BadAmplify$ breaks ties uniformly, it outputs a hypothesis with error at least $0.99$ with probability at least $\eps - 2^{-\Omega(n)}$.
\end{proof}